\documentclass{article}

\PassOptionsToPackage{numbers, sort, compress}{natbib}
\usepackage[preprint]{neurips_2022}

\usepackage[utf8]{inputenc} % allow utf-8 input
\usepackage[T1]{fontenc}    % use 8-bit T1 fonts
\usepackage{url}            % simple URL typesetting
\usepackage{booktabs}       % professional-quality tables
\usepackage{amsfonts}       % blackboard math symbols
\usepackage{nicefrac}       % compact symbols for 1/2, etc.
\usepackage{microtype}      % microtypography
\usepackage{xcolor}         % colors

\long\def\comment#1{}
\usepackage{microtype}
\usepackage{graphicx}
\usepackage{booktabs} % for professional tables

\usepackage{algorithm}

\usepackage{amsmath}
\usepackage{amssymb}
\usepackage{mathtools}
\usepackage{amsthm}
\usepackage{enumitem}

\usepackage{xr-hyper}
\usepackage{hyperref}
\usepackage[capitalize,noabbrev]{cleveref}

\makeatletter
\newcommand*{\addFileDependency}[1]{
  \typeout{(#1)}
  \@addtofilelist{#1}
  \IfFileExists{#1}{}{\typeout{No file #1.}}
}
\makeatother

\theoremstyle{plain}
\newtheorem{theorem}{Theorem}[section]
\newtheorem{proposition}[theorem]{Proposition}
\newtheorem{lemma}[theorem]{Lemma}
\newtheorem{corollary}[theorem]{Corollary}
\theoremstyle{definition}
\newtheorem{definition}[theorem]{Definition}
\newtheorem{assumption}[theorem]{Assumption}
\theoremstyle{remark}
\newtheorem{remark}[theorem]{Remark}

\usepackage{url}
\usepackage{graphicx,wrapfig,lipsum} 
\usepackage{times}
\usepackage{soul}
\usepackage{url}
\usepackage[utf8]{inputenc}
\usepackage[small]{caption}
\usepackage{graphicx}
\usepackage{amsthm}
\usepackage{booktabs}
\usepackage{mwe} % for blindtext and example-image-a in example
\usepackage{wrapfig}

\usepackage[ruled,linesnumbered,noend,algo2e]{algorithm2e}
\usepackage{setspace}
\usepackage{amsmath,amssymb,amsfonts}
\usepackage{amsmath,amsfonts,bm}

\def\eqref#1{(\ref{#1})}
\def\1{\bm{1}}

\def\eps{{\epsilon}}

\DeclareMathAlphabet{\mathsfit}{\encodingdefault}{\sfdefault}{m}{sl}
\SetMathAlphabet{\mathsfit}{bold}{\encodingdefault}{\sfdefault}{bx}{n}

\def\gB{{\mathcal{B}}}

\def\gD{{\mathcal{D}}}

\def\gF{{\mathcal{F}}}

\def\gH{{\mathcal{H}}}

\def\gO{{\mathcal{O}}}
\def\gP{{\mathcal{P}}}
\def\gQ{{\mathcal{Q}}}

\def\gU{{\mathcal{U}}}

\def\gZ{{\mathcal{Z}}}

\DeclareMathOperator*{\argmax}{arg\,max}
\DeclareMathOperator*{\argmin}{arg\,min}

\crefformat{equation}{(#2#1#3)}
\crefrangeformat{equation}{(#3#1#4) to~(#5#2#6)}
\crefmultiformat{equation}{(#2#1#3)}%
{ and~(#2#1#3)}{, (#2#1#3)}{ and~(#2#1#3)}
\crefname{assumption}{assumption}{assumptions}
\crefname{Assumption}{Assumption}{Assumptions}
\usepackage{bbm}
\usepackage[mathscr]{euscript}
\usepackage{graphicx}
\usepackage{textcomp}
\usepackage{xcolor}
\usepackage[detect-all]{siunitx}

\usepackage{xspace}
\usepackage{caption}
\usepackage{subcaption}
\usepackage{bm}
\usepackage{array}
\usepackage{multirow}

\makeatletter
\newcommand{\multilines}[1]{%
	\begin{tabularx}{\dimexpr\linewidth-\ALG@thistlm}[t]{@{}X@{}}
		#1
	\end{tabularx}
}
\makeatother

\usepackage[textsize=tiny]{todonotes}

\newcommand{\OurAlg}{\ensuremath{\textsf{WAFL}}\xspace}
\Crefname{figure}{Fig.}{Figs.}
\Crefname{table}{Table}{Tables}
\Crefname{section}{Sec.}{Secs.}

\DeclarePairedDelimiterX{\norm}[1]{\lVert}{\rVert}{#1}
\DeclarePairedDelimiterX{\abs}[1]{\lvert}{\rvert}{#1}
\newcommand{\innProd}[1]{\left\langle#1\right\rangle}
\newcommand{\set}[1]{\left\lbrace#1\right\rbrace}

\DeclareSymbolFont{extraup}{U}{zavm}{m}{n}
\DeclareMathSymbol{\varheart}{\mathalpha}{extraup}{86}
\newcommand{\SumNoLim}[2]{\ensuremath{\sum\nolimits_{#1}^{#2}}}
\newcommand{\SumLim}[2]{\ensuremath{\sum\limits_{#1}^{#2}}}

\newcommand{\defeq}{\vcentcolon=}
\newcommand{\eqdef}{=\vcentcolon}
\newcommand{\nbigP}[1]{\ensuremath{(#1)}}

\newcommand{\bigP}[1]{\ensuremath{\bigl(#1\bigr)}}
\newcommand{\bigS}[1]{\ensuremath{\bigl[#1\bigr]}}
\newcommand{\bigC}[1]{\ensuremath{\bigl\{#1\bigr\}}}
\newcommand{\bigNorm}[1]{\ensuremath{\bigl\lVert#1\bigr\rVert}}
\newcommand{\biggP}[1]{\ensuremath{\biggl(#1\biggr)}}
\newcommand{\biggS}[1]{\ensuremath{\biggl[#1\biggr]}}
\newcommand{\biggC}[1]{\ensuremath{\biggl\{#1\biggr\}}}

\newcommand{\BigP}[1]{\ensuremath{\Bigl(#1\Bigr)}}
\newcommand{\BigS}[1]{\ensuremath{\Bigl[#1\Bigr]}}
\newcommand{\BigC}[1]{\ensuremath{\Bigl\{#1\Bigr\}}}
\newcommand{\BigNorm}[1]{\ensuremath{\Bigl\lVert#1\Bigr\rVert}}
\newcommand{\BiggP}[1]{\ensuremath{\Biggl(#1\Biggr)}}
\newcommand{\BiggS}[1]{\ensuremath{\Biggl[#1\Biggr]}}

\DeclareUnicodeCharacter{FB01}{fi}
\newcommand{\blue}[1]{{\color{blue}#1}\xspace}

\newcommand{\up}{\texttt{up}\xspace}
\newcommand{\mt}{\texttt{mt}\xspace}
\newcommand{\sv}{\texttt{sv}\xspace}

\title{On the Generalization of Wasserstein Robust Federated Learning}

\author{%
  Tung-Anh Nguyen \\
  School of Computer Science\\
  The University of Sydney\\
  Sydney, NSW 2000 \\
  \texttt{tung6100@uni.sydney.edu.au} \\
  \And
  Tuan Dung Nguyen \\
  School of Computing \\
  The Australian National University \\
  Canberra, ACT 2601\\
  \texttt{josh.nguyen@anu.edu.au} \\
  \AND
  Long Tan Le \\
  School of Computer Science\\
  The University of Sydney\\
  Sydney, NSW 2000 \\
  \texttt{long.le@sydney.edu.au} \\
  \And
  Canh T. Dinh\\
  School of Computer Science\\
  The University of Sydney\\
  Sydney, NSW 2000 \\
  \texttt{canh.dinh@sydney.edu.au} \\
  \And
  Nguyen H. Tran \\
  School of Computer Science\\
  The University of Sydney\\
  Sydney, NSW 2000 \\
  \texttt{nguyen.tran@sydney.edu.au} \\
}

\begin{document}

\maketitle

\begin{abstract}

In federated learning, participating clients  typically possess non-i.i.d. data, posing a significant challenge to generalization to unseen distributions. To address this, we propose a Wasserstein distributionally robust optimization scheme called \OurAlg. Leveraging its duality, we frame \OurAlg as an empirical surrogate risk minimization problem, and solve it using a local SGD-based algorithm with convergence guarantees. We show that the robustness of \OurAlg is more general than related approaches, and the generalization bound is robust to all adversarial distributions inside the Wasserstein ball (ambiguity set). Since the center location and radius of the Wasserstein ball can be suitably modified, \OurAlg shows its applicability not only in robustness but also in domain adaptation. Through empirical evaluation, we demonstrate that \OurAlg generalizes better than the vanilla FedAvg in non-i.i.d. settings, and is more robust than other related methods in distribution shift settings. Further, using benchmark datasets we show that \OurAlg is capable of generalizing to  unseen target domains.

\end{abstract}

\section{Introduction}

%\blue{
%Points to cover:
%\begin{itemize}
%	\item Federated learning: promises, systems heterogeneity, statistical heterogeneity
%	\item Problems: generalization to unseen distributions, bounding the generalization error, robustness
%	\item Robustness and adversarial training: agnostic FL and variants, affine covariate shift, ambiguity set
%	\item Generalization bound: Rademacher complexity
%	\item Our contributions: (1) \OurAlg, (2) Empirical surrogate risk minimization and algorithm, (3) Robust generalization bound, (4) Flexibility, (5) Experiments
%\end{itemize}
%}

Federated learning (FL) \citep{konecny_federated_2016,mcmahan_communication-efficient_2017} has emerged as a cutting-edge technique in distributed and privacy-preserving machine learning. The nature of non-i.i.d. data in clients' devices poses an important challenge to FL commonly called \emph{statistical heterogeneity}. The global model trained on this data using the de facto FedAvg algorithm \citep{mcmahan_communication-efficient_2017} has been shown to generalize poorly to individual clients' data, and further to \emph{unseen distributions} on new clients as they enter the network.

Several solutions to data heterogeneity have been proposed. Personalized FL \citep{mansour_three_2020, fallah_personalized_2020,deng_adaptive_2020,dinh_personalized_2021,li_ditto_2021,Collins2021} and multi-task FL \citep{smith_federated_2018, Marfoq2021} are \emph{client-adaptive} approaches, where a personalized model is adapted to each client from the global model. From another perspective, distributionally robust FL trains a model using a worst-case objective over an \emph{ambiguity set} \citep{mohri_agnostic_2019, du_fairness-aware_2020, reisizadeh_robust_2020, deng_distributionally_2020}. This approach is \emph{client-uniform} because a single global model is judiciously learned to deliver uniformly good performance not only for all training clients but also for new/unseen clients with unknown data distributions. It is specifically useful when test distributions drift away from the training distributions.

A natural question when designing distributionally robust FL frameworks is \emph{generalization}: How can minimizing the training error also bound the test error? In FL, \citet{mohri_agnostic_2019} proposed agnostic FL where a model is designed to be robust against any distribution that lies inside the convex hull of the clients' distributions. 
%any ambiguity set as a convex combination of the clients' distributions. 
\citet{reisizadeh_robust_2020} applied the general affine covariate shift -- used in the standard adversarial robust training -- into FL training. In characterizing the generalization bounds, while \citet{mohri_agnostic_2019} relied on the standard Rademacher complexity, \citet{reisizadeh_robust_2020} use the margin-based technique developed by \citet{bartlett_spectrally-normalized_2017}.

In this work, we take a different approach called WAsserstein distributionally robust FL (\OurAlg for short). The ambiguity set in \OurAlg is a Wasserstein ball of all adversarial distributions in close proximity to the nominal data distribution at the center. Our main contributions are:
\begin{itemize}[leftmargin=20pt]
    \item We propose \OurAlg, a Wasserstein distributionally robust optimization problem for FL. To make \OurAlg amenable to distributed optimization, we transform the original problem into a minimization of the empirical surrogate risk and solve it using a local SGD-based algorithm with convergence guarantees.
    \item We demonstrate \OurAlg's flexibility in robustness and domain adaptation by adjusting its hyperparameters related to the Wassterstein ball's center and radius. We show how \OurAlg's output can reduce the test error by bounding its excess risk, and call this the \emph{robust generalization bound} as it is applicable to all adversarial distributions inside the Wasserstein ball.
    \item Experimentally, we show \OurAlg's significant improvement over the de facto FedAvg and other robust FL methods both in scenarios with adversarial attacks and in applications of multi-source domain adaptation.
\end{itemize}

\comment{
\begin{itemize}
	\item We propose a distributionally robust optimization problem to address statistical heterogeneity in FL. By controlling the center and radius of the Wasserstein ball, we show that \OurAlg is robust to a wider range of adversarial distributions than agnostic or adversarial FL.
	\item To make \OurAlg more amenable to distributed optimization, we transform the original problem into a minimization of the empirical surrogate risk. We propose a local SGD-based algorithm to solve this surrogate problem. With additional Lipschitz smoothness conditions, standard techniques can be applied to find the convergence rate for the proposed algorithm.
	\item We show how \OurAlg's output can reduce the test error by bounding its excess risk. We call this the \emph{robust generalization bound} as it is applicable to all adversarial distributions inside the Wasserstein ball. By scaling the Wasserstein radius based on local data sizes, we show this bound is applicable to the true (unknown) data  distribution among all clients.
	\item We show \OurAlg's extra flexibility in controlling the Wasserstein ball by adjusting the nominal distribution. This enables applications such as multi-source domain adaptation and robustness to all unknown distributions with minimal Wasserstein radius.
	\item Experimentally, we discuss how to control this radius for a given center location by fine-tuning a robust hyperparameter. We show that \OurAlg generalizes better than \blue{the FedAvg baseline} in non-i.i.d. settings and further outperforms existing robust methods in distribution shift settings. We finally explore \OurAlg's capability in transferring knowledge from multi-source domains to related target domains with much less data and/or without labels. \blue{TODO: Rewrite this after finishing the experiments section.}
\end{itemize}
}

\section{Related Work}
\textbf{Federated learning} was introduced in response to three challenges of machine learning at scale: massive data quantities at the edge, communication-critical networks of participating devices, and privacy-preserving learning without central data storage \citep{konecny_federated_2016,mcmahan_communication-efficient_2017}. The de facto FedAvg algorithm \citep{mcmahan_communication-efficient_2017} based on local stochastic gradient descent (SGD) and averaging is often considered a baseline in FL.

Most challenges of FL are categorized into \emph{systems heterogeneity} and \emph{statistical heterogeneity}. The former focuses on communication problems such as connection loss and bandwidth minimization. This motivated some prior works to design more communication-efficient methods \citep{konecny_federated_2016,konecny_federated_2017,suresh_distributed_2017, pmlr-v108-reisizadeh20a}. On the other hand, statitical heterogeneity is concerned with clients' non-i.i.d. data, which
%Distributed optimization techniques often assume i.i.d. data among a network's nodes, but this assumption does not hold for FL. 
is the main cause behind aggregating very different models leading to one that does not perform well on any data distribution. To address this, many ideas have been introduced. \citet{li_convergence_2020} provided much theoretical analysis of FL non-i.i.d. settings. \citet{zhao_federated_2018} proposed an FL framework which globally shares a small subset of data among clients to train the model with non-i.i.d. data. Furthermore, some studied on multi-task FL frameworks \citep{smith_federated_2018, Marfoq2021} in which each client individually learns its own data pattern while borrowing information from other clients, \comment{\citet{mansour_three_2020} suggested three approaches to adapt the FL model to enable personalization, in reponse to distribution shift} while several personalized FL models have also been developed in response to distribution shifts \citep{mansour_three_2020, fallah_personalized_2020, deng_adaptive_2020, dinh_personalized_2021, li_ditto_2021, Collins2021}.

%\red{\textbf{Robustness in FL} can roughly be categorized into three main approaches: robustness through personalization, robustness to adversarial data, and robustness to unseen domain. Regarding robustness to unseen domain with distributional uncertainties, several works were proposed to build a learning model that is robust against an unknown testing distribution \citep{mohri_agnostic_2019, du_fairness-aware_2020, reisizadeh_robust_2020, deng_distributionally_2020, huang_compositional_2021}. Based on agnostic FL proposed by \citet{mohri_agnostic_2019}, \citet{du_fairness-aware_2020} introduced AgnosticFair as a two-player adversarial minimax game between the learner and the adversary, and \citet{deng_distributionally_2020} proposed DRFA, a FL algorithm based on distributionally robust optimization with efficient communication.}

\textbf{Wasserstein distributionally robust optimization (WDRO)} aims to learn a robust model against adversarially manipulated data. An unknown data distribution is assumed to lie within a Wasserstein ball centered around the empirical distribution \citep{kuhn_wasserstein_2019}. WDRO has received attention as a promising tool for training parametric models, both in centralized and federated learning settings.
%, but without showing risk consistency

In centralized learning, many studies have proposed solutions based on WDRO problems for certain machine learning tasks. %\citep{shafieezadeh_abadeh_distributionally_2015, gao_distributionally_2016, esfahani_data-driven_2017, chen_robust_2018, sinha_certifying_2020, blanchet_robust_2019, shafieezadeh-abadeh_regularization_2019, gao_wasserstein_2020}. 
For instance, \citet{shafieezadeh_abadeh_distributionally_2015} considered a robust logistic regression model under the assumption that the probability distributions lie in a Wasserstein ball. \citet{chen_robust_2018, blanchet_robust_2019, gao_wasserstein_2020} leveraged WDRO to recover regularization formulations in classification and regression. \citet{gao_distributionally_2016} proposed a minimizer based on a tractable approximation of the local worst-case risk. \citet{esfahani_data-driven_2017} used WDRO to formulate the search for the largest perturbation range as an optimization problem and solve its dual problem. \citet{sinha_certifying_2020} introduced a robustness certificate based on a Lagrangian relaxation of the loss function which is provably robust against adversarial input distributions within a Wasserstein ball centered around the original input distribution. \citet{Lau2022} suggested using the notion of Wasserstein barycenter to construct the nominal distribution in WDRO problems.

%In FL, some existing works have explored the Wasserstein distance to enhance robustness \citep{reisizadeh_robust_2020,  diamandis_wasserstein_2021, du_fairness-aware_2020, deng_distributionally_2020}. \citet{reisizadeh_robust_2020} introduced FedRobust based on the adversarial robust training. Based on agnostic FL proposed by \citet{mohri_agnostic_2019}, \citet{du_fairness-aware_2020} introduced AgnosticFair as a two-player adversarial minimax game between the learner and the adversary, and \citet{deng_distributionally_2020} proposed DRFA, an FL algorithm based on distributionally robust optimization with efficient communication.

In the context of FL, several works have studied robustness from different perspectives. For example, \citet{reisizadeh_robust_2020} proposed an adversarial robust training method called FedRobust based on a minimax formulation involving the Wasserstein distance. \citet{deng_distributionally_2020} proposed DRFA, a communication-efficient distributionally robust algorithm based on periodic averaging techniques. \citet{mohri_agnostic_2019} and \citet{du_fairness-aware_2020} introduced agnostic FL frameworks using two-player adversarial minimax games between the learner and the adversary to achieve fairness. 
\section{Wasserstein Robust Federated Learning}
\subsection{Expected Risk and Empirical Risk Minimization in Federated Learning}
Consider $m$ clients where each client $i \in  [m] \defeq \set{1, \ldots, m}$ has its data generating  distribution $P_i$ supported on domain $\mathcal{Z}_i \defeq (\mathcal{X}_i, \mathcal{Y}_i)$. Consider the parametrized hypothesis class $\mathcal{H} = \bigC{h_{\theta} \mid \theta  \in \mathbb{R}^d}$, where each member $h_{\theta}$ is a mapping from $\mathcal{X}_i$ to $\mathcal{Y}_i$ parametrized by $\theta$.  With $z_i \defeq (x_i, y_i) \in \mathcal{Z}_i$, we use  $\ell(z_i, h_{\theta})$, shorthand for $\ell(y_i, h_{\theta}(x_i))$, to represent the cost of predicting $h_{\theta}(x_i)$ when the ground-truth label is $y_i$.  For example, if $h_{\theta}(x_i) = \theta^{\mathsf{T}} x_i $ and $y_i \in \mathbb{R}$, a square loss $\ell(z_i, h_{\theta}) = \ell(y_i,  h_{\theta}(x_i))=(\theta^{\mathsf{T}} x_i-y_i)^{2}$ can be considered.
%For example, if $y_i$ is real-valued, we can employ the linear hypothesis $h_{\theta}(x_i) = \theta^\top x_i$ and the square loss $\ell(z_i, h_{\theta}) = (\theta^\top x_i - y_i)^2$. 
In FL, all clients collaborate with a server to find a global model $\theta$ such that the weighted sum of risks is minimized:
\begin{align}
\min_{\theta \in \mathbb{R}^d} \SumLim{i=1}{m} \lambda_i \mathbf{E}_{Z_i \sim P_i} \bigS{\ell(Z_i, h_{\theta})},
\label{Prob:FL_total_risk}
\end{align}
where $\mathbf{E}_{Z_i \sim P_i} \bigS{\ell(Z_i, h_{\theta})}$ is client $i$'s expected risk and $\lambda_i \geq 0$ represents the relative ``weight'' of client $i$ satisfying $\SumNoLim{i=1}{m} \lambda_i = 1$.  Therefore, $\lambda \defeq [\lambda_1, \ldots, \lambda_m]^\top$ belongs to the simplex $\Delta \defeq \set{\lambda \in \mathbb{R}^m: \lambda \succcurlyeq 0 \text{~and~} \lambda^\top \1_m = 1}$. Define by ${P}_{\lambda} := \SumNoLim{i=1}{m} \lambda_i {P}_{i}$  the mixed clients' distribution over $m$ domains $\mathcal{Z} \defeq \set{\mathcal{Z}_1, \ldots, \mathcal{Z}_m}$.   We denote by  $Z \sim  {P}_{\lambda}$ a random data point $Z$ generated by ${P}_{\lambda}$, which means that the  domain of client $i$ is chosen with probability   $\mathbf{P}(Z = Z_i) = \lambda_i$ first, then a data point $z_i \in \gZ_i$ is selected with  probability $\mathbf{P}(Z_i = z_i)$,   $Z_i \sim P_i$.  

While the underlying distributions $P_i$ are unknown, clients have access to  finite observations $z_i \in [n_i]$. We abuse the notation $[n_i]$ to denote the set of  client $i$'s both observable data points and their  indexes. Let $\widehat{P}_{n_i} \defeq \frac{1}{n_i} \SumNoLim{z_i \in [n_i]}{} \delta_{z_{i}}$ be the empirical distribution of $P_i$, where $\delta_{z_{i}}$ is the Dirac point mass at $z_i$. In general, we use the notation $~\widehat{}~$ for quantities that are dependent on training data. Define by $\widehat{P}_{\lambda} := \SumNoLim{i=1}{m} \lambda_i \widehat{P}_{n_i}$  the mixed empirical distribution of $n \defeq \SumNoLim{i=1}{m} n_i$  training data from $m$ clients. The  empirical risk minimization (ERM) problem of \cref{Prob:FL_total_risk} is:
\begin{align}
\label{Prob:FL_ERM}
\min_{\theta \in  \mathbb{R}^d}
\biggC{\mathbf{E}_{Z \sim \widehat{P}_{\lambda}} 
	\bigS{\ell(Z,  h_{\theta})} =
	\SumLim{i=1}{m}
	\lambda_i
	\mathbf{E}_{Z_i \sim \widehat{P}_{n_i}}
	\bigS{\ell(Z_i,  h_{\theta})}
	= \SumLim{i=1}{m}
	\frac{n_i}{n} 
	\biggP{\frac{1}{n_i}
		\SumLim{z_i \in [n_i]}{}
		\ell(z_{i},h_{\theta})}
}, 
\end{align}
where $\lambda_i = n_i / n$ is typically chosen in the ERM of the standard FL \citep{mcmahan_communication-efficient_2017}. A detailed discussion on how to choose more appropriate values for $\lambda_i$ is found later in \Cref{sec:lambda,sec:experiment}.

\subsection{Wasserstein Robust Risk in Federated Learning}
\label{subsec:wasserstein_risk_fl}

Models resulting from \cref{Prob:FL_ERM} have been shown to be vulnerable to adversarial attacks and to lack of robustness to distribution shifts. We consider a robust variant of the ERM framework involving the worst-case risk with respect to the $p$-Wasserstein distance between two probability measures. Given a set $\mathcal{Z}$, define $d : \mathcal{Z} \times \mathcal{Z} \rightarrow [0, \infty)$ to be the cost of ``transportation'' between its two points.\footnote{The function $d$ must satisfy non-negativity, lower semi-continuity and $d(z, z) = 0, \forall z \in \mathcal{Z}$.} Suppose $P$ and $Q$ are two distributions on $\mathcal{Z}$. Let $\Pi(P, Q)$, called their couplings, be the set of joint probability measures $\pi$ on $\mathcal{Z} \times \mathcal{Z}$ whose marginals are $P$ and $Q$. In other words, $\pi(A, \mathcal{Z}) = P(A)$ and $\pi(\mathcal{Z}, A) = Q(A), \forall A \subset \gZ$. The {$p$-Wasserstein distance} between $P$ and $Q$ is defined as
\begin{align}
W_p(P, Q) = \inf_{\pi \in \Pi(P, Q)} \bigP{\mathbf{E}_{(Z, Z') \sim \pi}\bigS{d^p(Z, Z')}}^{1/p}.
\label{E:Wasserstein_distance}
\end{align}
This distance represents the minimum cost of transporting one distribution to another, where the cost of moving a unit point mass is determined by the ground metric on the space of uncertainty realizations. In this work, we mainly work with $p=2$. 

%In computer science, this distance is often aptly termed the ``earth mover's distance'' \citep{rubner_earthmove_2000}. 
Let $\gB_p(P, \rho) \defeq \set{Q: W_p(P, Q) \leq \rho}$ denote the {Wasserstein ball}  centered at $P$ (i.e., nomimal distribution) and having radius $\rho \geq 0$. We modify \cref{Prob:FL_ERM} into  the following Wasserstein robust risk minimization in FL:
\begin{align}
\text{ \OurAlg:  } \underset{\theta  \in \mathbb{R}^d}{\mbox{min }} \BigC{\sup\nolimits_{Q \in \gB_p(\widehat{P}_{\lambda}, \rho)} \mathbf{E}_{Z' \sim Q} \bigS{\ell(Z',  h_{\theta})}}.
\label{eq:3_2_2}
\end{align}
%This is a DRO problem and can be viewed as a two-player zero-sum game between a decision-maker and an adversary: the decision-maker aims to minimize the expected loss function, while the adversary tries to maximize the risk by replacing the underlying distribution with another from an ambiguity set. Ideally, by solving the worst-case optimization problem, the risk under all distributions in the ambiguity set will be minimized. 
There are several merits to this framework. First, the \emph{ambiguity set} $\gB_p( \widehat{P}_{\lambda}, \rho)$ contains all (continuous or discrete) distributions $Q$ that can be converted from the (discrete) nominal distribution $\widehat{P}_{\lambda}$ at a bounded transportation cost $\rho$. Second, Wasserstein distances can be approximated  from the samples. Based on  the non-asymptotic convergence results of ~\citet{fournier_rate_2013}, we can specify a suitable value for $\rho$ to probabilistically bound $W_p(P, Q)$ by the distance between their empirical distributions $W_p(\widehat{P},\widehat{Q})$ (e.g., for multi-source domain adaptation). 
\comment{Third, with a reasonable adjustment of ambiguity sets, the solution to WDRO problems promises to put a low risk on out-of-sample or out-of-distribution data. If the value of $\rho$ in \cref{eq:3_2_2} is carefully calibrated, we can estimate the worst-case risk of the unknown true distribution $P$ in the ambiguity set.
}

In any robust optimization problem, the ambiguity set  is a key ingredient to defining the level of robustness. We will compare \OurAlg in \cref{eq:3_2_2} with other approaches in terms of their ambiguity set, showing that the Wasserstein ambiguity set can easily be adjusted to cover other ambiguity sets, making \OurAlg more general and flexible than existing methods.

\begin{wrapfigure}{r}{0.36\textwidth} 
	\vspace{-20pt}
	\begin{center}
		\includegraphics[width=0.28\textwidth]{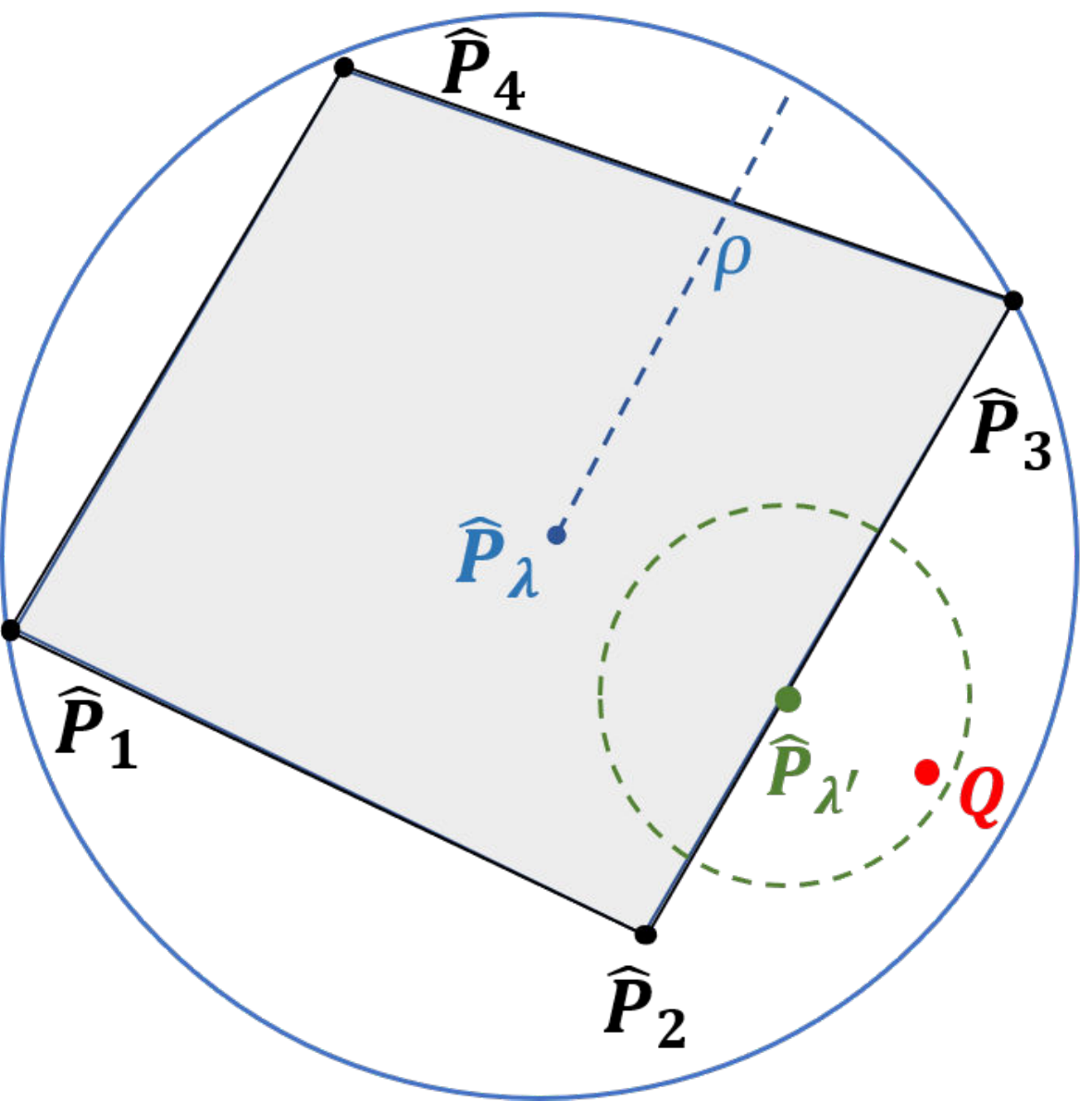}
	\end{center}
	
	\caption{Example of four FL clients with empirical data distributions $\widehat{P}_1, \ldots, \widehat{P}_4$. The shaded area (Agnostic FL's ambiguity set)  is covered by the blue ball with radius $\rho$ and centered at $\widehat{P}_{\lambda}$ (Wasserstein ambiguity set). For domain adaptation with $Q$ as a target domain, the nominal distribution (multi-source domain) is shifted to  $\widehat{P}_{\lambda'}$ such that $W_2(\widehat{P}_{\lambda'}, Q)$ is minimal.} 
	\label{F:wball} 
	\vspace{-25pt}
\end{wrapfigure}

%\begin{figure}
%    \centering
%    \captionsetup{width=.99\columnwidth}
%	\includegraphics[width=0.2\textwidth]{figures/Section3_3_Figure1.pdf}
%	\caption{Example of four FL clients with empirical data distributions $\widehat{P}_1, \ldots, \widehat{P}_4$. The shaded area (Agnostic FL's ambiguity set)  is covered by the blue ball with radius $\rho$ and centered at $\widehat{P}_{\lambda}$ (Wasserstein ambiguity set). For domain adaptation with $Q$ as a target domain, the nominal distribution (multi-source domain) is shifted to  $\widehat{P}_{\lambda'}$ such that $W_2(\widehat{P}_{\lambda'}, Q)$ is minimal.} 
%	\label{F:wball} 
%\end{figure}

\textbf{Agnostic  FL.} Using this approach, existing techniques \citep{mohri_agnostic_2019,deng_distributionally_2020} minimize the worst-case loss 
\begin{align*}
\underset{\lambda \in \Delta}{\mbox{max }} \mathbf{E}_{Z \sim \widehat{P}_{\lambda}} \bigS{\ell(Z,  h_{\theta})}, 
\end{align*}
hence its distributional ambiguity set is $\gQ_{\Delta} \defeq \bigC{\widehat{P}_{\lambda} :  \lambda \in \Delta}.$  While Agnostic FL's ambiguity set is  the static convex hull of $\bigC{\widehat{P}_{n_i}}_ {i\in [m]}$,  \OurAlg's ambiguity set $\gB( \widehat{P}_{\lambda}, \rho)$  can be adjusted by controlling the robustness level $\rho$ and by positioning the ball center using $\lambda$, which is useful for domain adaptation. Furthermore,  by controlling  $\rho$ and $\lambda$, we can flexibly \emph{enlarge  $\gB( \widehat{P}_{\lambda}, \rho)$ to cover $\gQ_{\Delta}$, or shrink it down to sufficiently include an arbitrary distribution $Q$ that is outside of the convex hull for domain adaptation} (see \Cref{F:wball}).
 
\textbf{Adversarial robust  FL.} \citet{reisizadeh_robust_2020} combined a general affine covariate shift in standard adversarial robust training with FL. Most existing techniques under this approach \citep{goodfellow_explaining_2015,  kurakin_adversarial_2017, carlini_towards_2017, madry_towards_2019, tramer_ensemble_2020} define an adversarial perturbation $u$ at a data point $Z$ and minimize the worst-case loss over all perturbations: $\max\nolimits_{u \in \gU} \mathbf{E}_{Z \sim \widehat{P}_{\lambda}} \big{[}\ell(Z + u,  h_{\theta})\big{]}$, where the ambiguity set is $\gU \defeq \set{u \in \mathbb{R}^{d+1} : \norm{u} \leq \eps}$.  In \Cref{Apx:adversarial}, we show that the Wasserstein ambiguity set can also contain the perturbation points induced by the solution to this adversarial robust training problem. 

\comment{\textbf{Adversarial robust  FL.} Using this approach,~\citet{reisizadeh_robust_2020} combines a general  affine covariate shift  in standard adversarial robust training with FL. Most existing techniques  for learning robust models in this approach~\citep{goodfellow_explaining_2015, papernot_limitations_2015, kurakin_adversarial_2017, carlini_towards_2017, madry_towards_2019, tramer_ensemble_2020} define an adversarial perturbation $u$ at a data point $Z$, and minimize the following worst-case loss over all possible perturbations
	\begin{equation}
	\underset{u \in \gU}{\mbox{max }} \mathbf{E}_{Z \sim \widehat{P}_{\lambda}} \Big{[}\ell(Z + u,  h_{\theta})\Big{]},  \label{E:U}
	\end{equation}
	where the ambiguity set  $\, \gU \defeq \set{u \in \mathbb{R}^{d+1} : \norm{u} \leq \eps}$.  To compare this approach with Wasserstein-robust FL,  we  relate the above problem to its counterpart defined  in the probability space of input as follows
	\begin{equation}
	\underset{\tilde{Q} \in \gQ(\eps)}{\mbox{max }} \mathbf{E}_{\tilde{Z} \sim \tilde{Q}} \Big{[}\ell(\tilde{Z},  h_{\theta})\Big{]},\label{E:Q}
	\end{equation}
	$\text{where } \gQ(\eps) \defeq \set{\tilde{Q} : \mathbb{P} \bigS{\norm{\tilde{Z} - Z} \leq \eps} = 1, Z \sim \widehat{P}_{\lambda}, \tilde{Z} \sim \tilde{Q}}.$
	Considering $u^*$ as a solution to problem \cref{E:U},  we see that the distribution $Q'$ of perturbation points (i.e., $\tilde{Z} \defeq Z+u^* \sim {Q}'$) in  problem \cref{E:U}   belongs to the feasible set  $\gQ(\eps)$ in problem \cref{E:Q} (If not, then $\mathbb{P} \bigS{\norm{u^*} \leq \eps} <1$, a contradiction). Next, consider an arbitrary  distribution $\tilde{Q} \in \gQ(\eps)$ in problem \cref{E:Q},  with any $\tilde{Z} \sim \tilde{Q}$ and $Z \sim \widehat{P}_{\lambda}$, we have
	\begin{align*}
	\norm{\tilde{Z} - Z} \stackrel{\text{w.p.1}}{\leq }\eps  \implies \mathbf{E}_{Z \sim \widehat{P}_{\lambda}, \tilde{Z} \sim  \tilde{Q}} \bigS{\norm{\tilde{Z} - Z}} \leq \eps \implies \inf_{\pi \in \Pi(\widehat{P}_{\lambda}, \tilde{Q})} \mathbf{E}_{(Z, Z') \sim \pi} \bigS{\norm{\tilde{Z} - Z}} \leq \eps,
	\end{align*}
	which implies that $ W_1(\widehat{P}_{\lambda}, \tilde{Q}) \leq \eps, \forall \tilde{Q} \in \gQ(\eps)$, and thus  $\gQ(\eps) \subset \gB_1(\widehat{P}_{\lambda}, \eps)$. We have shown that the Wasserstein ambiguity set contains the perturbation points induced by the solution to the adversarial robust training problem \cref{E:U}. 
}

%\blue{Apart from the differences especially in terms of ambiguity set, both adversarial and agnostic robust approaches use the common technique of saddle-point convergence analysis for the min-max problem. On the other hand, we will show that \OurAlg can enable a standard FL algorithm design, i.e., local SGD~\citep{stich_local_2019, woodworth_is_2020, khaled_tighter_2020}, using  the approximation of its duality problem.} 

\subsection{\OurAlg: Algorithm Design and Convergence Analysis} 
\label{sec:convergence}

%We do not try to find the optimal $\gamma^*$ ($\gamma$ is fixed value). Each client will do its local optimization with surrogate function $\phi_\gamma(Z_i,\theta)$ parametrized by $\gamma$, instead of $ \ell(Z_i,  h_{\theta}) $
%
%\begin{align}
%     & \underset{\theta}{\mbox{min }} \Big{\{}  \mathbf{E}_{Z \sim \widehat{P}_{\lambda}}\Big{[}\phi_{\gamma}(Z, f_{\theta})\Big{]} \Big{\}} \nonumber\\
%     & = \underset{\theta}{\mbox{min }} \BigC{\sum_{i=1}^m \lambda_i \, \mathbf{E}_{Z_i \sim \widehat{P}_{n_i}}\Big{[}\phi_\gamma(Z_i, f_{\theta}) | Z = Z_i \Big{]}} \nonumber\\
%     & = \underset{\theta}{\mbox{min }} \BigC{ \sum_{i=1}^m \lambda_i \cdot \frac{1}{n_i}\sum_{z_i \in \mathcal{Z}_i}\phi_\gamma(z_i, f_{\theta})} \nonumber 
%\end{align}
%
%We need to have: $\blue{\nabla_{\!\theta} \,  \phi_\gamma(z^i_j,\theta)}$
%
%\begin{equation}
%\begin{split}
%& \nabla_{\! \theta} \, \phi_\gamma(z_i, f_{\theta}) = \nabla_{\! \theta}  \, \ell(z_i^*, h_{\theta}) \\
%\mbox{where } & z_i^* = \underset{\zeta}{\mbox{argmax}}\Big{\{} \ell(\zeta,  h_{\theta}) - \gamma d(\zeta,z_i) \Big{\}} \\
%& \zeta^{k+1} = \zeta^k + \nabla_{\!\zeta}\Big{[} \ell(\zeta,  h_{\theta}) - \gamma d(\zeta, z_i)\Big{]}_{\zeta = \zeta^k}
%\end{split}
%\end{equation}

The original form of \OurAlg in \cref{eq:3_2_2} is not friendly for distributed algorithm design. Fortunately, the \emph{Wasserstein robust risk} (or $\gB(\widehat{P}_{\lambda},\rho)$-worst-case risk) has its dual formulation as follows \citep{gao_distributionally_2016, sinha_certifying_2020}
\begin{align} \label{E:duality_Wass}
	\sup_{Q \in \gB(\widehat{P}_{\lambda},\rho)} \mathbf{E}_{Z' \sim Q} \bigS{\ell(Z',  h_{\theta})}
	= \inf_{\gamma \geq 0} \bigC{ \gamma \rho^2 +  \mathbf{E}_{Z \sim \widehat{P}_{\lambda}}\bigS{\phi_{\gamma}(Z, {\theta})}},
\end{align}
where $\phi_\gamma(z_i,  {\theta}) \defeq  \sup\nolimits_{\zeta \in \mathcal{Z}} \bigS{\ell(\zeta, h_\theta) - \gamma d^2(\zeta, z_i)}$, and $d^{2}(z, z') = \|x - x'\|^2 + \kappa \abs{y - y'}^2, \kappa > 0$. The crux of using the dual is that the inner supremum problem (finding $\phi_\gamma$) is easily solvable when its objective is well-conditioned: if $z \mapsto \ell(z, \cdot)$ is $L_{zz}$-smooth and $z \mapsto d(z,\cdot)$ is $1$-strongly convex, setting $\gamma > L_{zz}$ ensures that $\zeta \mapsto {\ell(\zeta, h_\theta) - \gamma d^2(\zeta, z)}$ is strongly concave, and using gradient ascent for the inner supremum problem (for finding $\phi_\gamma$) enjoys linear convergence. Therefore, instead of finding the optimal $\gamma^*$ to \eqref{E:duality_Wass} that may not satisfy $\gamma^* > L$, we  set $\gamma > L$  as a control hyperparameter to ensure there exists a unique solution $z_i^*$  to $\sup\nolimits_{\zeta \in \mathcal{Z}} \bigS{\ell(\zeta, h_\theta) - \gamma d^2(\zeta, z_i)}$ for each $z_i$, and thus $\nabla_{\theta} \phi_{\gamma}(z_i, \theta) =  \nabla_{\theta} \ell({z}_i^*, h_\theta)$ \citep[Lemma 1]{sinha_certifying_2020}. We will characterize the effect of \emph{sub-optimality of $\gamma$ to the excess risk} in Lemma~\ref{lem:excess_risk}. Then, we obtain the following client-decomposable  problem, which is amenable to distributed algorithm design:
%\blue{(The hyperparameter $\gamma$ characterizes the robustness budget in this setting; see \Cref{F:p_gamma} in \Cref{sec:experiment} for a characterization of the relationship between $\gamma$ and $\rho$, the worst-case perturbation in distributionally robust learning.)}
\begin{align} \label{E:surrogate}
	\min_{\theta \in \mathbb{R}^d} \biggC{\mathbf{E}_{Z \sim \widehat{P}_{\lambda}} \bigS{\phi_{\gamma}(Z, {\theta})} 
		= \SumLim{i=1}{m}
		\lambda_i
		\mathbf{E}_{Z_i \sim \widehat{P}_{n_i}}
		\bigS{\phi_{\gamma}(Z_i,  {\theta})}}. 
\end{align}
This motivates the development of \Cref{alg:ouralg} for solving \cref{E:surrogate}. The structure of \OurAlg is similar to FedAvg with $T$ communication rounds and three additional key components. First, \emph{client sampling} (line \ref{alg:client_sampling}) refers to the partial participation of clients in each global round. Second, each client performs $K$ \emph{local steps} (line \ref{alg:local_steps_start}) before sending  its local model to the server. Finally, \emph{stochastic approximation} of a client's gradient using a mini-batch (lines \ref{alg:stochastic_grad_a} and \ref{alg:stochastic_grad_c}) is necessary when the data size is large.  The main difference between \OurAlg and FedAvg is that \OurAlg aims to minimize the risk with respect to the surrogate loss $\phi_{\gamma}$, rather than $\ell$.
We show that the convergence of \OurAlg can be similarly characterized as that of FedAvg, the de facto FL algorithm based on local SGD updates~\citep{mcmahan_communication-efficient_2017}. In FedAvg optimization, we seek to establish the convergence when using the original loss function $\ell$. On the other hand, in \OurAlg the convergence is with respect to the surrogate loss $\phi_{\gamma}$, through which the local and global risks are defined by $F_i(\theta) \defeq \mathbf{E}_{Z_i \sim {P}_i}\bigS{\phi_{\gamma}(Z_i, {\theta})}$ and $F(\theta) \defeq \mathbf{E}_{Z \sim {P}_{\lambda}}\bigS{\phi_{\gamma}(Z, {\theta})}$, respectively.

{\small
\begin{algorithm}[t]
	\DontPrintSemicolon
	\SetAlgoNlRelativeSize{0}
	\SetNlSty{text}{}{:}
	%\For(\tcp*[f]{Global rounds}){$t=0,1,...,T-1$}{
	\nl \For(\tcp*[f]{Global rounds}){$t=0,1,...,T-1$}{
		Sample a subset of clients $S_t \subset \bigS{m}$ \label{alg:client_sampling} \;
		
		\For{client $i \in S_t$ in parallel}{
			Set local parameters: $\theta_i^{(t, 0)}= \theta^{t}$  % \tcp*[r]{Communication}

			\For(\tcp*[f]{Local rounds}){$k=0,1,...,K-1$}{ \label{alg:local_steps_start}
				Sample a mini-batch $\gD_i$ from client $i$'s dataset \label{alg:stochastic_grad_a} \;
				
				$\theta_i^{(t, k+1)} = \theta_i^{(t, k)} - \frac{\eta}{|\gD_i|} \underset{{z_i \in \gD_i}}{\sum} \nabla_{\theta} \phi({z}_i, \theta_i^{(t, k)})$ \label{alg:stochastic_grad_c}
			}
			Send $\theta_i^{(t, K)}$ to server %\tcp*[r]{Communication}
		}
		%\textbf{Server:} update $\theta^{(t+1)} = \SumNoLim{i \in S_t}{}\lambda_i \theta_i^{(t, k^*)}$ \;
		{\textbf{Server:} update $\theta^{(t+1)} = \SumNoLim{i \in S_t}{}\lambda_i \theta_i^{(t, K)} / \SumNoLim{i \in S_t}{} \lambda_i$}
	}
	\caption{Local SGD for \OurAlg}
	\label{alg:ouralg}

\end{algorithm}
}
	\vspace{-5pt}
We first make the following assumptions, common to analyses of Wasserstein-robust optimization~\citep{sinha_certifying_2020}. Unless stated otherwise, all norms are the Euclidean norm. 

\begin{assumption} \label{Assumption:continuous_distance}
	The function $d: \mathcal{Z} \times \mathcal{Z} \rightarrow \mathbb{R}_{+}$ is continuous, and $d(\cdot, z_{0})$ is $1$-strongly convex, $\forall z_{0} \in \mathcal{Z}$.
\end{assumption}

\begin{assumption} \label{Assumption:Lip_cont}
	The loss function $\ell: \mathcal{Z} \times \mathcal{H} \rightarrow \mathbb{R}$ is Lipschitz continuous as follows
%    \begin{enumerate}[label=(\alph*)]
%        \item $\norm{ \ell(z, h_{\theta}) - \ell(z', h_{\theta})} \leq L_{z} \norm{z - z'}$
%        \item $\norm{ \ell(z, h_{\theta}) -  \ell(z, h_{\theta'})} \leq L_{\theta} \norm{\theta - \theta'}$
%    \end{enumerate}
 	\begin{align*}
 	(a)\; \norm{ \ell(z, h_{\theta}) - \ell(z', h_{\theta})} \leq L_{z} \norm{z - z'}; \quad \quad \quad
 	(b)\;  \norm{ \ell(z, h_{\theta}) -  \ell(z, h_{\theta'})} \leq L_{\theta} \norm{\theta - \theta'}.
 	\end{align*}
\end{assumption}

\begin{assumption} \label{Assumption:smooth_loss}
	The loss function $\ell: \mathcal{Z} \times \mathcal{H} \mapsto \mathbb{R}$ is Lipschitz  smooth as follows
%	\begin{enumerate}[label=(\alph*)]
%    	\item $\norm{\nabla_{\theta} \ell(z, h_{\theta}) - \nabla_{\theta} \ell(z, h_{\theta'})} \leq  L_{\theta\theta} \norm{\theta - \theta'}$
%    	\item $\norm{\nabla_{z} \ell(z, h_{\theta}) - \nabla_{z} \ell(z', h_{\theta})}  \leq  L_{zz} \norm{z - z'}$
%    	\item $\norm{\nabla_{\theta} \ell(z, h_{\theta}) - \nabla_{\theta} \ell(z', h_{\theta})}  \leq  L_{\theta z} \norm{z - z'}$
%        \item $\norm{\nabla_{z} \ell(z, h_{\theta}) - \nabla_{z} \ell(z, h_{\theta'})}  \leq  L_{z\theta} \norm{\theta - \theta'}$
%	\end{enumerate}
	\begin{align*}
	&(a)\; \norm{\nabla_{\theta} \ell(z, h_{\theta}) - \nabla_{\theta} \ell(z, h_{\theta'})} \leq  L_{\theta\theta} \norm{\theta - \theta'}; \quad 
	(b)\; \norm{\nabla_{z} \ell(z, h_{\theta}) - \nabla_{z} \ell(z', h_{\theta})}  \leq  L_{zz} \norm{z - z'};\\
	&(c)\; \norm{\nabla_{\theta} \ell(z, h_{\theta}) - \nabla_{\theta} \ell(z', h_{\theta})}  \leq  L_{\theta z} \norm{z - z'}; \quad 
	(d)\;\norm{\nabla_{z} \ell(z, h_{\theta}) - \nabla_{z} \ell(z, h_{\theta'})}  \leq  L_{z\theta} \norm{\theta - \theta'}.
\end{align*}
\end{assumption}
Given \Cref{Assumption:smooth_loss}, it has been shown  that the mapping $\theta \mapsto \phi_{\gamma}(\cdot, {\theta})$ is $L$-smooth with $L = L_{\theta \theta} + \frac{L_{\theta z} L_{z \theta}}{\gamma - L_{zz}}, \gamma > L_{zz}$ (\citet{sinha_certifying_2020}, more detail in \Cref{Lemma:f_smooth} in \Cref{proof:thr1}). In addition, we make the following assumptions common to FL analysis \citep{wang_field_2021}. 
%These assumptions are imposed on the risks with respect to the original loss function $\ell$.

% Let $\varphi_{\gamma}(\zeta, h_{\theta}; z_i) = \ell(\zeta, h_{\theta}) - \gamma d(\zeta, z_i)$ and $z_i^* = \argmax_{\zeta}\varphi_{\gamma}(\zeta, h_{\theta}; z_i)$. Therefore, $\phi(z_i, h_{\theta}) = \varphi(z^*_i, h_{\theta}; z_i)$. Denote $g_i(\theta)$ the stochastic gradient of $\theta$ with respect to the loss function on user $i$.

\begin{assumption}
	\label{Assumption:bounded_variance}
	The unbiased stochastic approximation of $\nabla F_i(\theta)$, denoted by $g_{\phi_i}(\theta) \defeq \nabla_{\theta} \phi_{\gamma}(z_i, {\theta}), z_i \sim {P}_i$, has $\sigma^2$-uniformly bounded variance, i.e.,  $	\mathbf{E}\BigS{\norm{g_{\phi_i}(\theta) - \nabla F_i(\theta)}^2} \leq \sigma^2.$
%	\begin{align*}
%		\mathbf{E}\BigS{\norm{g_{\phi_i}(\theta) - \nabla F_i(\theta)}^2} \leq \sigma^2.
%	\end{align*}
\end{assumption}

%\begin{assumption} 
%	\label{Assumption:bounded_gradient}
%	%Let $L_i(\theta) \defeq \mathbf{E}_{Z_i \sim P_i}\BigS{\ell(Z, h_{\theta})}$ be the local (client) population risk, and $L(\theta) \defeq \mathbf{E}_{Z \sim P_{\lambda}}\BigS{\ell(Z, h_{\theta})}$ be the global population risk.
%	The difference between the local gradient $\nabla_{\theta} L_i(\theta)$ and the global gradient $\nabla L(\theta)$ is $\Omega$-uniformly bounded, i.e.,
%	\begin{align*}
%	\max_{i} \sup_{\theta} \norm{\nabla L_i(\theta) - \nabla L(\theta)} \leq \Omega.
%	\end{align*}
%\end{assumption}

\begin{assumption} 
	\label{Assumption:bounded_gradient_surrogate}
	%Let $L_i(\theta) \defeq \mathbf{E}_{Z_i \sim P_i}\BigS{\ell(Z, h_{\theta})}$ be the local (client) population risk, and $L(\theta) \defeq \mathbf{E}_{Z \sim P_{\lambda}}\BigS{\ell(Z, h_{\theta})}$ be the global population risk.
	The difference between the local gradient $\nabla F_i(\theta)$ and the global gradient $\nabla F(\theta)$ is $\Omega$-uniformly bounded, i.e., $\max_{i} \sup_{\theta} \norm{\nabla F_i(\theta) - \nabla F(\theta)} \leq \Omega.$
%	\begin{align*}
%	\max_{i} \sup_{\theta} \norm{\nabla F_i(\theta) - \nabla F(\theta)} \leq \Omega.
%	\end{align*}
\end{assumption}
%We show (in Appendix \Cref{Lemma:bounded_variance_surrogate}) that \Cref{Assumption:bounded_variance} implies that stochastic gradient of $F_i$, denoted as $g_i$, has $\widehat{\sigma}^2$-uniformly bounded gradient, where $\widehat{\sigma}^2 = 2\sigma^2 + \frac{2 L_{\theta z}^2}{\gamma - L_{zz}} \nu$. Given this, we can frame the convergence of \OurAlg in the same way as FedAvg.
Assuming complete participation of clients in every round ($S_t = [m], \forall t$), using standard techniques in \cite{wang_field_2021}, we have:
\begin{theorem}[\OurAlg's convergence for convex loss function]\label{Thrm:Convergence}
	Let \Cref{Assumption:continuous_distance,Assumption:Lip_cont,Assumption:smooth_loss,Assumption:bounded_variance,Assumption:bounded_gradient_surrogate} hold and the mapping $\theta \mapsto \ell(z, h_{\theta})$ be convex. Denote by $\bar{\theta}^{(t, k)}$ the ``shadow'' sequence, defined as $\bar{\theta}^{(t, k)} = \SumNoLim{i=1}{m}\lambda_i \theta_i^{(t, k)}$ and by ${\theta}^*$ the {optimal solution} to $\min_{\theta \in \mathbb{R}^d} F(\theta)$. If the learning rate $\eta$ is at most
%	\begin{align*}
%	\eta = \min{\set{\frac{1}{4L}, \frac{\Lambda^{\frac{1}{2}} D}{K^{\frac{1}{2}} T^{\frac{1}{2}} \widehat{\sigma}}, \frac{D^{\frac{2}{3}}}{K^{\frac{2}{3}} T^{\frac{1}{3}} L^{\frac{1}{3}} \widehat{\sigma}^{\frac{2}{3}}}, \frac{D^{\frac{2}{3}}}{K T^{\frac{1}{3}} L^{\frac{1}{3}} \Omega^{\frac{2}{3}}}}},
%	\end{align*}
    \begin{align*}
    	%\eta \leq 
    	\min\biggC{\frac{1}{3L}, \frac{ D}{2 \sqrt{ K T\Lambda} \widehat{\sigma}}, \frac{D^{\frac{2}{3}}}{48^{\frac{1}{3}}K^{\frac{2}{3}} T^{\frac{1}{3}} L^{\frac{1}{3}} \bar{\Omega}^{\frac{2}{3}}},\frac{D^{\frac{2}{3}}}{40^{\frac{1}{3}} K T^{\frac{1}{3}} L^{\frac{1}{3}} \bar{\Omega}^{\frac{2}{3}}}},
	\end{align*}
	then we have
	\begin{align} %\label{eq:convergence}
	\mathbb{E}\biggS{\frac{1}{KT} \SumLim{t=0}{T-1} \SumLim{k=0}{K-1}  F(\bar{\theta}^{(t, k)}) - F({\theta}^*)}
	\leq  \gO \biggP{{\frac{ L D^2}{KT} + 
		\frac{ \widehat{\sigma} D \Lambda^{\frac{1}{2}}}{\sqrt{ KT}}} + 
	{\frac{ L^{\frac{1}{3}} \bar{\Omega}^{\frac{2}{3}} D^{\frac{4}{3}}}{K^{\frac{1}{3}} T^{\frac{2}{3}}} + 
		\frac{ L^{\frac{1}{3}} \bar{\Omega}^{\frac{2}{3}} D^{\frac{4}{3}}}{T^{\frac{2}{3}}}}}, \nonumber
	\end{align}
	where $\hat{\sigma}^2 \defeq\frac{\sigma^2 }{\abs{\gD_i}} { }$,  $D \defeq \norm{\theta^{(0)} - {\theta}^*}$ and $\Lambda \defeq { \SumNoLim{i=1}{m} \lambda_i^2}$. 
\end{theorem}

% We provide a proof of \Cref{Thrm:Convergence} in \Cref{proof:thr1}. 

%\Cref{F:convergence} in the Appendix shows the convergence of \Cref{alg:ouralg} on two datasets: MNIST and CIFAR-10. See \Cref{sec:experiment} for more detail on the experiments.

%\Cref{Thrm:Convergence} shows two types of error. The first two terms of \cref{eq:convergence} are errors with synchronous SGD ($K=1$). When multiple local updates are used ($K>1$), the last two terms denote the additional incurred. This additional error is proportional to $\eta^2$ and can decay at the rate of $O(1/T^{\frac{2}{3}})$. 
\section{Robust Generalization Bounds} \label{Generalization}

We show the generalization and robustness properties of \OurAlg's output by bounding its excess risk. Denote the loss class by $\gF \defeq \ell \,\circ \, \gH \defeq \set{ z \mapsto \ell(z,h), h \in \gH }$, where we  use $f  \text{ (resp. $f_{\theta}$)} \in \gF$ to represent a generic loss (resp. a loss function parametrized by $\theta$). 
\begin{definition}
	Denote the expected risk and surrogate of Wasserstein robust risk  of an arbitrary $f$, respectively, as
	\begin{align*}
		%	\mathscr{L}(\widehat{P}_{\lambda}, f_{\theta}) &\defeq	\mathbf{E}_{Z \sim \widehat{P}_{\lambda}} \bigS{\ell(Z,  h_{\theta})} \\
		%	{\mathscr{L}(P_{\lambda}, f_{\theta}) \defeq \mathbf{E}_{Z \sim P_{\lambda}} \bigS{\ell(Z,  h_{\theta})}}. \\
		%	\mathscr{L}_{\rho}^{\gamma} (\widehat{P}_{\lambda},f) := & \mathbf{E}_{Z \sim \widehat{P}_{\lambda}}\Big{[}\phi_\gamma(Z, f)\Big{]} + \gamma \rho^2 \\
		\mathscr{L}(P_{\lambda}, f)  \defeq  \mathbf{E}_{Z \sim P_{\lambda}} \bigS{\ell(Z,  h)} \quad \text{  and  } \quad
		\mathscr{L}_{\rho}^{\gamma} (P_{\lambda},f)  \defeq \mathbf{E}_{Z \sim {P_{\lambda}}}\Big{[}\phi_\gamma(Z, f)\Big{]} + \gamma \rho^2.
	\end{align*}
	Then their excess risks are defined respectively as follows
	\begin{align*}
		\mathscr{E}(P_{\lambda}, f) \defeq \mathscr{L}(P_{\lambda}, f) - \inf_{f' \in \mathcal{F}} \mathscr{L}(P_{\lambda}, f') \quad \text{  and  } \quad 
		\mathscr{E}_{\rho}^{\gamma}(P_{\lambda}, f) \defeq \mathscr{L}_{\rho}^{\gamma}(P_{\lambda}, f) - \inf_{f' \in \mathcal{F}} \mathscr{L}_{\rho}^{\gamma}(P_{\lambda}, f'). 
	\end{align*}
\end{definition}
If a distribution $Q$ is in the ambiguity set $\gB(P_{\lambda}, \rho)$, we can bound its excess risk $\mathscr{E}(Q, f)$ as follows.
\begin{lemma} \label{lem:excess_risk}
	Let \Cref{Assumption:Lip_cont} (a) holds and  $\gamma \geq L_z/\rho $. For all $f \in \gF$ and for all  $Q \in \gB(P_{\lambda}, \rho),$
	\begin{align*}
		\mathscr{E}_{\rho}^{\gamma}(P_{\lambda}, f)  - g(\rho, \gamma) &\leq \mathscr{E}(Q, f)  \leq  \mathscr{E}_{\rho}^{\gamma}(P_{\lambda}, f) + g(\rho, \gamma),
	\end{align*}
	%	where $g(\rho, \gamma) \defeq  2L_z \rho  + \abs{\gamma - \gamma^*} \cdot \max\bigC{\rho, D}, D \defeq \mathbf{E}_{Z \sim P}\Big[\sup_{\zeta \in \mathcal{Z}} d(\zeta,  Z)\Big]$.
	where $g(\rho, \gamma) \defeq 2L_z \rho   + \abs{\gamma - \gamma^*} {\rho}^2$, and $ \gamma^* \defeq \argmin_{\gamma' \geq 0} \mathscr{L}_{\rho}^{\gamma'} (P_{\lambda},f)$. 
	%	$\gamma \rho + \mathbf{E}_{Z \sim \widehat{P}_{\lambda}}\Big{[}\phi_\gamma (Z,  h_{\theta})\Big{]} $
\end{lemma}
% We provide a proof of \Cref{lem:excess_risk} in \Cref{proof:excess_risk}.
\begin{remark}
	\Cref{lem:excess_risk} shows that the lower and upper bounds for $\mathscr{E}(Q, f)$ can be analyzed using  $\mathscr{E}_{\rho}^{\gamma}(P_{\lambda}, f)$  and a two-component error term $g(\rho, \gamma)$ capturing the impact of the control parameters $\rho$ and $\gamma$. Particularly, the first component, $2L_z \rho$, says that when $\rho$ is increased -- to allow for a larger  Wasserstein distance between the nominal $P_{\lambda}$ and any worst-case distribution $Q$ -- the difference between the excess risks $\mathscr{E}(Q, f)$  and  $\mathscr{E}_{\rho}^{\gamma}(P_{\lambda}, f)$   increases, and this error is amplified at most by the Lipschitz constant $L_z$ of the mapping $z \mapsto \ell(z, \cdot)$. The second component, $\abs{\gamma - \gamma^*} \rho^2$, addresses the sub-optimality error of a chosen value of $\gamma$, which is amplified when  $\gamma$ is drifted away from the optimal $\gamma^*$. Note that $\mathscr{L}_{\rho}^{\gamma^*} (P_{\lambda},f)$ is the same as $\gB(P_{\lambda},\rho)$-worst-case risk thanks to the strong duality in \cref{E:duality_Wass}, obtained with $\rho > 0$. 
\end{remark}

Denote by $\widehat{\theta}^{\epsilon} \in \Theta$    an  $\varepsilon$-minimizer to  the surrogate ERM, i.e., $\mathbf{E}_{Z \sim \widehat{P}_{\lambda}} \bigS{\phi_\gamma(Z, f_{\widehat{\theta}^{\varepsilon}})}  \leq \inf\nolimits_{\theta \in \Theta}  {\mathbf{E}_{Z \sim \widehat{P}_{\lambda}} \bigS{\phi(Z,  f_{\theta})}   } + \varepsilon$, where $\Theta \subset \mathbb{R}^d$ is a parameter class, we obtain the following. 

%Denote by $\widehat{\theta}^{*}$ as  a solution  to the surrogate ERM, i.e.,  $ \widehat{\theta}^{*} = \argmin\nolimits_{\theta \in \Theta}  {\mathbf{E}_{Z \sim \widehat{P}_{\lambda}} \bigS{\phi(Z,  f_{\theta})} }$, where $\Theta \subset \mathbb{R}^d$ is a class of parameter, we obtain the following result. 
\begin{theorem}[Robust generalization bounds] \label{thrm:excess_risk} 
	Let  \Cref{Assumption:Lip_cont,Assumption:smooth_loss} hold, $\gamma \geq \max\bigC{L_{zz}, L_z/\rho}$, and $\abs{\ell(z,h) } \leq M_{\ell}$.  We have the following result for all $Q \in \gB(P_{\lambda}, \rho)$
	\begin{align*}
	    \mathscr{E}(Q, {f}_{\widehat{\theta}^{\varepsilon}}) \leq    \sum_{i=1}^{m} \lambda_i \BiggS{\frac{48 \mathscr{C}(\Theta)}{\sqrt{n_i}} + 2 M_{\ell} \sqrt{\frac{2 \log (2  m / \delta)}{n_i}}}
	     +\varepsilon + g(\rho, \gamma)
	\end{align*}
% 	\[ 
% 	\mathscr{E}(Q, {f}_{\widehat{\theta}^{\varepsilon}}) \leq  \sum_{i=1}^{m} \lambda_i \BiggS{\frac{48 \mathscr{C}(\Theta)}{\sqrt{n_i}} + 2 M_{\ell} \sqrt{\frac{2 \log (2  m / \delta)}{n_i}}} +\varepsilon + g(\rho, \gamma)
% 	\]
	with probability at least $1 - \delta$, where
	$\mathscr{C}(\Theta):= L_{\theta} \int_{0}^{\infty} \sqrt{\log \mathcal{N}\left(\Theta,\|\cdot\|_{\Theta}, \epsilon \right)} \mathrm{d} \epsilon \,$ and  $\mathcal{N}\left(\Theta,\|\cdot\|_{\Theta}, \epsilon \right)$ denotes the $\epsilon$-covering number of $\Theta$ w.r.t  a  metric $\|\cdot\|_{\Theta}$ as the norm on $\Theta$.
\end{theorem}

The proof of \Cref{thrm:excess_risk} leverages \Cref{lem:excess_risk} to bound $\mathscr{E}(Q, f), \forall Q \in \gB(P_{\lambda}, \rho),$  based on the bound of  $\mathscr{E}_{\rho}^{\gamma}(P_{\lambda}, {f}_{\widehat{\theta}^{\varepsilon}})$. The result shows using \OurAlg to minimize the surrogate of {Wasserstein robust empirical risk} also controls robustness and generalization. For example, $\mathcal{H} = \bigC{\innProd{\theta, \cdot}, \theta \in \Theta}$, $\Theta = \bigC{\theta \in \mathbb{R}^d:  \norm{\theta}_2 \leq C}$.  The diameter of $\Theta$ is $ \sup_{\theta, \theta' \in \Theta} \norm{\theta - \theta'} = 2C$, thus $\mathcal{N}\left(\Theta,\|\cdot\|_2, \epsilon \right) = \nbigP{1 + {2C}/{\epsilon}}^d$, and  $\mathscr{C}(\Theta) \leq 3 C L_{\theta} \sqrt{d}/2$ \citep{lee_minimax_2018}.

% We provide the proof  in \Cref{proof:thrm_excess_risk}. 
% We sketch the  proof as follows: first bound  $\mathscr{E}_{\rho}^{\gamma}(P_{\lambda}, {f}_{\widehat{\theta}^{\varepsilon}})$ using standard excess risk decomposition and uniform convergence with Rademacher complexity. Then leverage the bound in \Cref{lem:excess_risk} to bound $\mathscr{E}(Q, f), \forall Q \in \gB(P_{\lambda}, \rho),$  based on the bound of  $\mathscr{E}_{\rho}^{\gamma}(P_{\lambda}, {f}_{\widehat{\theta}^{\varepsilon}})$. The result shows using \OurAlg to minimize the surrogate of {Wasserstein robust empirical risk} also controls robustness and generalization. For example, $\mathcal{H} = \bigC{\innProd{\theta, \cdot}, \theta \in \Theta}$, $\Theta = \bigC{\theta \in \mathbb{R}^d:  \norm{\theta}_2 \leq C}$.  The diameter of $\Theta$ is $ \sup_{\theta, \theta' \in \Theta} \norm{\theta - \theta'} = 2C$, thus $\mathcal{N}\left(\Theta,\|\cdot\|_2, \epsilon \right) = \nbigP{1 + {2C}/{\epsilon}}^d$, and  $\mathscr{C}(\Theta) \leq 3 C L_{\theta} \sqrt{d}/2$ \citep{lee_minimax_2018}. 

%\begin{align}
%	\mathscr{C}(\Theta) \leq L_{\theta} \int_{0}^{\infty} \sqrt{d \log \BigP{1 + \frac{2 C}{\epsilon}} } \mathrm{d} \epsilon  \leq L_{\theta} \int_{0}^{D} \sqrt{d \log \BigP{ \frac{3 C}{\epsilon}} } \mathrm{d} \epsilon\leq 3 C L_{\theta} \sqrt{d}/2. 
%\end{align}

Generally, the radius of Wasserstein ball $\rho$ can be considered a hyperparameter that needs fine-tuning (e.g., through cross-validation). In principle, $\rho$ should not be too large to become over-conservative, which can hurt the empirical average performance, but also not too small to become similar to the ERM, and thus can lack robustness. 
From a statistical  standpoint, we are interested in learning how to scale  $\rho$ w.r.t. the sample size $n_i, i \in [m]$,  such that the  generalization of the \OurAlg solution $\widehat{\theta}^{\varepsilon}$ w.r.t.  the true distribution $P_{\lambda}$ is guaranteed, while still ensuring robustness w.r.t. all distributions inside the Wasserstein ball. Using the result from \citet{fournier_rate_2013} showing that  $\widehat{P}_{n_i}$ converges in Wasserstein distance to the true $P_i$ at a specific rate, we obtain:

\comment{\begin{proposition}[Measure concentration~{\citep[Theorem~2]{fournier_rate_2013}}]\label{pro:WassRate}
	Let $P$ be a probability distribution on a bounded set $\mathcal{Z}$. Let $\widehat{P}_{n}$ denote the empirical distribution of $Z_{1}, \ldots, Z_{n} \stackrel{\text { i.i.d. }}{\sim} P.$ Assuming that there exist constants $a>1$  such that $A \defeq \mathbf{E}_{Z \sim P}\bigS{\exp(\norm{Z}^a)} < \infty$ (i.e., $P$ is a light-tail distribution). Then, for any $\rho > 0$,
	$$
	%\mathbf{P}\left[ W (P_{\lambda}, \widehat{P}_{n}) \geq  \rho \right] \leq c_{1} \exp \left(-c_{2} n \rho^{\max\{d/p, 2\}}\right) \\
	\mathbf{P}\left[ W_p (\widehat{P}_{n}, P) \geq  \rho \right]  \leq \begin{cases}c_{1} \exp \left(-c_{2} n \rho^{\max \{d/p,2\}}\right) & \text { if } \rho \leq 1 \\ c_{1} \exp \left(-c_{2} n \rho^{a}\right) & \text { if } \rho>1\end{cases}
	$$
	where $c_{1}, c_{2}$ are constants depending on $a, A$ and $d$.
\end{proposition}
As a consequence of this proposition,  for any $\delta > 0$, we have  
\begin{align}\label{E:deltarho}
	\mathbf{P}\left[ W_2 (\widehat{P}_{n}, P) \leq  \widehat{\rho}_{n}^{\delta} \right] \geq 1 - \delta \; \text{  where  } \; \widehat{\rho}_{n}^{\delta} \defeq \begin{cases}\left(\frac{\log \left(c_{1} / \delta \right)}{c_{2} n}\right)^{\min \{2/ d, 1 / 2\}} & \text { if } n \geq \frac{\log \left(c_{1} / \delta\right)}{c_{2}}, \\ \left(\frac{\log \left(c_{1} / \delta\right)}{c_{2} n}\right)^{1 / \alpha} & \text { if } n<\frac{\log \left(c_{1} / \delta\right)}{c_{2}}.\end{cases}
\end{align}

\begin{remark} 
	In \Cref{pro:WassRate}, \citet{fournier_rate_2013} show that the empirical distribution $\widehat{P}_n$ converges in Wasserstein distance to the true $P$ at a specific rate. This implies that  judiciously scaling the radius of Wasserstein balls according to \cref{E:deltarho} provides natural confidence regions for the data-generating distribution $P$.  Exploiting this fact for \OurAlg, we obtain the excess risk of ${f}_{\widehat{\theta}^{\varepsilon}}$ w.r.t. its true data distribution $P_{\lambda}$.
\end{remark}
} %comment out

\begin{corollary} \label{corr1}
	With all assumptions as in \Cref{thrm:excess_risk}, defining $\rho_n \defeq \sqrt{\sum\nolimits_{i=1}^{m} \lambda_i \widehat{\rho}_{n_i}^{\delta/m} }$, we have
	\begin{align*}
	    \mathscr{E}(P_{\lambda}, {f}_{\widehat{\theta}^{\varepsilon}})\leq  \sum_{i=1}^{m} \lambda_i \BiggS{\frac{48 \mathscr{C}(\Theta)}{\sqrt{n_i}} + 2 M_{\ell} \sqrt{\frac{2 \log (4  m / \delta)}{n_i}}   } 
	    + g (\rho_n, \gamma)+\varepsilon
	\end{align*}
    % 	\[ 
    % 	\mathscr{E}(P_{\lambda}, {f}_{\widehat{\theta}^{\varepsilon}})\leq  \sum_{i=1}^{m} \lambda_i \BiggS{\frac{48 \mathscr{C}(\Theta)}{\sqrt{n_i}} + 2 M_{\ell} \sqrt{\frac{2 \log (4  m / \delta)}{n_i}}   } + g (\rho_n, \gamma)+\varepsilon
    % 	\]
	with probability at least $1 - \delta$, where 
	$
	    \widehat{\rho}_{n}^{\delta} \defeq \begin{cases}\left(\frac{\log \left(c_{1} / \delta \right)}{c_{2} n}\right)^{\min \{2/ d, 1 / 2\}} & \text { if } n \geq \frac{\log \left(c_{1} / \delta\right)}{c_{2}}, \\ \left(\frac{\log \left(c_{1} / \delta\right)}{c_{2} n}\right)^{1 / \alpha} & \text { if } n<\frac{\log \left(c_{1} / \delta\right)}{c_{2}}.\end{cases}
	$
\end{corollary}

\section{Choosing $\lambda$: Applications} \label{sec:lambda}
We focus on two applications: multi-source domain adaptation and generalization to all client distributions. We provide insights on choosing the  weights $\lambda$ for these applications.

\textbf{Multi-source domain adaptation:} Consider the multi-source domain distribution $P_\lambda$ \citep{mansour_theory_2021}. \citet{lee_minimax_2018} show that solving the minimax risk with the Wasserstein ambiguity set can help transfer data/knowledge from the source domain $P_\lambda$ to a different but related target domain $Q$. They bound the distance $W_p(P_\lambda,Q)$ using the triangle inequality
\begin{align}
\!\!\! W_p(P_\lambda, Q) \leq  W_p(P_\lambda, \widehat{P}_\lambda) + W_p(\widehat{P}_\lambda, \widehat{Q}) + W_p(\widehat{Q}, Q),\label{E:tria}
%	\vspace{-4pt}
\end{align}
where $\widehat{P}_\lambda$ and $\widehat{Q}$ are the empirical versions of $P_\lambda$ and $Q$, respectively. While $W_p(P_\lambda, \widehat{P}_\lambda)$ and $W_p(\widehat{Q}, Q)$ can be probabilistically bounded  with a confidence parameter $\delta \in (0,1)$ according to \citet{fournier_rate_2013}, $W_p(\widehat{P}_\lambda, \widehat{Q})$ can be deterministically computed using linear or convex programming \citep{gabriel_2019_cot}.
In the FL context, in order to have a better bound for $W_2(P_{\lambda}, Q)$ similar to \cref{E:tria}, it is straightforward to choose $\lambda = \argmin_{\lambda' \in \Delta} W_2(\widehat{P}_{\lambda'}, \widehat{Q})$. To relax this problem into a form solvable using existing approaches, observe  that $W_2(\widehat{P}_{{\lambda}}, Q) \leq   \sum\nolimits_{i=1}^{m} \lambda_i W_2(\widehat{P}_{n_i}, {Q})$ due to the convexity of the Wasserstein distance. We then consider the following upper bound to $\min_{\lambda \in \Delta} W_2(\widehat{P}_{\lambda}, \widehat{Q})$:
\begin{align}
	\min_{\lambda \in \Delta} \SumNoLim{i=1}{m} \lambda_i W_2^2(\widehat{P}_{n_i}, \widehat{Q}) \eqdef {\rho^{\star}}^2,  \label{E:star}
\end{align}
which is a linear program, considering each $W_2^2(\widehat{P}_{n_i}, \widehat{Q})$ can be found by  efficiently solving  convex programs  especially with entropic regularization and the Sinkhorn algorithm \citep{cuturi_2013_sinkhorn}. %\blue{[Sinkhorn Distances:Lightspeed Computation of Optimal Transport]}. 
\begin{corollary} Denote  the solution to \cref{E:star} by $\lambda^{\star}$, and assume that domain $Q$ generates $n_Q$ i.i.d. data points. With probability at least  $1-\delta$, we have
	\begin{align*}
		W_2(P_{\lambda^{\star}}, Q) \leq W_2(P_{\lambda^{\star}}, \widehat{P}_{\lambda^{\star}}) + W_2(\widehat{P}_{\lambda^{\star}}, \widehat{Q}) + W_2( Q, \widehat{Q})
		\leq \sqrt{\sum\nolimits_{i=1}^{m} \lambda_i^{\star} \widehat{\rho}_{n_i}^{\delta/m}} + \rho^{\star} + \widehat{\rho}_{n_Q}^{\delta/2}.  %\label{E:bary2}
	\end{align*}
\end{corollary}
The proof of this corollary is similar to that of \Cref{coro_bary} in \Cref{sec:lambda:generalization}.

\textbf{Covering all client distributions in the Wasserstein ball:}  Suppose we want to cover all client distributions inside a Wassertein ball so that the generalization and robustness result by \OurAlg in \Cref{thrm:excess_risk} is applicable to all clients' distributions. We show in \Cref{sec:lambda:generalization} that this is a problem of finding $\lambda$ such that the Wasserstein distance between $P_{\lambda}$ and $P_j, \forall j$, is as small as possible.
\comment{
\textbf{Covering all client distributions in the Wasserstein ball:}  Suppose we want to cover all client distributions inside a Wassertein ball so that the generalization and robustness result by \OurAlg  in \Cref{thrm:excess_risk} is applicable to all clients' distributions. This is the problem of finding $\lambda$ such that 
the Wasserstein distance between $P_{\lambda}$ and $P_j, \forall j$, is as small as possible. Instead of directly finding the minimum  Wasserstein radius  that cover all client distributions, we will leverage the popular Wasserstein barycenter problem \citep{gabriel_2019_cot}. Specifically, consider the problem
\begin{align} \label{E:barycenter}
	\min_{\lambda \in \Delta} W_2(\widehat{P}_{{\lambda}}, P^{\varheart})  \quad \text{s.t. } \quad P^{\varheart} = \argmin_{ Q \in \gP} \SumNoLim{i=1}{m}\lambda_i W_2(\widehat{P}_i, {Q}), 
\end{align}
where $P^{\varheart}$ is the Wasserstein bary center w.r.t the solution $\lambda^*$ to this problem. Even though the solution is not straightforward, we propose to solve its tractable upper-bound:
\begin{align} \label{E:barycenter_relaxed}
	\min_{\lambda \in \Delta, Q \in \gP} \SumNoLim{i=1}{m} \lambda_i W_2(\widehat{P}_i, {Q}). 
\end{align}
This is a bi-convex problem, which is convex w.r.t to $\lambda$ (resp. $Q$) when fixing $Q$ (resp. $\lambda$). Thus, we can use alternative minimization \citep{gorski_biconvex_2007} to find a local solution to this problem. Denoting $\tilde{\lambda}$ as the solution to \Cref{E:barycenter} and $({\lambda}^*, {P}^*)$ as a local solution to \Cref{E:barycenter_relaxed}, we obtain
\begin{align}
	W_2(\widehat{P}_{\tilde{\lambda}}, P^{\varheart}) \leq   W_2(\widehat{P}_{{\lambda}^*}, {P}^*) \leq \sum\nolimits_{i=1}^{m} {\lambda}^*_i W_2(\widehat{P}_i, {P}^*) \eqdef \rho^*.
\end{align}
\begin{corollary} \label{coro_bary}
	For all client $j \in [m]$, with probability at least $1-\delta$, we have
	\begin{align*}
		W_2(P_{\lambda^*}, P_j) &\leq W_2(P_{\lambda^*}, \widehat{P}_{\lambda^*}) + W_2(\widehat{P}_{\lambda^*}, P^*) + W_2( P^*, \widehat{P}_j) + W_2(  \widehat{P}_j, P_j) \nonumber\\
		& \leq \sqrt{\sum\nolimits_{i=1}^{m} \lambda_i^{*} \widehat{\rho}_{n_i}^{\delta/m}} + \rho^* + \frac{\rho^*}{\lambda_j} + \widehat{\rho}_{n_j}^{\delta/2} \quad
	\end{align*}
\end{corollary}
\begin{proof}
	The first line is by triangle inequality. The second line is by  following facts: (i) $\mathbf{P}\BigS{W_2(P_{\lambda^*}, \widehat{P}_{\lambda^*}) \geq \sqrt{\sum\nolimits_{i=1}^{m} \lambda_i^{*} \widehat{\rho}_{n_i}^{\delta/m}} } \leq \delta/2$ according to \cref{E:barry}, (ii) $W_2( P^*, \widehat{P}_j) = \frac{\lambda_j W_2( P^*, \widehat{P}_j)}{\lambda_j}\leq \frac{\rho^*}{\lambda_j}$, and (iii)  $\mathbf{P}\bigS{W_2(  \widehat{P}_j, P_j) \geq \widehat{\rho}_{n_j}^{\delta/2} } \leq \delta/2$ according to \cref{E:deltarho}, and (iv) using union bound. 
\end{proof}
}

\section{Experiments}
\label{sec:experiment}
% We first show how to change the level of worst-case perturbations by varying the robust parameter $\gamma$. To show the generalizability and robustness of \OurAlg, we evaluate \OurAlg under non-i.i.d. and  distribution shift settings. We compare \OurAlg with baseline robust methods and non-robust FedAvg. Finally, we show \OurAlg's capability in domain adaptation.
\begin{wrapfigure}{r}{0.33\textwidth}
	\vspace{-20pt}
	\centering
	\begin{center}
		\includegraphics[width=0.31\textwidth]{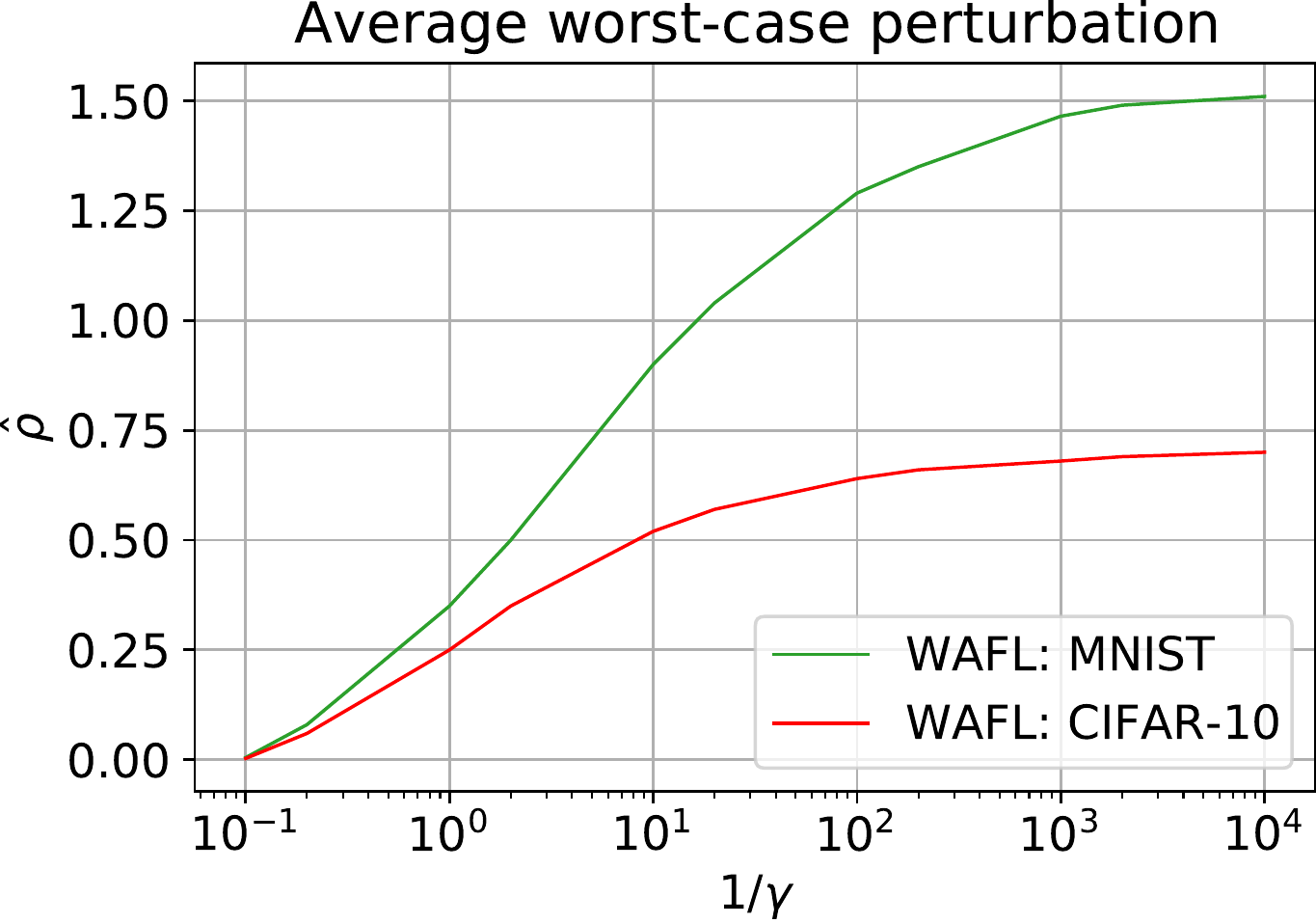}
	\end{center}
    %\vspace{-5pt}
	\caption{\OurAlg's hyperparameter $\gamma$ plays an opposite role as the average worst-case perturbation $\widehat \rho$: the smaller $\gamma$, the higher the level of perturbation $\widehat \rho$} % $\gamma$ generates the larger worst-case perturbation $\hat \rho$.} 
	\label{F:p_gamma} 
	\vspace{-15pt}
\end{wrapfigure}

We aim to show four key results through numerical experiments. First, we show the relationship between the hyperparameter $\gamma$ and the traditional worst-case perturbation $\rho$ used in distributionally robust learning \comment{(\Cref{sec:experiments:gamma_rho})}. Second, we investigate the effect of $\gamma$ on \OurAlg's performance in two data settings \comment{(\Cref{sec:experiments:gamma_general_robust})}. Third, we provide an extensive comparison of \OurAlg with other robust baselines and with FedAvg in scenarios with varying degrees of attack in an FL network \comment{(\Cref{sec:experiments:comparison_robust})}. Finally, \comment{in \Cref{sec:experiments:domain_adaptation},} we perform several experiments in multi-source domain adaptation to illustrate the findings in \Cref{sec:lambda}.

%\subsection{Experimental settings} 
\textbf{Experimental settings.} We design two non-i.i.d. FL settings. First, we use the MNIST dataset \cite{lecun_gradient-based_1998} to distribute to $100$ clients and employ a multinomial logistic regression model in a convex setting. We then use CIFAR-10 \cite{krizhevskyLearningMultipleLayers} to distribute to $20$ clients and employ a CNN model in \citet{mcmahan_communication-efficient_2017} in a non-convex setting. In the following experiments, we randomly sample $|S_t| = 10$ clients to participate in training at each communication round. When the stochastic gradient is calculated, we use a batch size of $|\gD_i| = 64$. For a fair comparison, we use the same number of global and local optimization rounds for each algorithm ($T = 200, K = 2$). More detail can be found in \Cref{sec:experiment:setting}. 
%\subsection{Effect of $\gamma$ on the worst-case risk perturbations} %\label{sec:experiments:gamma_rho}

\textbf{Effect of $\gamma$ on the worst-case risk perturbations.} Define the (squared) average worst-case perturbation as $\widehat{\rho}^2 = \mathbf{E}_{Z \sim \widehat P_{\lambda}} \bigS{d^2(\widehat{Z}, Z)}$, where $\widehat{Z}$ is the adversarial example of $Z$ as a solution to $\phi_\gamma(Z, \cdot)$. \Cref{F:p_gamma} depicts the relationship between $\widehat\rho$ and the predetermined $\gamma$ in the two data settings, and shows that smaller $\gamma$ corresponds to larger $\widehat \rho$. This allows us to indirectly control the amount of worst-case perturbation through the change of the hyperparameter $\gamma$ in the opposite direction. In other words, $\gamma$ is a hyperparameter that needs fine-tuning in order to obtain the best performance, and setting a sufficiently large $\gamma$ provides a \textit{moderate} level of robustness (smaller $\rho$ by duality) while ensuring $\phi_\gamma$ can be solved \textit{fast} using gradient methods (\Cref{sec:convergence}).

%\subsection{Effect of $\gamma$ on the generalizability and robustness of \OurAlg} %\label{sec:experiments:gamma_general_robust}

\textbf{Effect of $\gamma$ on the generalizability and robustness of \OurAlg.}
Consider $\widehat{P}_\lambda$ and $\widehat{Q}$ as the empirical distributions of training and test samples, respectively. By controlling the hyperparameter $\gamma$, we aim to train a global model robust to any test distribution $\widehat{Q}$. To do so, we design two scenarios. In the \emph{clean data} scenario, the global model is trained with different values of $\gamma$ and evaluated on clients' hold-out test data. In the \emph{distribution shift} scenario, the training process is the same, but the hold-out test data go through distribution shifts. To obtain these shifts, \comment{based on \citet{madry_towards_2019},} we employ the common PGD  attack \citep{madry_towards_2019} under the $l_\infty$-norm to generate an $\epsilon$-level perturbation of clients' test data. We fix the number of gradient steps to generate adversarial examples, and use $t_{avd} = 40, \epsilon = 0.3, \alpha = 0.01$ for MNIST, and $t_{avd} = 10, \epsilon = 8/255, \alpha = 2/255$ for CIFAR-10. We note that this setting is similar to that involving adversarial poisoning attacks, whose main purpose is to increase the Wasserstein distance between $\widehat P_\lambda$ and $\widehat Q$, thereby helping to verify the robustness of \OurAlg.

\Cref{F:Generalization} shows the performance of \OurAlg and FedAvg in the two scenarios. Under \emph{clean data}, the distance between $\widehat P_\gamma$ and $\widehat Q$ is relatively small, therefore requiring a lower amount of robustness (large $\gamma$). By carefully fine-tuning $\gamma$ in the ranges $[0.5, 1]$ for MNIST and $[10, 20]$ for CIFAR-10, \OurAlg enjoys the same or even better performance as FedAvg. The benefit of $\gamma$ emerges most clearly under \emph{distribution shift}. In this scenario, $\widehat P_\gamma$ and $\widehat Q$ grow further apart, requiring a larger ambiguity set $\gB(\widehat P_{\lambda}, \widehat \rho)$ (or, equivalently, a smaller $\gamma$) to ensure robustness. Meanwhile, too small $\gamma$ may violate the assumption that $\gamma > L_{zz}$ and can hurt \OurAlg's performance as $\widehat \rho$ becomes too large, as demonstrated in \Cref{Generalization}. In later experiments, we choose $\gamma = 0.5$ for MNIST and $\gamma = 10$ for CIFAR-10.

% \begin{figure*}[t]
% 	\centering
% % 	\includegraphics[scale=0.28]{figures/robust_attack.pdf}
% 	{\includegraphics[scale=0.275]{figures/MNIST_PGD.pdf}}\quad
% 	{\includegraphics[scale=0.275]{figures/MNIST_PGD_loss.pdf}}\quad
% 	{\includegraphics[scale=0.275]{figures/Cifar_PGD.pdf}}\quad
% 	{\includegraphics[scale=0.275]{figures/Cifar_PGD_loss.pdf}}
% 	\caption{Comparison with other robust methods on MNIST and CIFAR-10 at different proportion of attacked clients (clients are affected by distribution shifts).} %Even \OurAlg, FedPGM, and FedFGSM are trained at the same level of perturbation, \OurAlg outperforms others under the adversarial attacks (\blue{ distribution shifts}).}
% 	\label{F:robust}
% 	\vspace{-5pt}
% \end{figure*}

%\subsection{Comparison with other robust methods} \label{sec:experiments:comparison_robust}

\begin{figure*}[t]
    \vspace{-10pt}
	\centering
	\includegraphics[width=1\linewidth]{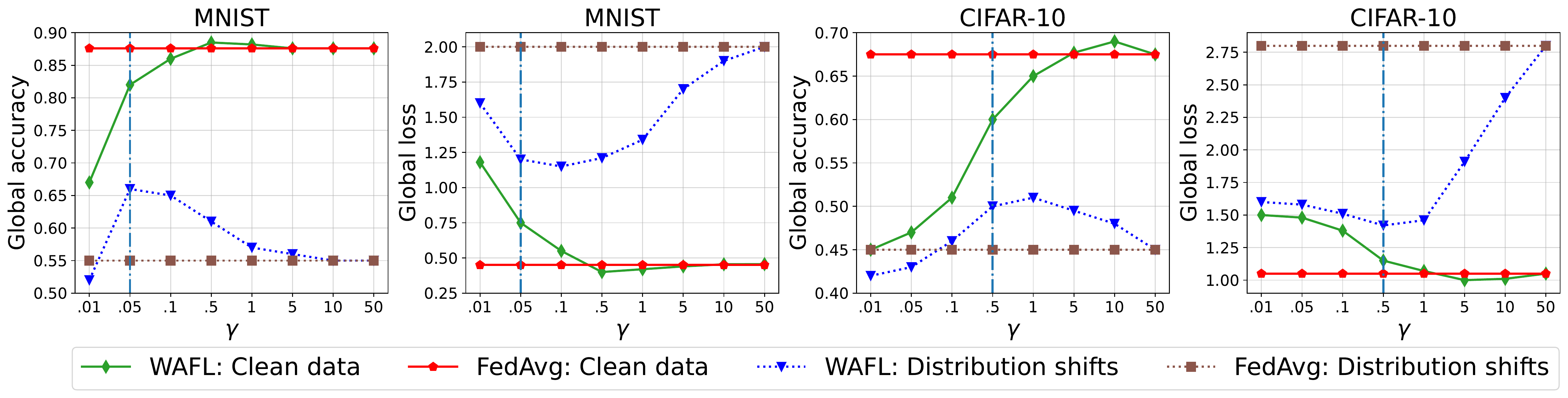}
	\caption{Global accuracy and loss of with different values of $\gamma$ on MNIST and CIFAR-10 under clean data and distribution shifts (40\% of clients are affected by PGD attack). The blue vertical line indicates the value of $\gamma$ giving the same level $\epsilon$ of PGD attack ($\gamma = 0.05$ for MNIST and   $\gamma = 0.5$ for CIFAR-10).}
	\label{F:Generalization}
	\vspace{-5pt}
\end{figure*}

\textbf{Comparison with other robust methods.}
We compare \OurAlg with FedAvg and four robust baselines in FL: FedPGM, FedFGSM, distributionally robust FedAvg \citep[DRFA]{deng_distributionally_2020} and agnostic FL \citep[AFL]{mohri_agnostic_2019}. FedPGM and FedFGSM are FedAvg with adversarial training using the PGD method \citep{madry_towards_2019} and the FGSM method \citep{goodfellow_explaining_2015} on local clients, respectively. In each local update of FedPGD and FedFGSM, all clients solve $\delta^* = \argmax\nolimits_{\norm {\delta}_{\infty} \leq \epsilon}\big{\{} \ell(  h_{\theta}(z+ \delta),y)  \big{\}}$ using projection onto an $l_\infty$-norm to find the worst-case perturbation $\delta$. While FedPGD uses $t_{avd}$ gradient steps to find $\delta^*$, FedFGSM uses only one gradient step. We use the same value of $t_{avd}$ when training using \OurAlg and FedPGM. On the other hand, DRFA and AFL both aim to achieve robustness by changing the clients' weights $\lambda_i$ based on local gradients and losses. AFL is considered a special case of DRFA by performing only one local gradient update.

To compare \OurAlg with these baselines, we consider a scenario in which a subset of clients suffers from distribution shifts (we call them \emph{attacked clients}). We generate the shifts using the same values of $\epsilon$ and $\alpha$ \comment{as in \Cref{sec:experiments:gamma_general_robust}}. We additionally train \OurAlg with the value of $\gamma$ generating the same level of perturbation $\epsilon$ in FedPGM and FedFGSM. The randomly-chosen attacked clients are between 20\% and 80\% of all clients. The global accuracy and loss for each dataset are presented in \cref{F:robust}. As expected, with all algorithms, the global accuracy decreases monotonically with the percentage of attacked clients. While FedAvg, by definition a non-robust method, unsurprisingly suffers the largest performance drop, \OurAlg ourperforms all baselines, retaining over 50\% accuracy on MNIST and nearly 45\% on CIFAR-10 even when 80\% of clients experience distribution shifts. We \comment{also} observe that the performance of FedPGD and FedFGSM is much better than DRFA and AFL, and is the closest to \OurAlg. This suggests that adjusting the clients' weights $\lambda_i$ may not notably help with achieving robustness.

\begin{figure*}[t]
	\centering
	\includegraphics[width=1\linewidth]{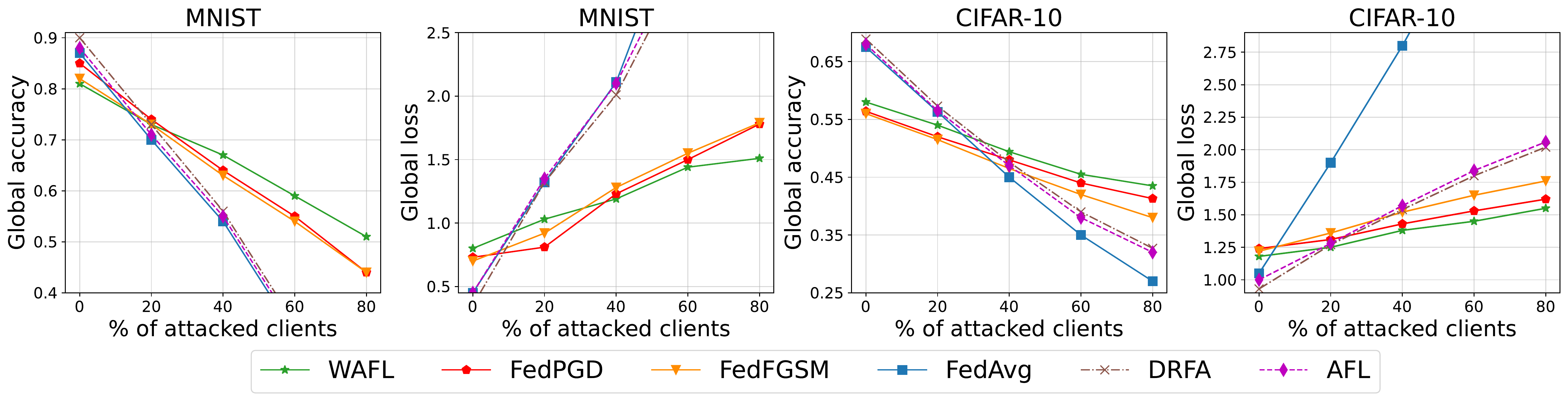}
	\caption{Comparison with other robust methods on MNIST and CIFAR-10 with different proportions of clients suffering from distribution shifts (attacked clients).}
	\label{F:robust}
	\vspace{-10pt}
\end{figure*}

Furthermore, we provide a comparison between the performance of WAFL with different $p$ values and other baselines in \Cref{Apx:WAFL_p} to show that the duality result in \Cref{E:duality_Wass} suffices with any $\ell_p$ norm.

%\subsection{Domain adaptation} \label{sec:experiments:domain_adaptation}
\textbf{Domain adaptation.}
\Cref{sec:lambda} describes \OurAlg's capability in multi-source domain adaptation by solving a linear program in $\lambda$. \comment{Here}We empirically demonstrate that capability using three digit recognition datasets including MNIST (\mt) \cite{lecun_gradient-based_1998}\comment{, containing 70,000 instances of $28 \times 28$ grayscale images)}, USPS (\up) \cite{hullDatabaseHandwrittenText1994}\comment{, containing 9,298 instances of $16 \times 16$ grayscale images)}  and SVHN (\sv) \cite{37648} \comment{, containing 89,289 instances of $32 \times 32$ RGB images)}. We convert all images to have the size of $28 \times 28 \times 3$. More information about these datasets \comment{, including how they are split to the clients,} can be found \comment{in \Cref{T:Data_Samples_2}} in \Cref{sec:experiment:setting}. We then train a global multinomial logistic regression model on two source domain datasets, and evaluate it using the remaining dataset as the target domain. To solve the linear program in \Cref{E:star}, we estimate the Wasserstein distance by leveraging the computational methods introduced in \cite{cuturi_2013_sinkhorn,otdd}, and solve the linear program using SciPy\comment{'s optimization module}\footnote{\url{https://docs.scipy.org/doc/scipy/reference/optimize.html}}. For comparison, we use FedAvg in two scenarios: $\lambda_i = n_i/n$ and $\lambda_i = 1 / m$. We also employ AFL and DRFA, both of which can\comment{are able to} vary the $\lambda_i$ to achieve robustness. All algorithms are fine-tuned to obtain their best performance on the target datasets.

\begin{table}[!htpb]
    \vspace{-10pt}
	\small
	\setlength\tabcolsep{4pt} 
	\caption{Accuracies on target domains. \comment{The models are each trained using two datasets (source domains), and evaluated on the other dataset (target domain).}}
	%The first two rows represent FedAvg using pre-determined $\lambda_i = n_i / n$ and $\lambda_i = 1 / m$, respectively. The highest accuracy in each column is in \textbf{bold}.}
	\centering
	\begin{tabular}{ccccc}
		\toprule
		\textbf{$\lambda$} & \textbf{\small \mt, \sv $\rightarrow$ \up} & \textbf{\small \mt, \up $\rightarrow$ \sv} & \textbf{\small \up, \sv $\rightarrow$ \mt} & Avg \\ 
		\midrule
		$n_i/n$      &        59.0    &             14.1       &              16.1         &  29.7 \\
		$1/m$        &        58.7    &             14.9       &              52.1         &  41.6 \\ 
		AFL          &        60.1    &             15.0       &              52.4         &  42.5 \\
		DRFA         &        61.6    &             15.1       &              53.0         &  43.2 \\
		\OurAlg      &\textbf{65.6}   &       \textbf{16.6}    &          \textbf{58.1}    &  \textbf{46.7} \\
		\bottomrule
	\end{tabular}
	\label{F:domain_2}
	%\vspace{-5pt}
\end{table}

The accuracies on the target domains are presented in \Cref{F:domain_2}. In all three scenarios, \OurAlg outperforms all other methods, especially in the settings \mt, \sv $\rightarrow$ \up and \sv, \up $\rightarrow$ \mt, where \OurAlg's accuracy exceeds the second-best accuracy (achieved by DRFA) by five percentage points. We note that the \sv dataset is the most different from the other two, measured by the Wasserstein distance, which is why generalization to \sv's domain is the most difficult. 
%\sout{\verify{Due to obtaining a large number of samples, the impact of $\lambda = n_i/n$ of \sv overwhelms the remaining datasets. This causes a detrimental effect on the performance of global model on the target domain.}}
%\verify{[NOTE: I'm not sure what this sentence means; we need to rephrase it.]}

\comment{

\textbf{Experimental settings.} %\label{sec:experiment:settings} 
We use two datasets in two non-i.i.d. settings. We distribute the first dataset -- MNIST \citep{lecun_gradient-based_1998} -- to 100 clients and use a multinomial logistic regression model to model a convex setting. For the second dataset -- CIFAR-10 \citep{krizhevskyLearningMultipleLayers} -- we use 20 clients and a CNN model employed in \citet{mcmahan_communication-efficient_2017} to model a non-convex setting. We set $\lambda_i = {n_i}/{n}$ as the client weights. 
% By varying $\gamma$, we control the level of worst-case perturbation $\widehat{\rho}$ to train a global model which is robust on any unknown distribution $\widehat Q$ inside the Wasserstein ball $\gB(\widehat {\lambda}, \widehat \rho)$. 
In optimization, we randomly sample $S_t = 10$ clients to participate in training at each communication round.
% \textcolor{red}{In our experiments, $\gamma$ is considered as a hyper-parameter for fine-tuning depending on the heterogeneous, non-i.i.d data and the level of distribution shifts. We conduct all experiments with the global rounds of communication $T=200$ and the number of local updates $K=2$}. 
More detail can be found in \Cref{sec:experiment:setting}.

% We use the MNIST \citep{lecun_gradient-based_1998} and CIFAR-10 \citep{krizhevskyLearningMultipleLayers} datasets and evaluate \OurAlg on two different FL network sizes: 100 clients for MNIST and 20 clients for CIFAR. We distribute all datasets so as to capture the non-i.i.d. and heterogeneous characteristics of FL. To verify the performance of \OurAlg, we employ a multinomial logistic regression model as the strongly convex model for MNIST and a CNN network used in \cite{mcmahan_communication-efficient_2017} as non-convex model for CIFAR. The details about datasets and models are provided in \Cref{sec:experiment:setting}. %as follows:

%In the experiment, we do not try to find optimal $\lambda^*$ for domain adaptation and Wasserstein ball covering client's distributions applications.

% \label{sec:experiment:gamma_rho}
\textbf{Effect of $\gamma$ on the worst-case risk perturbations.} We first examine the relationship between the robust parameter $\gamma$ and the average worst-case perturbations $\widehat{\rho}$, defined as $\widehat{\rho}^2 = \mathbf{E}_{Z \sim \widehat P_{\lambda}} \bigS{d^2(\widehat{Z}, Z)}$, where $\widehat{Z}$ is the adversarial perturbation of $Z$. As shown in \Cref{F:p_gamma}, smaller values of $\gamma$ correspond to larger worst-case perturbations $\widehat{\rho}$ on both the MNIST and CIFAR-10 datasets. For the rest of the experiment, %we use $\gamma$ as a robust parameter to control the level of worst-case perturbation $\widehat{\rho}$.
rather than control $\widehat \rho$ directly, we set $\gamma$ on the opposite direction to control the level of robustness.
%\begin{figure*}[!t]
%	\centering
%	{\includegraphics[scale=0.3]{figures/Average_worst_case.pdf}}
%	\caption{ Changing level of the worst-case perturbation $\widehat \rho$ by varying robust parameter $\gamma$. Smaller $\gamma$ generates the larger worst-case perturbation $\widehat \rho$.} %2 points  indicate the  level of perturbation of PGD attack on MNIST and CIFAR. At same level of perturbations the value of $\gamma$ is 0.4 and 2 for MNIST and CIFAR respectively.}
%	\label{F:p_gamma}
%\end{figure*}

%We first show the coordination between the average distance of the worst-case perturbation and the value of $\gamma$ on both MNIST and CIFAR in Fig.~\ref{F:p_gamma}. 
%For a given value of $\gamma$, we consider the distance between worst-case risk examples and the original ones as $\widehat{\rho}^2  = \mathbf{E}_{Z \sim \widehat P_{\lambda}} \bigS{d^2(\widehat{Z}, Z)}$.
%\begin{equation}
%	\widehat\rho  = \mathbf{E}_{Z \sim \widehat P_{\lambda}} \bigS{d(\widehat{Z}, Z)}, \nonumber
%\end{equation}
%where $d(\widehat{z}, z) := \norm{\widehat{x} - x}_2^2$. 
%Fig.~\ref{F:p_gamma} shows that smaller value of $\gamma$ generates the larger worst-case perturbation distances $\widehat\rho$. For the rest of the experiment, we use $\gamma$ as a robust parameter to control the worst-case perturbation $\widehat \rho$.% of the client's data.

% \label{sec:experiment:generalization}
\textbf{Generalization and robustness of \OurAlg.}
We consider $\widehat P_{\lambda}$ and $\widehat Q$ as the empirical distribution of the training samples and test samples of all clients, respectively. By varying $\gamma$, we aim to train a global model robust to any  empirical test distribution $\widehat Q$. To show the generalization and robustness of \OurAlg, we train and evaluate it in two different scenarios. First, in the \emph{clean data} scenario, the global robust model is trained with different values of $\gamma$ and then evaluated on given clients' test data (similar to traditional FL). Second, in the
\emph{distribution shift} setting, the training process is similar; however, the global robust model is evaluated when there are distribution shifts at the clients' test data. To obtain these shifts, we use the common PGD attack method in \cite{madry_towards_2019} under $l_{\infty}$-norm to generate an $\epsilon$-level perturbation on clients' test data. We choose the $l_{\infty}$-norm as it shows benefits in adversaries and gives large perturbations. Following \cite{madry_towards_2019}, we fix the number of gradient steps to generate adversarial examples with $t_{avd} = 40$, $\epsilon = 0.3$, $\alpha = 0.01$ for MNIST and $t_{avd} = 10$, $\epsilon = 8/255, \alpha = 2/255$ for CIFAR-10 with a batch size of 64.  %We consider different proportion of clients are affected by distribution shifts on their test data.
We note that this setting is similar to the adversarial poisoning attacks and the main purpose of this setting is to increase the Wasserstein distance between $\widehat P_{\lambda}$ and  $\widehat Q$, thereby verifying the robustness of \OurAlg. 

%We consider $\widehat P_{\lambda}$  and $Q$ as the average distribution on training sample and test sample of all clients respectively.
The performance of \OurAlg in both scenarios on MNIST and CIFAR-10 is shown in \Cref{F:Generalization}. When the clients' data is clean, the Wasserstein distance between $\widehat P_{\lambda}$ and  $\widehat Q$ is relatively small, and training \OurAlg with small $\gamma$ (large $\widehat \rho$) gives the worse performance on the test set. With a sufficiently large $\gamma$, \OurAlg is less robust and has the same generalization with FedAvg. By carefully fine-tuning $\gamma$ in the range $[0.5,1]$ for MNIST and $[10,20]$ for CIFAR-10, \OurAlg shows an improvement over FedAvg. $\gamma$ in this scenario plays the same role as a regularization parameter to handle non-i.i.d. data.

% , \OurAlg is evaluated as a traditional federated learning algorithm and compared with FedAvg. As the heterougenious nature of federated learning setting, 
%When client's test data is clean, with sufficient large value of $\gamma$, $\OurAlg$ behaviour similar to FedAvg.
 
By adding distribution shifts, we increase the Wasserstein distance between $\widehat P_{\lambda}$ and $\widehat{Q}$. \textcolor{red}{Thanks to} varying $\gamma$, we train a global robust model with different levels of worst-case perturbation to handle the distribution shifts. Small values of $\gamma$ generate larger ambiguity sets $\gB(\widehat P_{\lambda}, \widehat \rho)$ and increase the robustness of \OurAlg, thus increase the chance $\widehat{Q}$ lies inside $\gB(\widehat P_{\lambda}, \widehat \rho)$. In \Cref{F:Generalization}, \OurAlg shows better performance with smaller $\gamma$ values. However, as in mentioned in \Cref{Generalization}, when $\widehat \rho$ is much larger or smaller than the level of distribution shifts ($\gamma \leq 0.01$  or $\gamma \geq 1 $ for MNIST and $\gamma \leq 0.1$  or $\gamma > 10$ for CIFAR-10), \OurAlg performs inefficiently: too small $\gamma$ hurts the empirical performance, while too large $\gamma$ makes \OurAlg lack robustness. %The robustness of \OurAlg mainly depends on the value of paramter $\gamma$, We then need to adjust $\gamma$ correspondingly depends on the level of distribution shifts.
%The level robustness of \OurAlg mainly depends on $\gamma$. 
For every scenario, $\gamma$ needs to be tuned correspondingly. By carefully choosing $\gamma$, \OurAlg not only handles distribution shifts or common data poisoning attacks but also provides better performance than FedAvg in non-i.i.d. data and heterogeneous settings. Specifically, we set $\gamma = 0.5$ for MNIST  and $\gamma = 10$  for CIFAR-10.
% cover the distribution shift at client's test data, then \OurAlg has better accurancy with smaller $\gamma$. 
%with the radious $\rho$
% . Without distribution shifts at client's test data, \OurAlg with large value of $\gamma$ shows a similar performance with FedAvg under non-iid data. When the adversary has far less budget than that used for training, degradation in performance is reduced significantly.  When $\gamma$ is large enough, \OurAlg actually is FedAvg when $\rho$ closed to 0.  

\textbf{Comparison with other robust methods.} %\label{sec:experiment:comparison_robustness}
%\begin{figure*}[!t]
%	\centering
%	{\includegraphics[scale=0.275]{figures/Mnist_generalization.pdf}}\quad
%	{\includegraphics[scale=0.275]{figures/Mnist_generalization_error.pdf}}\quad
%	{\includegraphics[scale=0.275]{figures/Cifar_generalization.pdf}}\quad
%	{\includegraphics[scale=0.275]{figures/Cifar_generalization_error.pdf}}
%	\caption{ Global accuracy and global loss of \OurAlg with different values of $\gamma$  and FedAvg on MNIST and CIFAR-10 under clean data and distribution shifts scenarios. %Larger  $\gamma$  corresponds to small magnitudes of $\widehat $, which ensures the stability of the global model. 
%		The blue vertical line indicates the value of $\gamma$ which gives the same level of perturbation $\epsilon$ of PGD on MNIST and CIFAR.}
%	\label{F:Generalization}
%\end{figure*}

% \begin{figure*}[!t]
% 	\centering
% 	{\includegraphics[scale=0.275]{figures/MNIST_PGD.pdf}}\quad
% 	{\includegraphics[scale=0.275]{figures/MNIST_PGD_loss.pdf}}\quad
% 	{\includegraphics[scale=0.275]{figures/Cifar_PGD.pdf}}\quad
% 	{\includegraphics[scale=0.275]{figures/Cifar_PGD_loss.pdf}}
% 	\caption{Comparison with other robust methods on MNIST and CIFAR at different percentages of attacked clients. Even \OurAlg, FedPGD, and FedFGSM are trained at the same level of robustness $\widehat \rho$, \OurAlg outperforms others under the distribution shifts.}
% 	\label{F:robust}
% \end{figure*}
\textcolor{red}{We compare \OurAlg with five baselines: four robust methods in FL called FedPGM, FedFGSM, Distributionally Robust Federated Averaging DRFA \citep{deng2021distributionally} and Agnostic Federated Learning AFL \cite{mohri2019agnostic}, and one non-robust method FedAvg}. FedPGM and FedFGSM are FedAvg with adversarial training using the projected gradient method PGD \citep{madry_towards_2019} and the fast-gradient method FGSM \citep{goodfellow_explaining_2015}, respectively. In FedPGM and FedFGSM, in each local update, all clients solve  $\delta^* = \argmax\nolimits_{\norm {\delta}_{\infty} \leq \epsilon}\big{\{} \ell(  h_{\theta}(z+ \delta),y)  \big{\}}$ using projection onto an $l_\infty$-norm to find the worst-case perturbation $\delta$.
% We compare \OurAlg with three baselines: two common robust methods in FL called FedPGM and FedFGSM, and one non-robust method FedAvg. FedPGM and FedFGSM are FedAvg with adversarial training using the projected gradient method PGD \citep{madry_towards_2019} and the fast-gradient method FGSM \citep{goodfellow_explaining_2015}, respectively. In FedPGM and FedFGSM, in each local update, all clients solve  $\delta^* = \argmax\nolimits_{\norm {\delta}_{\infty} \leq \epsilon}\big{\{} \ell(  h_{\theta}(z+ \delta),y)  \big{\}}$ using projection onto an $l_\infty$-norm to find the worst-case perturbation $\delta$.
%\begin{align} \label{E:pgd}
% %z_i^* = \underset{\norm {\zeta} \leq \epsilon}{\mbox{argmax}}\Big{\{} \ell(\zeta,  h_{\theta})  \Big{\}} \\
% \delta^* = \underset{\norm {\delta}_{\infty} \leq \epsilon}{\mbox{argmax}}\Big{\{} \ell(  h_{\theta}(z+ \delta),y)  \Big{\}}
%\end{align}
 %FedAvg is the non-robust method using ERM.
While FedPGD uses $t_{avd}$ gradient steps to find $\delta^*$, FedFGSM uses only one gradient step. We use the same values of $\epsilon$ and $\alpha$ from the distribution shifts setting for projection. For a fair comparison, both \OurAlg and FedPGM have the same value $t_{avd}$ and all algorithms have the same number of local updates $K$. %then FedPGM and \OurAlg have the same computation complexity and have $t_{avd}$ times more compared to FedFGSM. 
We also train \OurAlg with the value of $\gamma$ generating the same level perturbation of $\epsilon$ in FedPGM and FedFGSM. 
% Since we are training the model for multiple epochs, there is no benefit from restarting PGM multiple times per batch, a new start will be chosen the next time each example is encountered.
\textcolor{red}{The DRFA and AFL ultilize gradients and losses from clients to gradually find $\lambda$, whereas other approaches use fixed values of $\lambda$ resulted from proportions of data in each client. The AFL algorithm is considered as a special case of DRFA by performing one local gradient update at client. While accuracies of DRFA and AFL are averagely $2 \%$ and $3 \%$ higher than FedAvg respectively, both algorithms have to solve a projection problem on $l_1$-norm ball which puts more computational burden on server side. Although DRFA and AFL need less number of communications to converge, they add more payload to exchanged data between server and clients. From Fig.4, it is plausible that results of FedAvg with choosing lambdas values as proportion of obtained data are not much different from DRFA and AFL in terms of accuracy. In \citep{deng2021distributionally}, experimental results illustrate that FedAvg can achieve the same accuracies of DRFA and AFL with large enough communications.}
We study different proportions of clients having distribution shifts in \Cref{F:robust}. The robust accuracy of all methods decreases when the percentage of clients having distribution shifts increases. Especially, the FedAvg's performance drops dramatically when the percentage of attacked clients reaches $80\%$. For all scenarios, \OurAlg outperforms all the baselines.  Specifically, in the case of $80\%$ attacked clients, \textcolor{red}{the improvements in accuracy of \OurAlg over FedPGM, FedFGSM, DRFA, AFL and FedAvg are $7\%$, $8\%$, $32\%$, $32.6\%$ and $33\%$ for MNIST  and $2.5\%$, $4.5\%$, $12.3\%$, $13\%$ and $18\%$ for CIFAR-10, respectively.}

\textbf{Domain adaptation applications.}
\textcolor{red}{In order to evaluate the potential of estimating $\lambda$ in \OurAlg, we conduct multi-source domain adataption experiments using benchmark digit recognition datasets (MNIST - \emph{mt}, USPS - \emph{up}, SVHN - \emph{sv}).  Particularly, we choose two of the aforementioned datasets as source domains for training, and leave the rest for target domain testing. Leveraging computational Wasserstein distance methods between data distributions \citep{cuturi_2013_sinkhorn,otdd}, $\lambda$ in \OurAlg is obtained by solving the linear programming problem proposed in \Cref{E:star} in respect of domain adaptation applications. We then observe \OurAlg's performance, and compare with other $\lambda$ estimation methods including FedAvg, AFL and DRFA. The experiment results in Table \Cref{F:domain_2} depict that \OurAlg outperforms the remaining approaches in all scenarios, especially the accuracy on target domains in \textit{mt,sv$\rightarrow$up} and \textit{sv,up$\rightarrow$mt} are substantially improved in comparison to the previous best results with 4\% and 5\% respectively. With $\lambda = n_i/n$ in \textit{sv,up$\rightarrow$mt}, trained model on source distributions has poor performance on target domain. Since $sv$ dataset acquires the largest Wasserstein distance from target, $sv$ is the most different dataset compared to target. Due to obtaining a large number of samples, the impact of $\lambda = n_i/n$ of $sv$ overwhelm the remaining datasets. This has detrimental effect on accuracy of global model on target dataset.}
% the accuracy on target dataset is much lower than the others because \textit{sv} dataset which is the greatest Wasserstein distance from target has the largest impact on the global model.}
% \begin{table}[ht!]
% \caption{Datasets for Domain Adaptation Experiments}
% \centering
% \begin{tabular}{|c|c|c|}
% \hline
% \textbf{} & \textbf{Training source} & \textbf{Testing target} \\ \hline
% msda1     & MNIST, SVHN              & USPS                    \\ \hline
% msda2     & MNIST, USPS              & SVHN                    \\ \hline
% msda3     & SVHN, USPS               & MNIST                   \\ \hline
% \end{tabular}
% \end{table}

% \begin{table}[ht!]
% \caption{Accuracy on target dataset}
% \centering
% \begin{tabular}{|c|l|l|l|}
% \hline
% \textbf{Lambdas}      & \textbf{msda1} & \textbf{msda2} & \textbf{msda3} \\ \hline
% proportion of dataset & 59.0\%         & 14.1\%         & 16.1\%         \\ \hline
% 1/(number of users)   & 60.7\%         & 14.9\%         & 52.1\%         \\ \hline
% AFL                   & 60.1\%         & 15.0\%         & 52.4\%         \\ \hline
% DRFA                  & 61.6\%         & 15.1\%         & 53.0\%         \\ \hline
% Domain Adaptation     & \textbf{65.6\%}         & \textbf{16.6\%}         & \textbf{58.1\%}\\ \hline
% \end{tabular}
% \end{table}

\begin{table}[ht!]
	\small
	\setlength\tabcolsep{4pt} 
	\caption{Effect of $\lambda$ on accuracy of target dataset}
	\centering
	\begin{tabular}{ccccc}
		\toprule
		\textbf{$\lambda$} & \textbf{ \small \textit{mt,sv$\rightarrow$up} }&\textbf{ \small \textit{mt,up$\rightarrow$sv} }& \textbf{ \small\textit{up,sv$\rightarrow$mt}} & Average\\ \midrule
		$n_i/n$      &        59.0\%    &             14.1\%       &              16.1\%         &  29.7\% \\
		$1/m$        &        58.7\%    &             14.9\%       &              52.1\%         &  41.6\% \\ 
		AFL          &        60.1\%    &             15.0\%       &              52.4\%         &  42.5\%\\
		DRFA         &        61.6\%    &             15.1\%       &              53.0\%         &  43.2\%\\
		\OurAlg      &\textbf{65.6}\%   &       \textbf{16.6}\%    &          \textbf{58.1}\%    &  46.7\%\\
		\bottomrule
	\end{tabular}
	\label{F:domain_2}
\end{table}

\comment{\textbf{Domain Adaptation.} \label{sec:experiment:domain_adaptation}
%\begin{wraptable}{r}{9cm}
\begin{table}[ht!]
	\small
	\setlength\tabcolsep{4pt} 
	\caption{Performance of \OurAlg on domain adaptation application.}
	\centering
	\begin{tabular}{cclllll}
		\toprule
		Algorithm & \blue{ \small \textit{mt$\rightarrow$mm} }&\blue{\small \textit{mt$\rightarrow$up }}& \blue{\small\textit{mm$\rightarrow$up}} & \blue{\small \textit{mm$\rightarrow$mt }}& Average \\ \midrule
		\OurAlg    &      \textbf{40.23}          &       \textbf{66.94}        &          \textbf{62.04}&            \textbf{79.27}        &   \textbf{62.12}  \\
		FedPGM    &       37.36       &      62.74   &    58.23        &      71.84    &    57.54  \\
		FedAvg    &        37.71    &             66.08        &              58.60    &           75.80                   &   59.55  \\ \bottomrule
	\end{tabular}
	\label{F:domain}
\end{table}
%\end{wraptable}
In order to show the ability of \OurAlg in transferring the knowledge from multi-source domain $\widehat P_{\lambda}$ to a different but related target domain $\widehat Q$. We consider an unsupervised setting in which the target client only has data without labels. So the target client is not able to train the model, thus it needs the global model contributed from all source clients to make predictions without training. In recent works on federated domain adaptation \citep{pengFederatedAdversarialDomain2019, liuFedDGFederatedDomain2021}, source domains need to access to the private data of target domain during training to find the common representation between source and target domains, which does not guarantee the data privacy of FL. By contrast, \OurAlg aims to build a robust global model which can be directly applied to the related target client without accessing its private data. To simulate this setting, we use three-digit classification datasets including MNIST (\emph{mt}), MNIST-M (\emph{mm}), and USPS (\emph{up}) which are different but related and widely used in domain adaptation. We distribute one dataset to 100 source clients, and leverage another dataset for the target client. For example, if we use MNIST for source clients and MNIST-M for the target client, then we have the \blue{\emph{mt}$\rightarrow$\emph{mm}} case. We do similar to other cases in \Cref{F:domain}. We compare \OurAlg with FedPGM and FedAvg as they have the same setting regarding preserving data privacy. The results in \Cref{F:domain} show that \OurAlg outperforms the baselines, which shows its benefits in transferring knowledge from multi-source clients to a related target client.
}

% We distribute one dataset, for example, MNIST to 100 source clients, and leave MNIST-M to the target client (\blue{\emph{mt}$\rightarrow$\emph{mm}} case)
%\begin{table}[ht!]
%	\setlength\tabcolsep{2pt} 
%	\begin{tabular}{cclllll}
%		Algorithm & \blue{ \small \textit{mt,sv,sy,up$\rightarrow$mm} }&\blue{\small \textit{mm,sv,sy,up$\rightarrow$mt }}& \blue{\small\textit{mt,mm,sy,up$\rightarrow$sv}} & \blue{\small \textit{mt,mm,sv,up$\rightarrow$sy }}&\blue{\small \textit{mt,mm,sv,sy$\rightarrow$up}}& Avg \\ \hline
%		FedRob    &                              &                              &                              &                              &                              &     \\
%		FedAvg    &                              &                              &                              &                              &                              &     \\ \hline
%	\end{tabular}
%	\label{F:domain}
%\end{table}
%\begin{table}[ht!]
%	\begin{tabular}{cclllll}
%		Algorithm & \textbf{B} & \textbf{D}& \textbf{E}& \textbf{K} & Avg \\ \hline
%		FedRob    &                              &                              &                              &                           &     \\
%		FedAvg    &                              &                              &                              &                           &     \\ \hline
%	\end{tabular}
%\end{table}
}
\section{Conclusion}

In this paper, we apply the Wasserstein distributionally robust training method to federated learning to handle statistical heterogeneity. We first remodel the duality of the worst-case risk to an empirical surrogate risk minimization problem, and then solve it using a local SGD-based algorithm with convergence guarantees. We show that \OurAlg is more general in terms of robustness compared to related approaches, and obtains an explicit robust generalization bound with respect to all unknown distributions in the Wasserstein ambiguity set. Through numerical experiments, we demonstrate that \OurAlg generalizes better than the standard FedAvg baseline in non-i.i.d. settings, and outperforms other robust FL methods in scenarios with distribution shifts and in applications of multi-source domain adaptation.
% In future work, we will explore the different aspects of distributionally robust optimization, and employ for enhancing robustness in Federated Learning.

\newpage
\bibliography{neurips_2022.bib}
\bibliographystyle{unsrtnat}

%%%%%%%%%%%%%%%%%%%%%%%%%%%%%%%%%%%%%%%%%%%%%%%%%%%%%%%%%%%%
\newpage
\appendix

\section*{Appendix}

%Optionally include extra information (complete proofs, additional experiments and plots) in the appendix.
%This section will often be part of the supplemental material.

\section{Adversarial Robust FL's Ambiguity Set vs  Wassertein Ball} \label{Apx:adversarial}

{We show that using the Wasserstein ambiguity set contains the perturbation points induced by the solution to the \textit{Adversarial Robust FL} approach.  As we present in \Cref{subsec:wasserstein_risk_fl},  existing techniques for adversarial training robust models ~\citep{goodfellow_explaining_2015, papernot_limitations_2015, kurakin_adversarial_2017, carlini_towards_2017, madry_towards_2019, tramer_ensemble_2020} define an adversarial perturbation $u$ at a data point $Z$, and minimize the following worst-case loss over all possible perturbations
	\begin{equation}
		\underset{u \in \gU}{\mbox{max }} \mathbf{E}_{Z \sim \widehat{P}_{\lambda}} \Big{[}\ell(Z + u,  h_{\theta})\Big{]},  \label{E:U}
	\end{equation}
	where the ambiguity set  $\, \gU \defeq \set{u \in \mathbb{R}^{d+1} : \norm{u} \leq \eps}$.  To compare this approach with Wasserstein-robust FL,  we  relate the above problem to its counterpart defined  in the probability space of input as follows
	\begin{equation}
		\underset{\tilde{Q} \in \gQ(\eps)}{\mbox{max }} \mathbf{E}_{\tilde{Z} \sim \tilde{Q}} \Big{[}\ell(\tilde{Z},  h_{\theta})\Big{]},\label{E:Q}
	\end{equation}
	$\text{where } \gQ(\eps) \defeq \set{\tilde{Q} : \mathbb{P} \bigS{\norm{\tilde{Z} - Z} \leq \eps} = 1, Z \sim \widehat{P}_{\lambda}, \tilde{Z} \sim \tilde{Q}}.$
	Considering $u^*$ as a solution to problem \cref{E:U},  we see that the distribution $Q'$ of perturbation points (i.e., $\tilde{Z} \defeq (Z+u^*) \sim {Q}'$) in  problem \cref{E:U}   belongs to the feasible set  $\gQ(\eps)$ in problem \cref{E:Q} (If not, then $\mathbb{P} \bigS{\norm{u^*} \leq \eps} <1$, a contradiction). Next, consider an arbitrary  distribution $\tilde{Q} \in \gQ(\eps)$ in problem \cref{E:Q},  with any $\tilde{Z} \sim \tilde{Q}$ and $Z \sim \widehat{P}_{\lambda}$, we have
	\begin{align*}
		\norm{\tilde{Z} - Z} \stackrel{\text{w.p.1}}{\leq }\eps  \implies \mathbf{E}_{Z \sim \widehat{P}_{\lambda}, \tilde{Z} \sim  \tilde{Q}} \bigS{\norm{\tilde{Z} - Z}} \leq \eps \implies \inf_{\pi \in \Pi(\widehat{P}_{\lambda}, \tilde{Q})} \mathbf{E}_{(Z, Z') \sim \pi} \bigS{\norm{\tilde{Z} - Z}} \leq \eps,
	\end{align*}
	which implies that $ W_1(\widehat{P}_{\lambda}, \tilde{Q}) \leq \eps, \forall \tilde{Q} \in \gQ(\eps)$, and thus  $\gQ(\eps) \subset \gB_1(\widehat{P}_{\lambda}, \eps)$. We have shown that the Wasserstein ambiguity set contains the perturbation points induced by the solution to the adversarial robust training problem \cref{E:U}. }

\comment{
\subsection{Robust Generalization Bounds} 
We now present how we adapt the result from \citet{fournier_rate_2013} to obtain \cref{corr1} in Section 4.
	\begin{proposition}[Measure concentration~{\citep[Theorem~2]{fournier_rate_2013}}]\label{pro:WassRate}
		Let $P$ be a probability distribution on a bounded set $\mathcal{Z}$. Let $\widehat{P}_{n}$ denote the empirical distribution of $Z_{1}, \ldots, Z_{n} \stackrel{\text { i.i.d. }}{\sim} P.$ Assuming that there exist constants $a>1$  such that $A \defeq \mathbf{E}_{Z \sim P}\bigS{\exp(\norm{Z}^a)} < \infty$ (i.e., $P$ is a light-tail distribution). Then, for any $\rho > 0$,
		$$
		%\mathbf{P}\left[ W (P_{\lambda}, \widehat{P}_{n}) \geq  \rho \right] \leq c_{1} \exp \left(-c_{2} n \rho^{\max\{d/p, 2\}}\right) \\
		\mathbf{P}\left[ W_p (\widehat{P}_{n}, P) \geq  \rho \right]  \leq \begin{cases}c_{1} \exp \left(-c_{2} n \rho^{\max \{d/p,2\}}\right) & \text { if } \rho \leq 1 \\ c_{1} \exp \left(-c_{2} n \rho^{a}\right) & \text { if } \rho>1\end{cases}
		$$
		where $c_{1}, c_{2}$ are constants depending on $a, A$ and $d$.
	\end{proposition}
	As a consequence of this proposition,  for any $\delta > 0$, we have  
	\begin{align}\label{E:deltarho}
		\mathbf{P}\left[ W_2 (\widehat{P}_{n}, P) \leq  \widehat{\rho}_{n}^{\delta} \right] \geq 1 - \delta \; \text{  where  } \; \widehat{\rho}_{n}^{\delta} \defeq \begin{cases}\left(\frac{\log \left(c_{1} / \delta \right)}{c_{2} n}\right)^{\min \{2/ d, 1 / 2\}} & \text { if } n \geq \frac{\log \left(c_{1} / \delta\right)}{c_{2}}, \\ \left(\frac{\log \left(c_{1} / \delta\right)}{c_{2} n}\right)^{1 / \alpha} & \text { if } n<\frac{\log \left(c_{1} / \delta\right)}{c_{2}}.\end{cases}
	\end{align}
	
	\begin{remark} 
		In \Cref{pro:WassRate}, \citet{fournier_rate_2013} show that the empirical distribution $\widehat{P}_n$ converges in Wasserstein distance to the true $P$ at a specific rate. This implies that  judiciously scaling the radius of Wasserstein balls according to \cref{E:deltarho} provides natural confidence regions for the data-generating distribution $P$.  Exploiting this fact for \OurAlg, we obtain the excess risk of ${f}_{\widehat{\theta}^{\varepsilon}}$ w.r.t. its true data distribution $P_{\lambda}$.
	\end{remark}
}

\section{Choosing $\lambda$: generalizing to all client distributions}
\label{sec:lambda:generalization}
{
We show that by calibrating appropriate $\lambda$ value, our proposed algorithm will be capable of generalizing to all client distributions. Suppose we want to cover all client distributions inside a Wassertein ball so that the generalization and robustness result by \OurAlg  in Theorem~\ref{thrm:excess_risk} is applicable to all clients' distributions. This is the problem of finding $\lambda$ such that the Wasserstein distance between $P_{\lambda}$ and $P_j, \forall j$, is as small as possible. Instead of directly finding the minimum  Wasserstein radius  that cover all client distributions, we will leverage the popular Wasserstein barycenter problem \citep{gabriel_2019_cot}. Specifically, consider the problem
	\begin{align} \label{E:barycenter}
		\min_{\lambda \in \Delta} W_2^2(\widehat{P}_{{\lambda}}, P^{\varheart})  \quad \text{s.t. } \quad P^{\varheart} = \argmin_{ Q \in \gP} \SumNoLim{i=1}{m}\lambda_i W_2^2 (\widehat{P}_{n_i}, {Q}), 
	\end{align}
	where $P^{\varheart}$ is the Wasserstein bary center w.r.t the solution $\tilde{\lambda}$ to this problem. Even though the solution is not straightforward, we propose to solve its tractable upper-bound:
	\begin{align} \label{E:barycenter_relaxed}
		\min_{\lambda \in \Delta, Q \in \gP} \SumNoLim{i=1}{m} \lambda_i W_2^2(\widehat{P}_{n_i}, {Q}). 
	\end{align}
	This is a bi-convex problem, which is convex w.r.t to $\lambda$ (resp. $Q$) when fixing $Q$ (resp. $\lambda$). Thus, we can use alternative minimization \citep{gorski_biconvex_2007} to find a local solution to this problem. Denoting $\tilde{\lambda}$ as the solution to \Cref{E:barycenter} and $({\lambda}^*, {P}^*)$ as a local solution to \Cref{E:barycenter_relaxed}, we obtain
	\begin{align}
		W_2^2(\widehat{P}_{\tilde{\lambda}}, P^{\varheart}) \leq   W_2^2(\widehat{P}_{{\lambda}^*}, {P}^*) \leq \sum\nolimits_{i=1}^{m} {\lambda}^*_i W_2^2(\widehat{P}_{n_i}, {P}^*) \eqdef {\rho^{\star}}^2.
	\end{align}
	\begin{corollary} \label{coro_bary}
		For all client $j \in [m]$, with probability at least $1-\delta$, we have
		\begin{align*}
			W_2(P_{\lambda^*}, P_j) &\leq W_2(P_{\lambda^*}, \widehat{P}_{\lambda^*}) + W_2(\widehat{P}_{\lambda^*}, P^*) + W_2( P^*, \widehat{P}_{n_j}) + W_2(  \widehat{P}_{n_j}, P_j) \nonumber\\
			& \leq \sqrt{\sum\nolimits_{i=1}^{m} \lambda_i^{*} \widehat{\rho}_{n_i}^{\delta/m}} + \rho^{\star} + \frac{\rho^{\star}}{\lambda_j} + \widehat{\rho}_{n_j}^{\delta/2} \quad
		\end{align*}
	\end{corollary}
	\begin{proof}
		The first line is by triangle inequality. The second line is by  following facts: (i) $\mathbf{P}\BigS{W_2(P_{\lambda^*}, \widehat{P}_{\lambda^*}) \geq \sqrt{\sum\nolimits_{i=1}^{m} \lambda_i^{*} \widehat{\rho}_{n_i}^{\delta/m}} } \leq \delta/2$ according to \cref{E:barry}, (ii) $W_2( P^*, \widehat{P}_{n_j}) = \frac{\lambda_j W_2( P^*, \widehat{P}_{n_j})}{\lambda_j}\leq \frac{\rho^*}{\lambda_j}$, and (iii)  $\mathbf{P}\bigS{W_2(  \widehat{P}_{n_j}, P_j) \geq \widehat{\rho}_{n_j}^{\delta/2} } \leq \delta/2$ according to \cref{E:deltarho}, and (iv) using union bound. 
\end{proof}}

\section{Proof of Theorem~\ref{Thrm:Convergence}}
\label{proof:thr1}

Our proof is based on the analysis of local SGD for FL presented in \cite{wang_field_2021}.

Fix some $z_i$. Define $\varphi(\zeta, \theta; z_i) \defeq \ell(\zeta, h_\theta) - \gamma d(\zeta, z_i).$  Since $\ell$ is $L_{zz}$-smooth and $d$ is $1$-strongly convex, $\varphi(\zeta, \theta; z_i)$ is $(\gamma - L_{zz})$-strongly concave with respect to $\zeta$, given that $\gamma > L_{zz}$.

\begin{lemma} \label{Lemma:f_smooth}
	Let $z_i^* = \argmax_{\zeta \in \mathcal{Z}}\varphi(\zeta, \theta; z_i)$. Therefore, $\phi_{\gamma}(\theta; z_i) = \varphi(z_i^*, \theta; z_i).$ Let $\ell$ satisfy Assumption \ref{Assumption:smooth_loss}. Then $\phi_{\gamma}$ is differentiable, and
	\begin{align*}
		\norm{\nabla \phi_{\gamma}(z_i, \theta) - \nabla \phi_{\gamma}(z_i, \theta')} \leq L \norm{\theta - \theta'},
	\end{align*}
	with $L = L_{\theta \theta} + \frac{L_{\theta z} L_{z \theta}}{\gamma - L_{zz}}$ when $\gamma > L_{zz}$.
\end{lemma}
The proof can be found in \cite{sinha_certifying_2020}. Lemma \ref{Lemma:f_smooth} implies that $\phi_{\gamma}$ is $L$-smooth.

Define $g_{i}(\theta) \defeq \frac{1}{\abs{\gD_i}} \SumNoLim{z_i \in \gD_i}{}\nabla_{\theta}\, \phi_{\gamma} ({z}_i, \theta)$, then we have 
	\begin{align} \label{E:bounded_var}
&\mathbf{E}\BigS{\norm{g_i(\theta) - \nabla F_i(\theta)}^2} =  \mathbf{E} \BigS{\BigNorm{ \frac{1}{\abs{\gD_i}} \SumLim{z_i \in \gD_i}{}\nabla_{\theta}\, \phi_{\gamma}({z}_i, {\theta}) - \nabla F_i(\theta)}^2} \nonumber\\
&\leq  \frac{1}{\abs{\gD_i}}  \mathbf{E} \BigS{\bigNorm{  g_{\phi_i}( {\theta}) - \nabla F_i(\theta)}^2} 
 \leq \frac{\sigma^2}{\abs{\gD_i}} \defeq \widehat{\sigma}^2  \quad \text{ (by \Cref{Assumption:bounded_variance}.)   }
\end{align}

%Under Assumptions \ref{Assumption:continuous_distance} and \ref{Assumption:smooth_loss}, we have

With the shadow sequence $\bar{\theta}^{(t, k)} = \SumNoLim{i=1}{m} \lambda_i \theta_i^{(t, k)}$, we have
\begin{align*}
	\bar{\theta}^{(t, k+1)} 
	= \SumLim{i=1}{m} \lambda_i \theta_i^{(t, k+1)} 
	= \SumLim{i=1}{m} \lambda_i \bigP{\theta_i^{(t, k)} - \eta g_i(\theta_i^{(t, k)})} 
	= \bar{\theta}^{(t, k)} - 
	\eta
	\SumLim{i=1}{m}
	\lambda_i g_i(\theta_i^{(t, k)}).
\end{align*}
\begin{lemma} \label{lem:progress}
	If the client learning rate satisfies $\eta \leq \frac{1}{3 L}$, then
	\begin{align*}
%	\mathbf{E} 
%	\BigS{
%		F(\bar{\theta}^{(t, k^*)}) - F({\theta}^*) 
%	}
%	\leq &
%	\frac{\eta \hat{\sigma}^2}{\Lambda}
%	+ L \SumLim{i=1}{m}
%	\lambda_i
%	\SumLim{k=0}{K-1}
%	\frac{\omega_k}{W_K} 
%	\mathbf{E} 
%	\BigS{
%		\bigNorm{\theta_i^{(t, k)} - \bar{\theta}^{(t, k)}}^2
%	}
%	\\
%	& +\frac{1}{2 \eta }
%	{
%		\bigNorm{{\theta}^{t} - {\theta}^*}^2 
%	}
	\frac{1}{K}\SumLim{k=0}{K-1}
\mathbf{E} 
\biggS{
	F(\bar{\theta}^{(t, k)}) - F({\theta}^*) 
}
\leq &
2 \eta \hat{\sigma}^2 \BiggP{\SumLim{i=1}{m} \lambda_i^2}	
+ 
L \SumLim{i=1}{m}
\lambda_i
\SumLim{k=0}{K-1}
\frac{1}{K} 
\mathbf{E} 
\BigS{
	\norm{\theta_i^{(t, k)} - \bar{\theta}^{(t, k)}}^2
}
\\
%	&+ \frac{1}{2 \eta} \SumLim{k=0}{K-1}
%	  \frac{\omega_k \delta_k - \omega_{k+1} \delta_{k+1}}{W_K}.
%	\\
& +\frac{1}{2 \eta K}
\BiggP{
	\norm{{\theta}^{(t)} - {\theta}^*}^2 
	- \mathbf{E}
	\BigS{
		\norm{{\theta}^{(t+1)} - {\theta}^*}^2
	}
}.
	\end{align*}
%	where $\hat{\sigma}^2 = \sigma^2 + \frac{2 L_{\theta z}^2}{\gamma - L_{zz}} \nu$, and $k^*$ is a random index  sampled according to the distribution $p_k = {\omega_k}/{W_K}, k=0, \ldots K-1$. 
\end{lemma}

\begin{proof}
Since $\bar{\theta}^{(t, k+1)} = \bar{\theta}^{(t, k)} - \eta \SumNoLim{i=1}{m} \lambda_i g_i(\theta_i^{(t, k)})$, by parallelogram law
\begin{align} \label{Eq:parallelogram}
	\SumLim{i=1}{m} &\lambda_i \innProd{g_i(\theta_i^{(t, k)}), \bar{\theta}^{(t, k+1)} - {\theta}^*} =
	\frac{1}{2 \eta} 
	\BigP{
		\norm{\bar{\theta}^{(t, k)} - {\theta}^*}^2
		- \norm{\bar{\theta}^{(t, k+1)} - \bar{\theta}^{(t, k)}}^2
		- \norm{\bar{\theta}^{(t, k+1)} - {\theta}^*}^2 }.
\end{align}
%where last inequality is due to  the relaxed triangle formula as follows with $\alpha = 2$:
%\begin{align*}
%- \frac{1}{2} \norm{\bar{\theta}^{(t, k+1)} - \bar{\theta}^{(t, k)}}^2 \leq \frac{1}{2}\BigP{ - \frac{1}{1 + \alpha} \norm{\bar{\theta}^{(t, k)} - {\theta}^*}^2 + \frac{1 + 1/\alpha}{1 + \alpha}\norm{\bar{\theta}^{(t, k+1)} - {\theta}^*}^2}
%\end{align*}

Fact: $F_i(\theta) = \mathbf{E}_{Z_i \sim {P}_i}\bigS{\phi_\gamma(Z_i, \theta)}$ is Lipschitz smooth with $L = L_{\theta \theta} + \frac{L_{\theta z} L_{z \theta}}{\gamma - L_{zz}}$ when $\gamma > L_{zz}$. With the assumption that  $\theta \mapsto \ell(z, h_{\theta})$ is convex, we  have  $\theta \mapsto F_i(\theta)$ is convex. 

Since $F_i$ is convex and $L$-smooth,
\begin{align}
	& F_i(\bar{\theta}^{(t, k+1)}) \leq F_i(\theta_i^{(t, k)}) 
	+ \innProd{\nabla F_i(\theta_i^{(t, k)}), \bar{\theta}^{(t, k+1)} - \theta_i^{(t, k)}}
	+ \frac{L}{2} \norm{\bar{\theta}^{(t, k+1)} - \theta_i^{(t, k)}}^2
	\nonumber \\
	\leq  
	&F_i({\theta}^*) 
	+ \innProd{\nabla F_i(\theta_i^{(t, k)}), 
			   \bar{\theta}^{(t, k+1)} - {\theta}^*}
	+ \frac{L}{2} \norm{\bar{\theta}^{(t, k+1)} - \theta_i^{(t, k)}}^2
	\nonumber \\
	\leq & 
	F_i({\theta}^*)
	+ \innProd{\nabla F_i(\theta_i^{(t, k)}), \bar{\theta}^{(t, k+1)} - {\theta}^*}
	+ L \norm{\bar{\theta}^{(t, k+1)} - \bar{\theta}^{(t, k)}}^2
	+ L \norm{\theta_i^{(t, k)} - \bar{\theta}^{(t, k)}}^2.
	\label{Eq:smooth_strongly_convex}
\end{align}
From \cref{Eq:parallelogram} and \cref{Eq:smooth_strongly_convex}, we have
\begin{equation}
\begin{aligned}
	& F(\bar{\theta}^{(t, k+1)}) - F({\theta}^*) 
	=
	\SumLim{i=1}{m}
	\lambda_i
	\BigP{F_i(\bar{\theta}^{(t, k+1)}) - F({\theta}^*)}
	\\ %\nonumber \\
	\leq & 
	\SumLim{i=1}{m}
	\lambda_i
	\innProd{\nabla F_i(\theta_i^{(t, k)}) - g_i(\theta_i^{(t, k)}),
	         \bar{\theta}^{(t, k+1)} - {\theta}^*}
    + L \norm{\bar{\theta}^{(t, k+1)} - \bar{\theta}^{(t, k)}}^2 
    + L \SumLim{i=1}{m} \lambda_i \norm{\theta_i^{(t, k)} - \bar{\theta}^{(t, k)}}^2 
    \\%\nonumber \\
    & +
   \frac{1}{2 \eta} 
   \BigP{
   	\norm{\bar{\theta}^{(t, k)} - {\theta}^*}^2
   	- \norm{\bar{\theta}^{(t, k+1)} - \bar{\theta}^{(t, k)}}^2
   	- \norm{\bar{\theta}^{(t, k+1)} - {\theta}^*}^2 
   }. \label{Eq:shadow_optimality_gap}
\end{aligned}
\end{equation}

%Since $\mathbf{E} \bigS{\nabla F_i(\theta_i^{(t, k)}) - g_i(\theta_i^{(t, k)})} = 0$, we have
We have
\begin{align}
	& \mathbf{E} 
	\BigS{
		\SumLim{i=1}{m}
		\lambda_i
		\innProd{\nabla F_i(\theta_i^{(t, k)}) - g_i(\theta_i^{(t, k)}),
			\bar{\theta}^{(t, k+1)} - {\theta}^*}
	}
	\nonumber \\
	&= \mathbf{E} 
	\BigS{
		\SumLim{i=1}{m}
		\lambda_i
		\innProd{\nabla F_i(\theta_i^{(t, k)}) - g_i(\theta_i^{(t, k)}),
			\bar{\theta}^{(t, k+1)} - \bar{\theta}^{(t, k)}} }  \quad \text{(since $ \mathbf{E}\bigS{g_i(\theta_i^{(t, k)}) } = \nabla F_i(\theta_i^{(t, k)})$ given $\bar{\theta}^{(t, k)}, {\theta}^*$)}
	\nonumber \\
	&\leq 
	\frac{3}{2}\eta
	\cdot 
	\mathbf{E} 
	\BigS{
		\norm{
			\SumLim{i=1}{m}
			\lambda_i
			\bigP{\nabla F_i(\theta_i^{(t, k)}) - g_i(\theta_i^{(t, k)})}
		}^2
	} 
	+
	\frac{1}{6 \eta}
	\mathbf{E} 
	\BigS{
		\norm{
			\bar{\theta}^{(t, k+1)} - \bar{\theta}^{(t, k)}
		}^2
	}
	 \quad \text{(by Peter Paul inequality)}\nonumber \\
	&\leq 
	 2 \eta \hat{\sigma}^2 \biggP{\SumLim{i=1}{m} \lambda_i^2}
	+
	\frac{1}{6 \eta}
	\mathbf{E} 
	\BigS{
		\norm{
				\bar{\theta}^{(t, k+1)} - \bar{\theta}^{(t, k)}
		}^2
	}, \label{Eq:inn_prod_bound}
\end{align}

Plugging \cref{Eq:inn_prod_bound} back to the conditional expectation of \cref{Eq:shadow_optimality_gap}, and noting that $\eta \leq \frac{1}{3L}$, we have
\begin{align*}
	& \mathbf{E}
	\BigS{
		F(\bar{\theta}^{(t, k+1)}) - F({\theta}^*)
	}
	+ \frac{1}{2 \eta}
	\BiggP{
		 \mathbf{E} 
		\biggS{
			\norm{
				\bar{\theta}^{(t, k+1)} - {\theta}^*
			}^2
		}
		-
		\norm{\bar{\theta}^{(t, k)} - {\theta}^*}^2
	}
	\\
	\leq &
	2 \eta \hat{\sigma}^2 \BiggP{\SumLim{i=1}{m} \lambda_i^2}
	- \BiggP{\frac{1}{3 \eta} - L} \mathbf{E} 
	\biggS{
		\norm{
			\bar{\theta}^{(t, k+1)} - \bar{\theta}^{(t, k)}
		}^2
	}
	+ L
	\SumLim{i=1}{m}
	\lambda_i
	\norm{\theta_i^{(t, k)} - \bar{\theta}^{(t, k)}}^2
	\\
	\leq &
	2 \eta \hat{\sigma}^2 \BigP{\SumLim{i=1}{m} \lambda_i^2}
	+
	L
	\SumLim{i=1}{m}
	\lambda_i
	\norm{\theta_i^{(t, k)} - \bar{\theta}^{(t, k)}}^2
\end{align*}

%Multiplying both sizes with ${\omega_k}/{W_K}$, where $\omega_k = (1/2)^k$ and $W_K = \sum\nolimits_{k=0}^{K-1} \omega_k = 2 (1 - {1}/{2}^K) \geq 1, \forall K \geq 1$, and define $\delta_k \defeq \mathbf{E} \bigS{\norm{\bar{\theta}^{(t, k)} - {\theta}^*}^2}$, we have
By convexity of $F$ and  telescoping $k$ from $0$ to $K-1$, we have
\begin{align*}
		\frac{1}{K}\SumLim{k=0}{K-1}
		\mathbf{E} 
		\biggS{
		F(\bar{\theta}^{(t, k)}) - F({\theta}^*) 
	 }
	\leq &
	2 \eta \hat{\sigma}^2 \BiggP{\SumLim{i=1}{m} \lambda_i^2}	
	+ 
	L \SumLim{i=1}{m}
	\lambda_i
	\SumLim{k=0}{K-1}
	\frac{1}{K} 
	\mathbf{E} 
	\BigS{
	\norm{\theta_i^{(t, k)} - \bar{\theta}^{(t, k)}}^2
	}
	\\
	& +\frac{1}{2 \eta K}
	\BigP{
		\norm{\bar{\theta}^{(t, 0)} - {\theta}^*}^2 
		- \mathbf{E}
		\BigS{
			\norm{\bar{\theta}^{(t, K)} - {\theta}^*}^2
		}
	}.
\end{align*}
Since $\bar{\theta}^{(t, 0)} = {\theta}^{(t)}$ and $\bar{\theta}^{(t, K)} = {\theta}^{(t+1)}$, we complete the proof. 
%By convexity of $F$, $W_K \geq 1$,  and  telescoping $k$ from $0$ to $K-1$, we have
%\begin{align*}
%\mathbf{E} 
%\biggS{
%	F\BigP{\SumNoLim{k=0}{K-1} \frac{\omega_k}{W_K}  \bar{\theta}^{(t, k)} }- F({\theta}^*) 
%}
%\leq &
%\eta \hat{\sigma}^2 \BiggP{\SumLim{i=1}{m} \lambda_i^2}	
%+ 
%L \SumLim{i=1}{m}
%\lambda_i
%\SumLim{k=0}{K-1}
%\frac{\omega_k}{W_K} 
%\mathbf{E} 
%\BigS{
%	\norm{\theta_i^{(t, k)} - \bar{\theta}^{(t, k)}}^2
%}
%\\
%&+ \frac{1}{2 \eta}   \bigP{\delta_0 - \delta_K}.
%\end{align*}
\end{proof}

{
\begin{lemma}[Bounded client drift] \label{lem:bounded_client_drift}
	Assuming the client learning rate satisfies $\eta \leq  \frac{1}{3L}$, we have
	\begin{align*}
		\mathbf{E}\BigS{\norm{\theta_i^{(t, k)} - \bar{\theta}^{(t, k)}}^2} \leq  \eta^2( 24 K^2 \bar{\Omega}^2 + 20 K \bar{\Omega}^2). 
	\end{align*}
where $\bar{\Omega}^2 \defeq \max\bigC{\widehat{\sigma}^2 , \Omega^2}$. 
\end{lemma}

\begin{proof}	
\begin{align}
	\mathbf{E} & \biggS{\bigNorm{\theta_{1}^{(t, k+1)} - \theta_{2}^{(t, k+1)}}^2} = \mathbf{E}\biggS{\BigNorm{\theta_{1}^{(t, k)} - \theta_{2}^{(t, k)} - \eta\BigP{g_{1}(\theta_{1}^{(t, k)}) -  g_{2}(\theta_{1}^{(t, k)})}}^2} \nonumber \\
	= & \norm{\theta_{1}^{(t, k)} - \theta_{2}^{(t, k)}}^2 - 2 \eta \innProd{g_{1}(\theta_{1}^{(t, k)}) - \nabla F_{1}(\theta_1^{(t,k)}), \theta_{1}^{(t, k)} - \theta_{2}^{(t, k)}} \nonumber \\
	& - 2 \eta \innProd{\nabla F_{2}(\theta_1^{(t,k)})- g_{2}(\theta_{2}^{(t, k)}), \theta_{1}^{(t, k)} - \theta_{2}^{(t, k)}} \nonumber\\
	&- 2 \eta \innProd{\nabla F_{1}(\theta_1^{(t,k)}) - \nabla F_{2}(\theta_2^{(t,k)}), \theta_{1}^{(t, k)} - \theta_{2}^{(t, k)}} +\eta^2 \norm{g_{1}(\theta_{1}^{(t, k)}) - g_{2}(\theta_{2}^{(t, k)})}^2.  \label{E:localrounds}
%	\leq & \norm{\theta_{1}^{(t, k)} - \theta_{2}^{(t, k)}}^2 - 2 \eta \innProd{\nabla F_{1}(\theta_1^{(t,k)}) - \nabla F_{2}(\theta_2^{(t,k)}), \theta_{1}^{(t, k)} - \theta_{2}^{(t, k)}} \\
%	& + \eta^2 \norm{\nabla F_{1}(\theta_1^{(t,k)}) - \nabla F_{2}(\theta_2^{(t,k)})}^2 + 2 \eta^2 \widehat{\sigma}^2. &\text{(by \Cref{Lemma:bounded_variance_surrogate})}
\end{align}
The second term (and similarly for the third term)  is bounded as follows
\begin{align*}
- &\innProd{g_{1}(\theta_{1}^{(t, k)}) - \nabla F_{1}(\theta_1^{(t,k)}), \theta_{1}^{(t, k)} - \theta_{2}^{(t, k)}} \\
&\leq  \frac{1}{6 \eta K}  \bigNorm{\theta_1^{(t,k)} - \theta_2^{(t,k)}}^2 + \frac{3 \eta K}{2} \bigNorm{g_{1}(\theta_{1}^{(t, k)}) - \nabla F_{1}(\theta_1^{(t,k)})}^2 \quad \text{(by Peter Paul inequality)}\\
& = \frac{1}{6 \eta K}  \bigNorm{\theta_1^{(t,k)} - \theta_2^{(t,k)}}^2 + \frac{3 \eta K}{2} \widehat{\sigma}^2    \qquad\text{(by \Cref{E:bounded_var})}
\end{align*}
Since $\max_{i} \sup_{\theta} \norm{\nabla F_i(\theta) - \nabla F(\theta)} \leq \Omega$ (\Cref{Assumption:bounded_gradient_surrogate}), the 4th-term is bounded as
\begin{align*}
	&-\innProd{\nabla F_{1}(\theta_1^{(t,k)}) - \nabla F_{2}(\theta_2^{(t,k)}), \theta_{1}^{(t, k)} - \theta_{2}^{(t, k)}}\\
	\leq & -\innProd{\nabla F(\theta_1^{(t,k)}) - \nabla F(\theta_2^{(t,k)}), \theta_{1}^{(t, k)} - \theta_{2}^{(t, k)}} + 2 \Omega \norm{\theta_1^{(t,k)} - \theta_2^{(t,k)}}\\
	\leq & -\frac{1}{L} \norm{\nabla F(\theta_1^{(t,k)}) - \nabla F(\theta_2^{(t,k)})}^2 + 2 \Omega \norm{\theta_1^{(t,k)} - \theta_2^{(t,k)}} &\text{(by smoothness and convexity)}\\
	\leq & -\frac{1}{L} \norm{\nabla F(\theta_1^{(t,k)}) - \nabla F(\theta_2^{(t,k)})}^2 + \frac{1}{6 \eta K} \norm{\theta_1^{(t,k)} - \theta_2^{(t,k)}}^2 + 6 \eta K \Omega^2 &\text{(by Young's inequality)}
\end{align*}

The last term is bounded as follows
\begin{align*}
&\norm{g_{1} (\theta_{1}^{(t, k)}) - g_{2}(\theta_{2}^{(t, k)})}^2  \\
&\leq 5 \BigP{\bigNorm{g_{1}(\theta_{1}^{(t, k)}) - \nabla F_{1}(\theta_1^{(t,k)})}^2 +
	\norm{\nabla F_{1}(\theta_1^{(t,k)}) - \nabla F(\theta_1^{(t,k)})}^2  + \norm{\nabla F(\theta_1^{(t,k)}) - \nabla F (\theta_2^{(t,k)})}^2 \\
	& \quad + \norm{\nabla F(\theta_2^{(t,k)}) - \nabla F_{2}(\theta_2^{(t,k)})}^2 + \bigNorm{g_{2}(\theta_{2}^{(t, k)}) - \nabla F_{2}(\theta_2^{(t,k)})}^2} \\
& \leq 5 \norm{\nabla F(\theta_1^{(t,k)}) - \nabla F(\theta_2^{(t,k)})}^2 + 10 (\widehat{\sigma}^2 + \Omega^2) \quad \text{(by \Cref{E:bounded_var} and \Cref{Assumption:bounded_gradient_surrogate})}
\end{align*}

Substituting the above four bounds back to \cref{E:localrounds} gives (note that $\eta \leq \frac{1}{3L}$)
\begin{align*}
	\mathbf{E}&\BigS{\norm{\theta_{1}^{(t, k+1)} - \theta_{2}^{(t, k+1)}}^2} \leq \BigP{1 + \frac{1}{K}} \norm{\theta_{1}^{(t, k)} - \theta_{2}^{(t, k)}}^2 - \eta \BigP{\frac{2}{L} - 5 \eta} \bigNorm{\nabla F(\theta_1^{(t,k)}) - \nabla F(\theta_2^{(t,k)})}^2 \\
	& \quad + 6 \eta^2 K \widehat{\sigma}^2 + 12  \eta^2 K \Omega^2 + 10 \eta^2 (\widehat{\sigma}^2 + \Omega^2)  \\
	& \leq \BigP{1 + \frac{1}{K}} \norm{\theta_{1}^{(t, k)} - \theta_{2}^{(t, k)}}^2 + 6 \eta^2 K \widehat{\sigma}^2 + 12  \eta^2 K \Omega^2 + 10 \eta^2 (\widehat{\sigma}^2 + \Omega^2).
\end{align*}

Unrolling recursively, we obtain
\begin{align*}
	\mathbf{E}\BigS{\bigNorm{\theta_{1}^{(t, k+1)} - \theta_{2}^{(t, k+1)}}^2} &\leq \frac{\bigP{1 + {1}/{K}}^K - 1}{{1}/{K}}  \BigS{ 6 \eta^2 K \widehat{\sigma}^2 + 12  \eta^2 K \Omega^2 + 10 \eta^2 (\widehat{\sigma}^2 + \Omega^2)}  \\
	&\leq  12 \eta^2 K^2 \widehat{\sigma}^2 + 24 \eta^2 K^2 \Omega^2 + 20 \eta^2 K (\widehat{\sigma}^2 + \Omega^2)\\
	&\leq  \eta^2( 24 K^2 \bar{\Omega}^2 + 20 K \bar{\Omega}^2). 
\end{align*}
where we use the fact that $\frac{\bigP{1 + {1}/{K}}^K - 1}{{1}/{K}} \leq K (e - 1) \leq 2K$, and $\bar{\Omega}^2 \defeq \max\bigC{\widehat{\sigma}^2 , \Omega^2}$. 

By convexity, for any $i$,
\begin{align*}
\mathbf{E}\BigS{\norm{\theta_{i}^{(t, k+1)} - \bar{\theta}^{(t, k+1)}}^2} 
%&\leq 12 \eta^2 K^2 \widehat{\sigma}^2 + 24 \eta^2 K^2 \Omega^2 + 20 \eta^2 K (\widehat{\sigma}^2 + \Omega^2). \\
\leq  \eta^2( 24 K^2 \bar{\Omega}^2 + 20 K \bar{\Omega}^2). 
\end{align*}
\end{proof}

Substituting the result of \Cref{lem:bounded_client_drift} to \Cref{lem:progress}, and telescoping over $t$, we obtain
\begin{align*}
\mathbb{E}\BigS{\frac{1}{T} \SumLim{t=0}{T-1} \frac{1}{K} \SumLim{k=0}{K-1} F(\bar{\theta}^{(t, k)}) - F({\theta}^*)} 
\leq  &\frac{D^2}{2 \eta K T}
+	2\eta \hat{\sigma}^2 \Lambda  + \eta^2 L ( 24 K^2 \bar{\Omega}^2 + 20 K \bar{\Omega}^2), 
\end{align*}
%\begin{theorem}[Convergence rate for convex local functions]
%	Under the aforementioned assumptions, if the client learning rate satisfies $\eta \leq \frac{1}{2L}$ then
%	\begin{align*}
%	\mathbb{E}\BigS{\frac{1}{T} \SumLim{t=0}{T-1}  F(\bar{\theta}^{(t, k^*)}) - F({\theta}^*)} 
%	\leq  	\eta \hat{\sigma}^2 \BiggP{\SumLim{i=1}{m} \lambda_i^2}	
%	+ 
%	L \SumLim{i=1}{m}
%	\lambda_i
%	\SumLim{k=0}{K-1}
%	\frac{1}{K} 
%	\mathbf{E} 
%	\BigS{
%		\norm{\theta_i^{(t, k)} - \bar{\theta}^{(t, k)}}^2
%	}
%	\\
%	%	&+ \frac{1}{2 \eta} \SumLim{k=0}{K-1}
%	%	  \frac{\omega_k \delta_k - \omega_{k+1} \delta_{k+1}}{W_K}.
%	%	\\
%	& +\frac{1}{2 \eta }
%	\BigP{
%		\norm{\bar{\theta}^{(t, 0)} - {\theta}^*}^2 
%		- \mathbf{E}
%		\BigS{
%			\norm{\bar{\theta}^{(t, K)} - {\theta}^*}^2
%		}
%	}\\
%	\frac{D^2}{2 \eta K T} + \frac{\eta \widehat{\sigma}^2}{\Lambda}+ 4 K \eta^2 L \widehat{\sigma}^2 + 18 K^2 \eta^2 L \Omega^2,
%	\end{align*}
%\end{theorem}
where $D \defeq \norm{\theta^{(0)} - {\theta}^*}$,  $ \Lambda \defeq {\SumLim{i=1}{m} \lambda_i^2}$. By optimizing $\eta$ on the R.H.S, we obtain
\begin{align*}
\mathbb{E}\BigS{\frac{1}{KT} \SumLim{t=0}{T-1} \SumLim{k=0}{K-1}  F(\bar{\theta}^{(t, k)}) - F({\theta}^*)} 
\leq  \gO \biggP{{\frac{ L D^2}{KT} + 
	\frac{ \widehat{\sigma} D \Lambda^{\frac{1}{2}}}{\sqrt{ KT}}} + 
{\frac{ L^{\frac{1}{3}} \bar{\Omega}^{\frac{2}{3}} D^{\frac{4}{3}}}{K^{\frac{1}{3}} T^{\frac{2}{3}}} + 
	\frac{ L^{\frac{1}{3}} \bar{\Omega}^{\frac{2}{3}} D^{\frac{4}{3}}}{T^{\frac{2}{3}}}}},
\end{align*}
when 
\begin{align*}
\eta = \min{\set{\frac{1}{3L}, \frac{ D}{2 \sqrt{ K T\Lambda} \widehat{\sigma}}, \frac{D^{\frac{2}{3}}}{48^{\frac{1}{3}}K^{\frac{2}{3}} T^{\frac{1}{3}} L^{\frac{1}{3}} \bar{\Omega}^{\frac{2}{3}}}, \frac{D^{\frac{2}{3}}}{40^{\frac{1}{3}} K T^{\frac{1}{3}} L^{\frac{1}{3}} \bar{\Omega}^{\frac{2}{3}}}}}. 
\end{align*}
}

%\section{Convergence with approximation}
%\label{proof:Convergence}
%\input{Sections/convergence_with_approx_tmp}

\section{Proof of Lemma~\ref{lem:excess_risk}}
\label{proof:excess_risk}
	We first prove the following fact:
	
	$\textbf{Fact 1:}$
	\begin{align*}
	%\textbf{Fact 1:} 
	\vspace{-1mm}
	\quad (a) \quad &\mathscr{L}(Q, f) \leq  \mathscr{L}_{\rho}^{\gamma}(P_{\lambda}, f), \qquad \forall f \in \mathcal{F},  Q \in \gB(P_{\lambda}, \rho).   \\
	 (b) \quad  &\inf_{f' \in \mathcal{F}} \mathscr{L}(Q, f') \leq \inf_{f' \in \mathcal{F}} \mathscr{L}_{\rho}^{\gamma}(P_{\lambda}, f') , \quad \forall Q \in \gB(P_{\lambda}, \rho).  
	\end{align*}
	
	For (a), we have
	\begin{align*}
	\mathscr{L}(Q,f) \leq  \underset{P' \in \gB({P}_{\lambda}, \rho)}{\mbox{sup }} \mathscr{L} (P',f) =  \, &\underset{\gamma' \geq 0}{{ \inf }} \Big{\{} \gamma' \rho^2 + \mathbf{E}_{Z \sim {P}_{\lambda}}\Big{[}\phi_\gamma (Z,  f)\Big{]}  \Big{\}} \\
	&  \leq   \gamma \rho^2 + \mathbf{E}_{Z \sim {P}_{\lambda}}\Big{[}\phi_\gamma (Z,  f)\Big{]}  \eqdef \mathscr{L}_{\rho}^{\gamma}(P_{\lambda},f),
	%	\mathscr{L}(Q,f) \leq  \underset{S \in \gB({P}, \rho)}{\mbox{sup }} \mathscr{L} (S,f) \leq  \mathscr{L}(P_{\lambda},\phi_{\gamma})
	\end{align*}
	where the equality is due to strong  duality result by~\citet{gao_distributionally_2016}. %\blue{Gao and Kleywegt [Distributionally robust stochastic optimization with Wasserstein distance]}.

	For (b), defining  $f_{P_{\lambda}} \defeq \argmin_{f' \in \mathcal{F}} \mathscr{L}_{\rho}^{\gamma}(P_{\lambda}, f')$, we have %and its corresponding hyposthesis $h_{P_{\lambda}} \defeq \argmin_{h' \in \mathcal{H}} \bigC{\gamma \rho^2 + \mathbf{E}_{Z \sim {P}_{\lambda}}\Big{[}\phi_\gamma (Z,  h')\Big{]}}$, we have
	\begin{align} 
	\inf_{f' \in \mathcal{F}} \mathscr{L}(Q, f') \leq \mathscr{L}(Q,f_{P_{\lambda}}) &\leq  \underset{P' \in \gB({P}_{\lambda}, \rho)}{\mbox{sup }} \mathscr{L} (P',f_{P_{\lambda}})  \\ 
	&= \underset{\gamma' \geq 0}{{ \inf }} \Big{\{} \gamma' \rho^2 + \mathbf{E}_{Z \sim {P}_{\lambda}}\Big{[}\phi_\gamma (Z,  f_{P_{\lambda}})\Big{]}  \Big{\}} \\ 
	&\leq \gamma \rho^2 + \mathbf{E}_{Z \sim {P}_{\lambda}}\Big{[}\phi_\gamma (Z,  f_{P_{\lambda}})\Big{]} \\
	& = \inf_{f' \in \mathcal{F}} \mathscr{L}_{\rho}^{\gamma}(P_{\lambda}, f'). 
	\end{align}
	
	We next prove the second fact:
	
	$\textbf{Fact 2:}$%\quad   \mathscr{L}_{\rho}^{\gamma}(P_{\lambda},f) \leq  \mathscr{L}(Q,f) + 2L_z \rho  + \abs{\gamma - \gamma^*} \cdot \max\bigC{\rho, D}, \quad \forall Q \in \gB(P_{\lambda}, \rho).$ 
	\begin{align*}
	\vspace{-1mm}
	\quad (a) \quad &\mathscr{L}_{\rho}^{\gamma}(P_{\lambda},f) \leq  \mathscr{L}(Q,f) + 2L_z \rho   + \abs{\gamma - \gamma^*} {\rho}^2, \quad \forall f \in \gF, Q \in \gB(P_{\lambda}, \rho)   \\
	(b) \quad  &\inf_{f' \in \mathcal{F}} \mathscr{L}_{\rho}^{\gamma}(P_{\lambda}, f') \leq   \inf_{f' \in \mathcal{F}} \mathscr{L}(Q, f') + 2L_z \rho   + \abs{\gamma - \gamma^*} {\rho}^2.
	   %\abs{\gamma - \gamma^*} \cdot \max\bigC{\rho, D}, \quad \forall Q \in \gB(P_{\lambda}, \rho).
	\end{align*}
	
	For (a), we have:
	\begin{align*} 
	\mathscr{L}_{\rho}^{\gamma}(P_{\lambda},f)   & =  \biggC{ \underset{P' \in \gB({P}_{\lambda}, \rho)}{\mbox{sup }} \mathscr{L} (P',f)} + \BigC{\mathscr{L}_{\rho}^{\gamma}(P_{\lambda},f)  - \underset{P' \in \gB({P}_{\lambda}, \rho)}{\mbox{sup }} \mathscr{L} (P',f)} \\
	&\leq  \BigC{ \mathscr{L}(Q,f) + 2L_z\rho}  + \biggC{\mathbf{E}_{Z \sim P_{\lambda}}[\phi_{\gamma}(Z, f)] + {\rho}^2\gamma - \min_{\gamma^{\prime} \geq 0} \Big\{ {\rho}^2 \gamma^{\prime} + \mathbf{E}_{Z \sim P_{\lambda}}[\phi_{\gamma'}(Z, f)] \Big\}} \\
%	& \leq   \mathscr{L}(Q,f) + 2L_z\rho + \max_{\gamma^{\prime} \geq 0}\bigg\{{\rho}(\gamma - \gamma^{\prime}) + \mathbf{E}_{Z \sim P}\Big[\phi_{\gamma}(Z, f) - \phi_{\gamma'}(Z, f)\Big]\bigg\} \\
	& \leq   \mathscr{L}(Q,f) + 2L_z\rho +  {\rho}^2(\gamma - \gamma^*) + \mathbf{E}_{Z \sim P}\Big[\phi_{\gamma}(Z, f) - \phi_{\gamma^*}(Z, f)\Big] \\
	%& =  \begin{aligned}[t] \mathscr{L}(Q,f) + 2L_z\rho + \max_{\gamma^{\prime} \geq 0}\Bigg\{\rho(\gamma - \gamma^{\prime}) + \mathbf{E}_{Z \sim P}\bigg[& \sup_{\zeta \in \mathcal{Z}} \Big\{ \ell(\zeta,h)  - \gamma d(\zeta,  Z) \Big\} \\ 
	& =  \begin{aligned}[t] \mathscr{L}(Q,f) + 2L_z\rho + {\rho}^2(\gamma - \gamma^*) + \mathbf{E}_{Z \sim P}\bigg[& \sup_{\zeta \in \mathcal{Z}} \Big\{ \ell(\zeta,h)  - \gamma d(\zeta,  Z) \Big\}  - \sup_{\zeta \in \mathcal{Z}} \Big\{ \ell(\zeta,h) - \gamma^* d^2(\zeta,  Z)  \Big\} \bigg] \end{aligned} \\
%	& =  \mathscr{L}(Q,f) + 2L_z\rho + \rho(\gamma - \gamma^*) - (\gamma - \gamma^*) \mathbf{E}_{Z \sim P}\bigg[\sup_{\zeta \in \mathcal{Z}} \Big\{ d(\zeta,  Z)\Big\}\bigg] \\
%	& =  \mathscr{L}(Q,f) + 2L_z\rho + \underbrace{(\gamma - \gamma^*) \bigg(\rho - \mathbf{E}_{Z \sim P}\Big[\sup_{\zeta \in \mathcal{Z}} d(\zeta,  Z)\Big]\bigg)}_{A} \\
	& =  \mathscr{L}(Q,f) + 2L_z\rho +  (\gamma - \gamma^*) \BigP{{\rho}^2 - \mathbf{E}_{Z \sim P}\Big[\sup_{\zeta \in \mathcal{Z}} d^2(\zeta,  Z)\Big] }\\
	% & \leq   \mathscr{L}(Q,f) + 2L_z \rho  + \abs{\gamma - \gamma^*} \cdot \max\bigC{\rho, D},
	& \leq   \mathscr{L}(Q,f) + 2L_z \rho  + \abs{\gamma - \gamma^*}  {\rho}^2,
	\end{align*}
	where the first inequality is due to Proposition~\ref{pro1}, and the last inequality is because we choose $\gamma \geq L_z/\rho$ and that fact that $\gamma^* \leq L_z/\rho$ by Lemma 1 of~\citet{lee_minimax_2018}. %\blue{[Raginski]}. 

%	and the last inequality is because  
%	\begin{align*}	
%	A \leq \left\{\begin{array}{cc} \bigP{\gamma - \gamma^*}{\rho} & \text { if } \gamma \geq \gamma^*   \\ 
%	\\ \bigP{\gamma^* - \gamma} D & \text { if }  \gamma < \gamma^*.\end{array}\right.
%	\end{align*}

%	We then prove the third fact:
%	
%	$\textbf{Fact 3:}\quad  	\inf_{f' \in \mathcal{F}} \mathscr{L}_{\rho}^{\gamma}(P_{\lambda}, f') \leq   \inf_{f' \in \mathcal{F}} \mathscr{L}(Q, f') + 2L_z \rho  + \abs{\gamma - \gamma^*} \cdot \max\bigC{\rho, D}, \quad \forall Q \in \gB(P_{\lambda}, \rho).$ 
	
	For (b), defining  $f_Q \defeq \argmin_{f \in \mathcal{F}} \mathscr{L}(Q, f)$, we have
	\begin{align}
	%	\mathscr{L}(Q, f) &\leq  \mathscr{L}_{\rho}^{\gamma}(P_{\lambda}, f),  \\ 
	\inf_{f' \in \mathcal{F}} \mathscr{L}_{\rho}^{\gamma}(P_{\lambda}, f') &\leq  \mathscr{L}_{\rho}^{\gamma}(P_{\lambda}, f_Q)  \\
	&\leq \mathscr{L}(Q, f_{Q}) + 2L_z \rho   + \abs{\gamma - \gamma^*} {\rho}^2 \\
	&= \inf_{f' \in \mathcal{F}} \mathscr{L}(Q, f') + 2L_z \rho   + \abs{\gamma - \gamma^*} {\rho}^2,
	\end{align}
	where the second line is due to \textbf{Fact 2}(a). 
	
	Combining all facts, we complete the proof. Specifically, by adding two inequalities in \textbf{Fact 1}(a) and \textbf{Fact 2}(b), we obtain the upperbound of Lemma~\ref{lem:excess_risk}. Similarly, adding two inequalities in \textbf{Fact 1}(b) and \textbf{Fact 2}(a), we obtain the lowerbound of this lemma. 
	
	Finally, we provide the proof of the following proposition that was used in proving \textbf{Fact 2}(a).
	\begin{proposition} \label{pro1} Let Assumption~\ref{Assumption:Lip_cont} (a) holds. For any $f \in \gF$ and for all  $Q \in \gB(P_{\lambda}, \rho)$, we have
		\begin{align*}
		\underset{P' \in \gB({P}_{\lambda}, \rho)}{\sup} \mathscr{L} (P',f) \leq \mathscr{L}(Q,f) + 2 L_z  \rho. 
		\end{align*}
	\end{proposition}
	\begin{proof} Denote $P^* \defeq \underset{P' \in \gB({P}_{\lambda}, \rho)}{\argmax} \mathscr{L} (P',f)$. We have
		\begin{align}
		\underset{P' \in \gB({P}_{\lambda}, \rho)}{\mbox{sup }} \mathscr{L} (P',f) &=  \mathscr{L}(Q,f)  + \underset{P' \in \gB({P}_{\lambda}, \rho)}{\mbox{sup }} \mathscr{L} (P',f) - \mathscr{L}(Q,f) \nonumber\\ 
		& \leq \mathscr{L}(Q,f)  +  \abs*{\mathscr{L} (P^*,f) - \mathscr{L} (Q,f) }, \nonumber\\
		& \leq \mathscr{L}(Q,f)  +  L_z \abs*{\mathbf{E}_{Z \sim P^*} \bigS{\ell(Z, h)/L_z} - \mathbf{E}_{Z \sim Q} \bigS{\ell(Z, h)/L_z}}\nonumber \\
		& \leq  \mathscr{L}(Q,f) + L_z W_1  (P^*, Q)  \nonumber\\
		& \leq  \mathscr{L}(Q,f) + L_z \bigS{W_2  (P^*, P_{\lambda}) + W_2(P_{\lambda}, Q)} \label{E:triangle2}\\
		& \leq \mathscr{L}(Q,f) + L_z 2 \rho, \nonumber
		\end{align}
		where the  fourth line is due to the Kantorovich-Rubinstein dual representation theorem, i.e., 
		\begin{align*}
		W_1(P, Q) = \sup_{h} \BigC{ \mathbf{E}_{Z \sim P} \bigS{h(Z)} - \mathbf{E}_{Z \sim Q} \bigS{h(Z)}: h(\cdot) \text{ is 1-Lipschitz}}
		\end{align*}
		and the fifth line is due to $W_1  (P^*, Q) \leq W_2  (P^*, Q)$ and triangle inequality. 
	\end{proof}

\section{Proof of Theorem~\ref{thrm:excess_risk}}
\label{proof:thrm_excess_risk}
\begin{proof}
	To simplify notation, we denote $ \Phi \defeq \phi_{\gamma} \circ \mathcal{F} = \set{z \mapsto \phi_{\gamma}(z, f), f \in \mathcal{F}}$ where $\mathcal{F} = \bigC{f_{\theta}, \theta \in \Theta  \subset \mathbb{R}^d}$, which represents the composition of $\phi_{\gamma}$ with each of the loss function $f_{\theta}$ parametrized by $\theta$ belonging to the parameter class $\Theta$. 
	
	Defining $f_{P_{\lambda}} \in \argmin_{f \in \mathcal{F}} \mathscr{L}_{\rho}^{\gamma}(P_{\lambda}, f)$ and $ \widehat{\theta}^* \in \underset{\theta \in \Theta} {\operatorname{argmin}} \ \mathbf{E}_{Z \sim \widehat{P}_{\lambda}} \bigS{\phi_{\gamma}(Z,  f_{\theta})}$ such that $\mathscr{L}_{\rho}^{\gamma}(\widehat{P}_{\lambda}, {f}_{{\theta}^*}) = \underset{\theta \in \Theta} {\operatorname{inf}}  \BigS{\mathbf{E}_{Z \sim \widehat{P}_{\lambda}} \bigS{\phi_{\gamma}(Z,  f_{\theta})}  + \gamma \rho^2 }$, we  decompose the excess risk as follows:
	
	\begin{align}
	\mathscr{E}_{\rho}^{\gamma}(P_{\lambda}, {f}_{\widehat{\theta}^{\varepsilon}})	&=\mathscr{L}_{\rho}^{\gamma}(P_{\lambda}, {f}_{\widehat{\theta}^{\varepsilon}}) - \inf_{f \in \mathcal{F}}  \mathscr{L}_{\rho}^{\gamma}(P_{\lambda}, f) \nonumber\\
	&= \mathscr{L}_{\rho}^{\gamma}(P_{\lambda}, {f}_{\widehat{\theta}^{\varepsilon}}) -   \mathscr{L}_{\rho}^{\gamma}(P_{\lambda}, f_{P_{\lambda}})\nonumber\\
	&=  \BigS{\mathscr{L}_{\rho}^{\gamma}(P_{\lambda}, {f}_{\widehat{\theta}^{\varepsilon}}) - \mathscr{L}_{\rho}^{\gamma}(\widehat{P}_{\lambda}, {f}_{\widehat{\theta}^{\varepsilon}})} + \underbrace{\BigS{\mathscr{L}_{\rho}^{\gamma}(\widehat{P}_{\lambda}, {f}_{\widehat{\theta}^{\varepsilon}}) - \mathscr{L}_{\rho}^{\gamma}(\widehat{P}_{\lambda}, {f}_{\widehat{\theta}^{*}})} }_{\leq \varepsilon}   \nonumber\\
	& \quad \,  + \underbrace{\BigS{\mathscr{L}_{\rho}^{\gamma}(\widehat{P}_{\lambda}, {f}_{\widehat{\theta}^{*}}) - \mathscr{L}_{\rho}^{\gamma}(\widehat{P}_{\lambda}, f_{P_{\lambda}})}}_{\leq 0} + \BigS{\mathscr{L}_{\rho}^{\gamma}(\widehat{P}_{\lambda}, f_{P_{\lambda}}) - \mathscr{L}_{\rho}^{\gamma}(P_{\lambda}, f_{P_{\lambda}})} \nonumber\\
	%	& \leq 2 \sup_{f \in \mathcal{F}} \; \big\lvert\mathscr{L}_{\rho}^{\gamma}(P_{\lambda}, f) - \mathscr{L}_{\rho}^{\gamma}(\widehat{P}_{\lambda}, f) \big\rvert  + \varepsilon\\
	&\leq 2 \sup_{\phi_{\gamma} \in \Phi} \abs*{ \mathbf{E}_{Z \sim P_{\lambda}}[\phi_{\gamma}(Z, f_{\theta})] -   \mathbf{E}_{Z \sim \widehat{P}_{\lambda}}[\phi_{\gamma}(Z, f_{\theta})] } + \varepsilon \nonumber\\
	&\leq 2 \sup_{\phi_{\gamma} \in \Phi}  \sum_{i=1}^{m} \lambda_i \abs*{\mathbf{E}_{Z_i \sim P_i}[\phi_{\gamma}(Z_i, f_{\theta})] -   \mathbf{E}_{Z_i \sim \widehat{P}_{i}}[\phi_{\gamma}(Z_i, f_{\theta})]}  + \varepsilon \nonumber\\
	&\leq 2   \sum_{i=1}^{m} \lambda_i \sup_{\phi_{\gamma} \in \Phi} \abs*{\mathbf{E}_{Z_i \sim P_i}[\phi_{\gamma}(Z_i, f_{\theta})] -   \mathbf{E}_{Z_i \sim \widehat{P}_{i}}[\phi_{\gamma}(Z_i, f_{\theta})]}  + \varepsilon \nonumber\\
	%&\stackrel{\text {with prob. } \geq 1-\delta}{\leq}  
	& \leq  \sum_{i=1}^{m} \lambda_i  \biggS{4 \mathscr{R}_{i}(\Phi)+  2 M_{\ell} \sqrt{\frac{2 \log (2 m / \delta)}{n_i}}} + \varepsilon\, \text{  with probability at least }  1-\delta, \label{unionbound1}
	\end{align}
	where the first inequality is due to optimization error and definition of $\widehat{\theta}^{*}$. The second inequality is due to the fact that $ \abs{\sum_{i=1}^m \lambda_i a_i } \leq \sum_{i=1}^m \lambda_i  \abs{a_i}, \forall a_i \in \mathbb{R}$ and $\lambda_i \geq 0$. The third inequality is because pushing the $\sup$ inside increases the value. For the last inequality, using  the facts that (i) $\abs{\phi_{\gamma}(z,f)} \leq M_{\ell}$  due to $-M_{\ell} \leq \ell(z, h) \leq \phi_{\gamma}(z, f) \leq \sup _{z \in \gZ} \ell(z, h) \leq M_{\ell}$ and (ii) the Rademacher complexity of the function class $\Phi$ defined by $\mathscr{R}_{i}(\Phi) = \mathbf{E}[\sup _{\phi_{\gamma} \in \Phi} \frac{1}{n_i} \sum_{k=1}^{n_i} \sigma_{k} \phi_\gamma(Z_k,  f_{\theta})] $ where the expectation is w.r.t both $Z_k \stackrel{\text { i.i.d. }}{\sim} P_i$ and i.i.d. Rademacher random variable $\sigma_{k}$ independent of $Z_k, \forall k \in [n_i]$,  we have 
	\begin{align}
	\sup_{\phi_{\gamma} \in \Phi} \abs*{\mathbf{E}_{Z_i \sim P_i}[\phi_{\gamma}(Z_i, f_{\theta})] -   \mathbf{E}_{Z_i \sim \widehat{P}_{i}}[\phi_{\gamma}(Z_i, f_{\theta})]}  \geq 2 \mathscr{R}_{i}(\Phi)+   M_{\ell} \sqrt{\frac{2 \log (2 m / \delta)}{n_i}} \label{Rade1}
	\end{align}
	with probability $\leq \delta/m$ due to the standard symmetrization argument and McDiarmid's inequality~{\citep[Theorem~26.5]{shalev_shwartz_understanding_2014}}. Multiplying $\lambda_i$ to both sides of \cref{Rade1}, summing up the inequalities over all $i\in [n]$, and using union bound, we obtain \cref{unionbound1}.
	
	% \blue{[UnderstandingML, Theorem26.5]}
	%Conditioning on the vector $Z_1^n \defeq (Z_k)_{k\in [n]}$, the above expectation gives us the empirical Rademacher complexity denoted by $\widehat{\mathscr{R}}_{n}(\Phi)$. 
	Define a stochastic process $\left(X_{\phi_{\gamma}}\right)_{\phi_{\gamma} \in \Phi}$ 
	$$
	X_{\phi_{\gamma}}:=\frac{1}{{\sqrt{n_i}}} \sum_{k=1}^{n_i} \sigma_{k} \phi_{\gamma}(Z_k, f_{\theta})
	$$
	which is  zero-mean because $\mathbf{E}\left[X_{\phi_{\gamma}}\right]=0$ for all $\phi_{\gamma} \in \Phi$.  To upper-bound  $\mathscr{R}_{n}(\Phi)$, we first show that $\left(X_{\phi_{\gamma}}\right)_{\phi_{\gamma} \in \Phi}$ is a sub-Gaussian process with respect to the following pseudometric
	\begin{align}
	%\left\|\phi_{\gamma}-\phi_{\gamma}^{\prime}\right\|_{L_2(\widehat{P}_{\lambda})}  \defeq \biggP{\frac{1}{n} \sum_{k=1}^{n}  \BigP{\phi_{\gamma}(Z_k, f_{\theta}) - \phi_{\gamma}(Z_k, f_{\theta'})}^2}^{1/2}. 
	\left\|\phi_{\gamma}-\phi_{\gamma}^{\prime}\right\|_{\infty}  \defeq  \sup_{z \in \mathcal{Z}} \Big\lvert  {\phi_{\gamma}(z, f_{\theta}) - \phi_{\gamma}(z, f_{\theta'})}  \Big\rvert . 
	\end{align}  
	
	For any $t \in \mathbb{R}$,  using Hoeffding inequality with the fact that $\sigma_{k}, k \in [n]$, are i.i.d. bounded random variable with sub-Gaussian  parameter 1, we have
	$$
	\begin{aligned}
	\mathbf{E}\left[\exp \left(t\left(X_{\phi_{\gamma}}-X_{\phi_{\gamma}^{\prime}}\right)\right)\right] &=\mathbf{E}\left[\exp \left(\frac{t}{\sqrt{n_i}} \sum_{k=1}^{n_i} \sigma_{k}\left(\phi_{\gamma}\left(Z_k, f_{\theta}\right)-\phi\left(Z_k, f_{\theta'}\right)\right)\right)\right] \\
	&=\left(\mathbf{E}\left[\exp \left(\frac{t}{{\sqrt{n_i}}} \sigma_{1}\left(\phi_{\gamma}\left(Z_{1}, f_{\theta}\right)-\phi_{\gamma}\left(Z_{1}, f_{\theta'}\right)\right)\right)\right]\right)^{n_i} \\
	%& \leq \exp \left(\frac{t^{2} \left\|\phi_{\gamma}-\phi_{\gamma}^{\prime}\right\|_{L_2(\widehat{P}_{\lambda})}^{2}}{2 }\right). 
	& \leq \exp \left(\frac{t^{2} \left\|\phi_{\gamma}-\phi_{\gamma}^{\prime}\right\|_{\infty}^2}{2 }\right). 
	\end{aligned}
	$$
	Then, invoking Dudley entropy integral, we have
	\begin{align} \label{E:Dudley}
	\sqrt{n_i}\, {\mathscr{R}}_{i}(\Phi) = \mathbf{E} \sup_{\phi_{\gamma} \in \Phi} X_{\phi_{\gamma}} \leq {12} \int_{0}^{\infty} \sqrt{\log \mathcal{N}\left(\Phi, \norm{\cdot}_{\infty}, \epsilon\right)} \mathrm{d} \epsilon
	\end{align}
	
	We will show that when $\theta \mapsto \ell(z,h_{\theta})$ is $L_{\theta}$-Lipschitz by Assumption~\ref{Assumption:Lip_cont}, then $\theta \mapsto \phi_{\gamma}(z, f_{\theta})$ is also $L_{\theta}$-Lipschitz as follows. 
	\begin{align*}
	\Big|\phi_{\gamma}(z, f_{\theta}) - \phi_{\gamma}(z, f_{\theta'})\Big|  &=  \Big| \sup_{\zeta \in \mathcal{Z}} \inf_{\zeta' \in \mathcal{Z}} \Big\{ \ell(\zeta,h_{\theta})  - \gamma d(\zeta,  z)   -   \ell(\zeta', h_{\theta'}) + \gamma d(\zeta',  z)  \Big\} \Big| \\
	&\leq  \Big| \sup_{\zeta \in \mathcal{Z}} \Big\{ \ell(\zeta,h_{\theta})    -   \ell(\zeta, h_{\theta'})   \Big\} \Big| \\
	&\leq \sup_{\zeta \in \mathcal{Z}}  \Big|  \ell(\zeta,h_{\theta})    -   \ell(\zeta, h_{\theta'})    \Big| \\
	& \leq L_{\theta}   \norm{\theta - \theta'}, \\
	%& \leq L_f L_{\theta} \norm{\theta - \theta'}
	\end{align*}
	which implies
	\begin{align*}
	\left\|\phi_{\gamma}-\phi_{\gamma}^{\prime}\right\|_{\infty} \leq  L_{\theta}   \norm{\theta - \theta'}.
	\end{align*}
	Therefore, by contraction principle~\citep{shalev_shwartz_understanding_2014}, we have
	\begin{align} \label{E:contraction_principle}
	\mathcal{N}\left(\Phi, \norm{\cdot}_{\infty}, \epsilon\right) \leq \mathcal{N}\left(\Theta, \norm{\cdot}, \epsilon/L_{\theta}\right).
	\end{align}
	Substituting \cref{E:contraction_principle} and \cref{E:Dudley} into \cref{unionbound1}, we obtain
	\begin{align}
	\mathscr{E}_{\rho}^{\gamma}(P_{\lambda}, {f}_{\widehat{\theta}^{\varepsilon}}) \leq  \sum_{i=1}^{m} \lambda_i \BiggS{\frac{48 \mathscr{C}(\Theta)}{\sqrt{n_i}} + 2 M_{\ell} \sqrt{\frac{2 \log (2  m / \delta)}{n_i}}} +\varepsilon,
	\end{align}
	which will be substituted into the upper-bound in Lemma~\ref{lem:excess_risk} to complete the proof. 
\end{proof}

\section{Proof of Corrolary~\ref{corr1}}	 
\label{proof:corr1}
	We now present how we adapt the result from \citet{fournier_rate_2013} to prove \Cref{corr1}
	\begin{proposition}[Measure concentration~{\citep[Theorem~2]{fournier_rate_2013}}]\label{pro:WassRate}
		Let $P$ be a probability distribution on a bounded set $\mathcal{Z}$. Let $\widehat{P}_{n}$ denote the empirical distribution of $Z_{1}, \ldots, Z_{n} \stackrel{\text { i.i.d. }}{\sim} P.$ Assuming that there exist constants $a>1$  such that $A \defeq \mathbf{E}_{Z \sim P}\bigS{\exp(\norm{Z}^a)} < \infty$ (i.e., $P$ is a light-tail distribution). Then, for any $\rho > 0$,
		$$
		%\mathbf{P}\left[ W (P_{\lambda}, \widehat{P}_{n}) \geq  \rho \right] \leq c_{1} \exp \left(-c_{2} n \rho^{\max\{d/p, 2\}}\right) \\
		\mathbf{P}\left[ W_p (\widehat{P}_{n}, P) \geq  \rho \right]  \leq \begin{cases}c_{1} \exp \left(-c_{2} n \rho^{\max \{d/p,2\}}\right) & \text { if } \rho \leq 1 \\ c_{1} \exp \left(-c_{2} n \rho^{a}\right) & \text { if } \rho>1\end{cases}
		$$
		where $c_{1}, c_{2}$ are constants depending on $a, A$ and $d$.
	\end{proposition}
	As a consequence of this proposition,  for any $\delta > 0$, we have  
	\begin{align}\label{E:deltarho}
	\mathbf{P}\left[ W_2 (\widehat{P}_{n}, P) \leq  \widehat{\rho}_{n}^{\delta} \right] \geq 1 - \delta \; \text{  where  } \; \widehat{\rho}_{n}^{\delta} \defeq \begin{cases}\left(\frac{\log \left(c_{1} / \delta \right)}{c_{2} n}\right)^{\min \{2/ d, 1 / 2\}} & \text { if } n \geq \frac{\log \left(c_{1} / \delta\right)}{c_{2}}, \\ \left(\frac{\log \left(c_{1} / \delta\right)}{c_{2} n}\right)^{1 / \alpha} & \text { if } n<\frac{\log \left(c_{1} / \delta\right)}{c_{2}}.\end{cases}
	\end{align}

		In \Cref{pro:WassRate}, \citet{fournier_rate_2013} show that the empirical distribution $\widehat{P}_n$ converges in Wasserstein distance to the true $P$ at a specific rate. This implies that  judiciously scaling the radius of Wasserstein balls according to \cref{E:deltarho} provides natural confidence regions for the data-generating distribution $P$.  
		%Exploiting this fact for \OurAlg, we obtain the excess risk of ${f}_{\widehat{\theta}^{\varepsilon}}$ w.r.t. its true data distribution $P_{\lambda}$.

	%\blue{[p.261, Optimal transport for Applied Mathematicians]}
	By the duality of transport cost~{\citep[p.261]{santambrogio_optimal_2015}}, we have 
	\begin{align*}
	W_p^p(\mu, \nu) = \sup_{\varphi(x) + \psi(y) \leq d^p(x,y)} \int \varphi \, {d} \mu + \psi \, {d} \nu  =  \sup_{\varphi(x) + \psi(y) \leq d^p(x,y)} T_f (\mu, \nu), \quad \forall p \geq 1,
	\end{align*}
	which is the supremum of linear functionals $T_f: \gP \times \gP \mapsto \mathbb{R}$ defined by $T_f (\mu, \nu) = \innProd{(\mu, \nu), (\varphi, \psi)}$; therefore,  $(\mu, \nu) \mapsto W_p^p(\mu, \nu)$ is convex, $\forall p \geq 1$. Thus we have 
	\begin{align}
	W_2^2(P_{\lambda}, \widehat{P}_{\lambda}) = W_2^2 \BigP{\sum\limits_{i=1}^m \lambda_i (\widehat{P}_{n_i}, P_i)} \leq  {\sum\limits_{i=1}^m \lambda_i W_2^2(\widehat{P}_{n_i}, P_i)}. \label{E:convex2}
	\end{align}
	Then, we have
	\begin{align}
	\mathbf{P}\BigS{W_2(P_{\lambda}, \widehat{P}_{\lambda}) \geq \sqrt{\sum\nolimits_{i=1}^{m} \lambda_i \widehat{\rho}_{n_i}^{\delta/m} }}
	&= \mathbf{P}\BigS{W_2^2(P_{\lambda}, \widehat{P}_{\lambda}) \geq \sum\limits_{i=1}^{m} \lambda_i \widehat{\rho}_{n_i}^{\delta/m} } \nonumber\\
	&\leq \mathbf{P}\BigS{\sum\limits_{i=1}^m \lambda_i W_2^2(\widehat{P}_{n_i}, P_i) \geq \sum\limits_{i=1}^{m} \lambda_i \widehat{\rho}_{n_i}^{\delta/m} } \nonumber\\
	&\leq \sum\limits_{i=1}^m  \mathbf{P}\BigS{ W_2^2(\widehat{P}_{n_i}, P_i) \geq   \widehat{\rho}_{n_i}^{\delta/m} } \nonumber\\
	&= \sum\limits_{i=1}^m  \mathbf{P}\BigS{ W_2(\widehat{P}_{n_i}, P_i) \geq   \widehat{\rho}_{n_i}^{\delta/2m} } \nonumber\\
	&\leq \sum\limits_{i=1}^m  \frac{\delta}{2 m}= \frac{\delta}{2}, \label{E:barry}
	\end{align}
	where the first inequality is due to \cref{E:convex2}, the second inequality is due to the union bound, and the last inequality is due to  Proposition~\ref{pro:WassRate} and \cref{E:deltarho}. 
	
	According to \cref{E:triangle2}, by setting $\rho = \BigP{\sum\nolimits_{i=1}^{m} \lambda_i \widehat{\rho}_{n_i}^{\delta/m} }^{1/2}$  in Theorem~\ref{thrm:excess_risk} and using union bound, we complete the proof.

\section{Additional Experimental Settings And Results}
\subsection{Datasets}\label{sec:experiment:setting}

\begin{table}[h]
	\centering
	\setlength\tabcolsep{3.5 pt} 
	\caption{  Statistics of all datasets using in the WAFL's robustness experiments.} 
	\begin{tabular}{lcccll}
		\toprule
		\multirow{2}{*}{Dataset} & \multirow{2}{*}{ $m$} 
		&\multicolumn{1}{c}{\multirow{2}{*}{\begin{tabular}[c]{@{}c@{}}Total\\samples\end{tabular}}}
		&\multicolumn{1}{c}{\multirow{2}{*}{\begin{tabular}[c]{@{}c@{}}Num labels \\ / client\end{tabular}}}
		&\multicolumn{2}{l}{Samples / client}   \\ 
		
		& &                                &                              & \multicolumn{1}{c}{Mean} & \multicolumn{1}{c}{Std}\\ 
		\midrule
		CIFAR-10                  & 20                             &       43,098         &         3     &                \multicolumn{1}{c}{2154 }&\multicolumn{1}{c}{ 593.8}  \\   
		MNIST                     & 100                            &      70,000          &     2          &          \multicolumn{1}{c}{700}       & \multicolumn{1}{c}{313.4}          \\
		%MNIST-M                  & 100                             &       70,000         &         2     &                \multicolumn{1}{c}{700 }&\multicolumn{1}{c}{ 322.4}  \\ 
		\bottomrule

	\end{tabular}
	\label{T:Data_Samples}
\end{table}

For robustness-related experiments, we distribute all datasets to clients as follows:
\begin{itemize}
\item \textbf{MNIST}: A handwritten digit dataset~\citep{lecun_gradient-based_1998} including $70,000$ instances belonged to 10 classes. We distribute dataset to $m = 100$ clients and each client has a different local data size with only $2$ of the $10$ classes.
\item \textbf{CIFAR-10}: An object recognition dataset~\citep{krizhevskyLearningMultipleLayers} including $60,000$ colored images belonged to $10$ classes. We partition the dataset to $m = 20$ clients and there are $3$ labels per client. Each client has a different local data size.
%\item \textbf{Three handwritten digit}: MNIST, MNIST-M \citep{ganin_unsupervised_2015}, and USPS \citep{hullDatabaseHandwrittenText1994}. We distribute one dataset, for example, MNIST to $100$ source clients, and leave MNIST-M to the target client. Each source client has only $2$ labels over $10$ labels. As the number of data samples in USPS is relatively small amount (9,298 samples), we only use this dataset for the target client. %we take turns setting one domain as the target domain and the rest as the distributed source domains, leading to five transfer tasks.
%\item Amazon-Review: Amazon Review (Blitzer et al., 2007a) This dataset provides a testbed for cross-domain sentimental analysis of text. The task is to identify whether the sentiment of the reviews is positive or negative. The dataset contains reviews from amazon.com users for four popular merchandise categories: Books (B), DVDs (D), Electronics (E), and Kitchen appliances (K). Following Gong et al. (2013), we utilize 400-dimensional bag-of-words representation and leverage a fully connected deep neural network as the backbone. .
\end{itemize}
We standardize and randomly split all datasets with $75\%$ and $25\%$ for training and testing, respectively. %In domain adaptation, the target client only has test data. 
The statistics of all datasets are summarized in~\Cref{T:Data_Samples}.

\begin{table}[h]
	\centering
	\setlength\tabcolsep{3.5 pt} 
	\caption{Statistics of all datasets using in the domain adaptation experiments.} 
	\begin{tabular}{lcccll}
		\toprule
		\multirow{2}{*}{Dataset} & \multirow{2}{*}{ Original Size} 
		&\multicolumn{1}{c}{\multirow{2}{*}{\begin{tabular}[c]{@{}c@{}}Total\\samples\end{tabular}}}
		&\multicolumn{1}{c}{\multirow{2}{*}{\begin{tabular}[c]{@{}c@{}}Num labels \\ / client\end{tabular}}}
		&\multicolumn{2}{l}{Samples / client}   \\ 
		
		& &                                &                              & \multicolumn{1}{c}{Training} & \multicolumn{1}{c}{Testing}\\ 
		\midrule
		MNIST                     & 28x28                            &      70,000          &     10          &          \multicolumn{1}{c}{60,000}       & \multicolumn{1}{c}{10,000}          \\
		%MNIST-M                  & 100                             &       70,000         &         2     &                \multicolumn{1}{c}{700 }&\multicolumn{1}{c}{ 322.4}  \\ 
		USPS                     &  16x16                           &      9,298         &     10          &          \multicolumn{1}{c}{7,291}       & \multicolumn{1}{c}{2,007}          \\
		SVHN                     &   32x32                          &      89,289       &     10          &          \multicolumn{1}{c}{63,257}       & \multicolumn{1}{c}{26,032}          \\
		\bottomrule

	\end{tabular}
	\label{T:Data_Samples_2}
\end{table}

For domain adaptation experiments, we use a set of three digit recognition datasets including MNIST \citep{lecun_gradient-based_1998}, USPS \citep{hullDatabaseHandwrittenText1994}, and SVHN \citep{37648}. In particular, we choose two datasets for training and one for testing. The statistics of all datasets are summarized in~\Cref{T:Data_Samples_2}.
\subsection{Models}
The details of models for each dataset is provided as follows:
\begin{itemize}
	\item \textbf{MNIST}: We use a multinomial logistic regression model (MLR) with a cross-entropy loss function and an $L_2$-regularization term. %The loss function at each client is defined as follow:
%	\begin{align}
%		F_k (w)  = \frac{-1}{D_k}{\sum_{j=1}^{D_k}\sum_{c=1}^{C}1_{\{y_j = c\}} \log\frac{\exp(\innProd{a_j, w_{c}})}{\sum_{i=1}^{C}\exp(\innProd{a_i, w_{i}})}}  +\frac{\alpha}{2}\sum_{c=1}^{C}\norm{w_c}_2^2. \nonumber		
%	\end{align}
	\item \textbf{CIFAR-10}: We use a CNN model employed in \citet{mcmahan_communication-efficient_2017}.
	\item \textbf{Three digit recognition datasets (MNIST, USPS, SVHN)}: We use a multinomial logistic regression model (MLR) with a cross-entropy loss function and an $L_2$-regularization term. %We use  a CNN with two $5\times5$ convolution layers (16 channels and 32 channels, respectively), each followed by a $2\times2$ max-pooling layer and ReLU activation, then followed by a fully connected layer with $1,568 ~(32\times7\times7)$ units.
\end{itemize}
In all settings, we set the number of local epochs to $K = 2$ and the number of communication rounds to $T = 200$. For domain adaptation applications, we assign one source domain to one client. For other experiments, we randomly sample $10$ clients to participate in training the global robust model in each communication round. %We report average results over $10$ runs. 
 All experiments were conducted using PyTorch \citep{PyTorch2019}.

\subsection{Comparison between \OurAlg (with p = 1 and p = 2) and other methods on MNIST} 
\label{Apx:WAFL_p}
\begin{figure*}[h]
\centering
\includegraphics[width=0.5\textwidth]{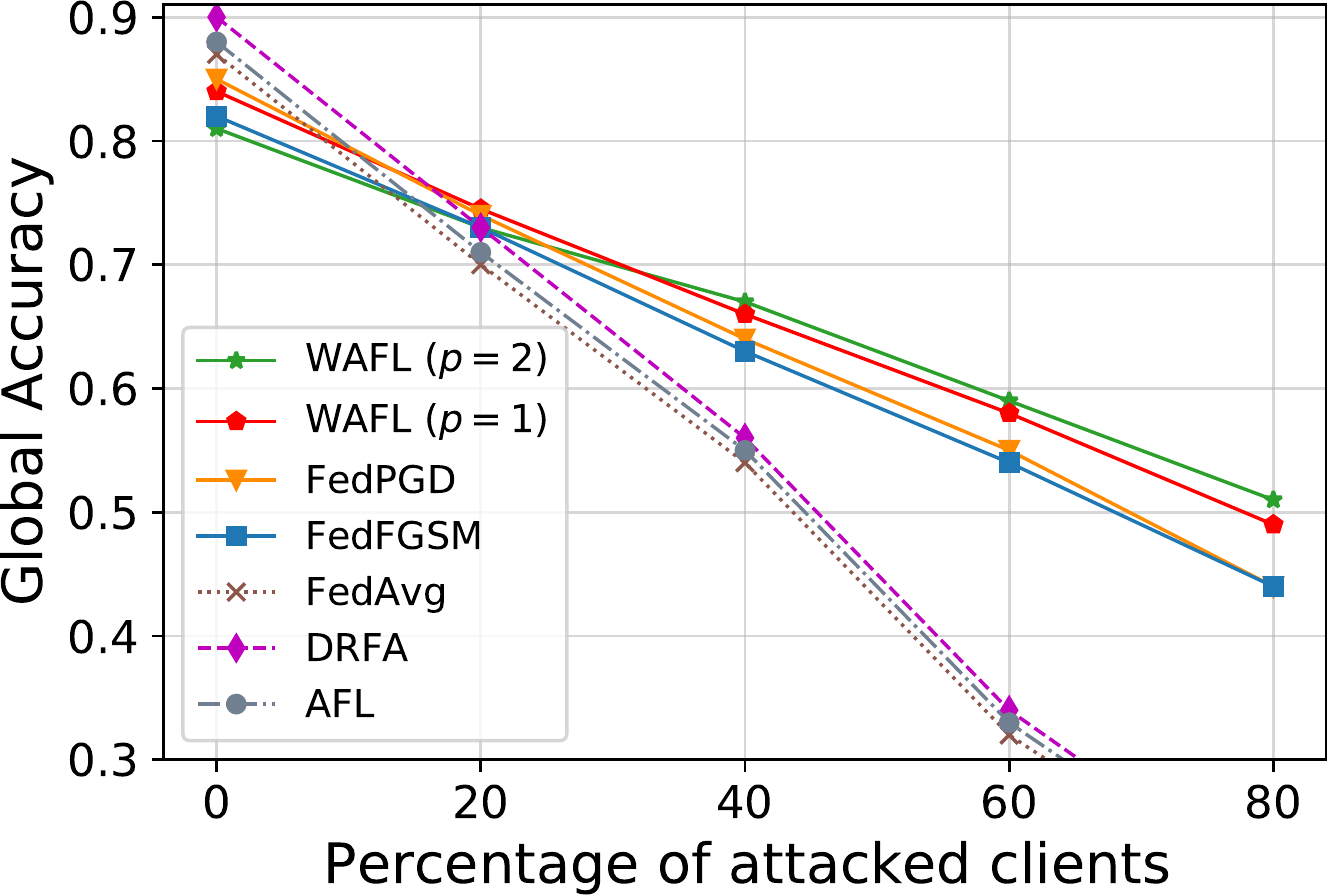}
\caption{Comparison between \OurAlg(with $p=1$ and $p=2$) and other methods on MNIST.}
\label{F:p_values}
\end{figure*}

In an additional experiment, we train  \OurAlg  using $p=1$. The duality result in \Cref{E:duality_Wass} requires only that the distance metric $d$ continuous and convex in its first argument \citep{sinha_certifying_2020}. Therefore, any $\ell_p$ norm would suffice. The use of the $\ell_2$ norm ensures that $d$ is $1$-strongly convex, implying that solving for $\phi_\gamma$ enjoys linear convergence.  As depicted in \Cref{F:p_values}, \OurAlg's performance when $p=1$ is close but not as good as when $p=2$.

\subsection{Convergence of \OurAlg}
We verify the convergence of \OurAlg under two cases: \emph{clean data} (no attacked clients) and \emph{distribution shifts} (where $40\%$ of clients are attacked). In each case, we use two datasets: MNIST and CIFAR-10 and employ the same setup as in \Cref{sec:experiment}. Specifically, for MNIST, we distribute the dataset to $100$ clients and set $\gamma = 0.05$. For CIFAR-10, we use $20$ clients and set $\gamma = 0.5$. We use $T = 200$ communication iterations.

To show \OurAlg's convergence, we plot both the \emph{original} loss (using the function $\ell$) and global accuracy in \Cref{F:convergence}.

\begin{figure*}[t]
	\centering
	\includegraphics[scale=0.28]{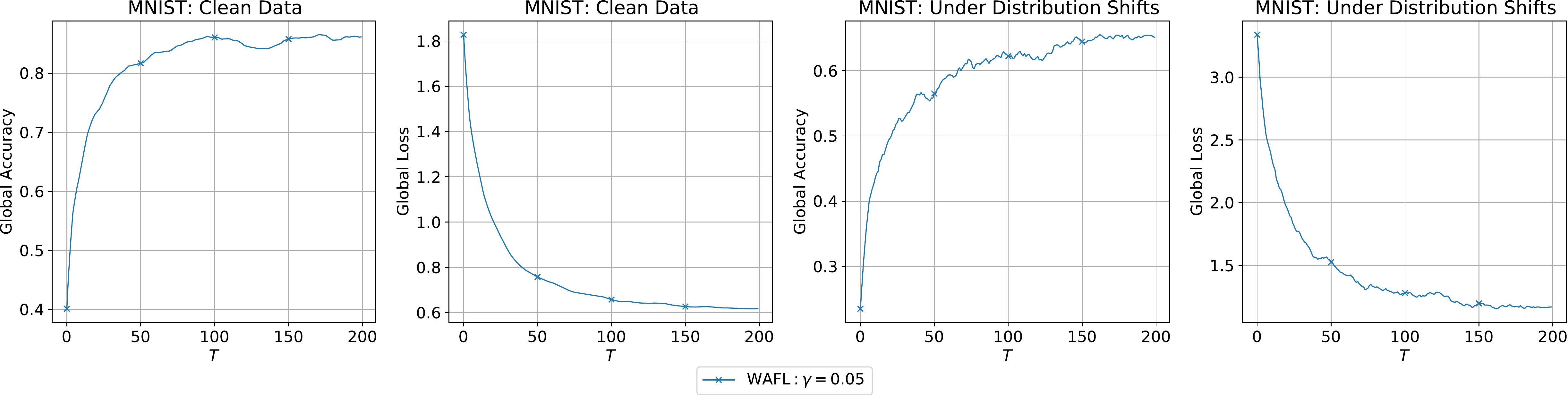}
	\includegraphics[scale=0.28]{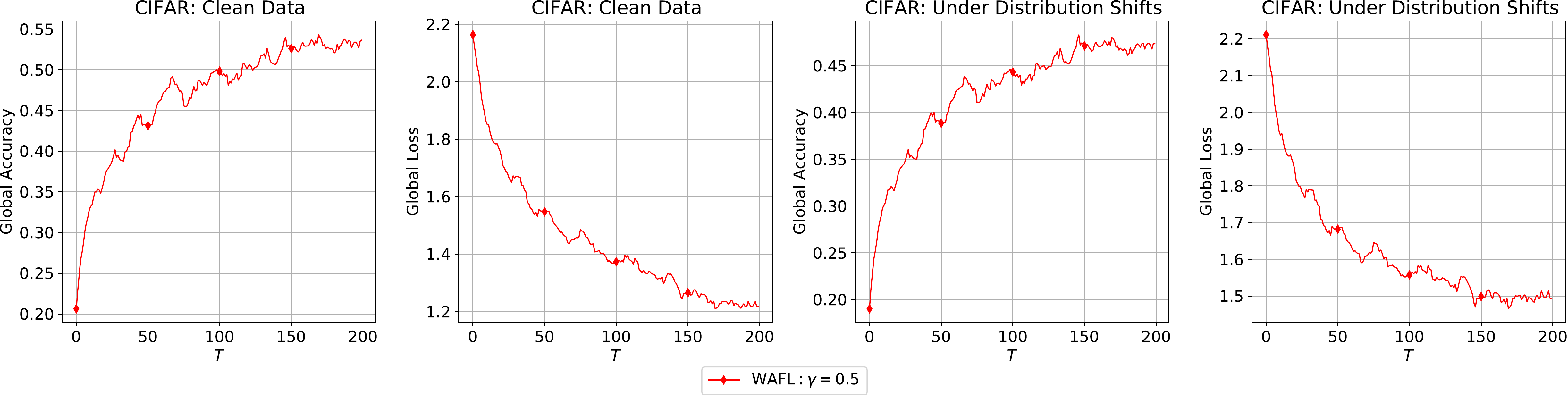}
	\caption{Convergence of \OurAlg.} %Even \OurAlg, FedPGM, and FedFGSM are trained at the same level of perturbation, \OurAlg outperforms others under the adversarial attacks (\blue{ distribution shifts}).}
	\label{F:convergence}
\end{figure*}

\end{document}